\newcommand{\methodname}{R-GFM}
\documentclass{article}

\usepackage{microtype}
\usepackage{graphicx}
\usepackage{subcaption}
\usepackage{booktabs} 
\usepackage{multirow}
\usepackage{pifont}  
\newcommand{\cmark}{\ding{51}}
\usepackage{adjustbox}
\usepackage{xcolor}
\newcommand{\method}[1]{#1}
\usepackage{algorithmic}
\usepackage{bm}

\usepackage{hyperref}

\newcommand{\cstep}[1]{\textcircled{\scriptsize #1}}

\usepackage{verbatim}

\usepackage[preprint]{icml2026}
\usepackage{amsmath}
\usepackage{amssymb}
\usepackage{mathtools}
\usepackage{amsthm}
\usepackage{booktabs}
\usepackage{longtable}
\usepackage{array}
\usepackage[normalem]{ulem}

\usepackage[capitalize,noabbrev]{cleveref}

\theoremstyle{plain}
\newtheorem{theorem}{Theorem}[section]

\theoremstyle{definition}
\newtheorem{definition}[theorem]{Definition}
\newtheorem{assumption}[theorem]{Assumption}
\theoremstyle{remark}

\usepackage[textsize=tiny]{todonotes}

\icmltitlerunning{Learning Graph Foundation Models on Riemannian Graph-of-Graphs}

\begin{document}

\twocolumn[
  \icmltitle{Learning Graph Foundation Models on Riemannian Graph-of-Graphs}

  \icmlsetsymbol{equal}{*}

  \begin{icmlauthorlist}
    \icmlauthor{Haokun Liu}{equal,bme,ddl}
    \icmlauthor{Zezhong Ding}{equal,ai,ddl}
    \icmlauthor{Xike Xie}{bme,ddl}
  \end{icmlauthorlist}

  \icmlaffiliation{bme}{School of Biomedical Engineering, University of Science and Technology of China (USTC), Suzhou, Jiangsu, China}
  \icmlaffiliation{ddl}{Data Darkness Lab, Suzhou Institute for Advanced Research, USTC, Suzhou, Jiangsu, China}
  \icmlaffiliation{ai}{School of Artificial Intelligence and Data Science, USTC, Hefei, Anhui, China}

  \icmlcorrespondingauthor{Xike Xie}{xkxie@ustc.edu.cn}

  \icmlkeywords{Graph Foundation Models, Riemannian Geometry, Graph Neural Networks, Graph-of-Graphs, Machine Learning}

  \vskip 0.3in
]

\printAffiliationsAndNotice{\icmlEqualContribution}

\begin{abstract}
  Graph foundation models (GFMs), pretrained on massive graph data, have transformed graph machine learning by supporting general-purpose reasoning across diverse graph tasks and domains. 
  Existing GFMs pretrained with fixed-hop subgraph sampling impose a fixed receptive field, causing scale mismatch on diverse tasks, which often require heterogeneous and unknown structural contexts beyond a fixed sampling scale.
  We propose {\bf R-GFM}, a Riemannian Graph-of-Graphs (GoG) based foundation model, that treats {\it structural scale} as a first-class citizen in modeling. 
  R-GFM constructs a multi-scale GoG over-sampled subgraphs at different hop distances and learns geometry-adaptive representations from Riemannian manifolds. 
  Theoretical analysis shows that R-GFM reduces structural domain generalization error compared to fixed-scale GFMs.
  Experiments on various datasets demonstrate that R-GFM achieves state-of-the-art performance, with up to a \textbf{49\%} relative improvement on downstream tasks. Our code is available at \url{https://github.com/USTC-DataDarknessLab/R-GFM}.
\end{abstract}

\section{Introduction}
\begin{figure}[t]
    \centering
    \includegraphics[width=\linewidth]{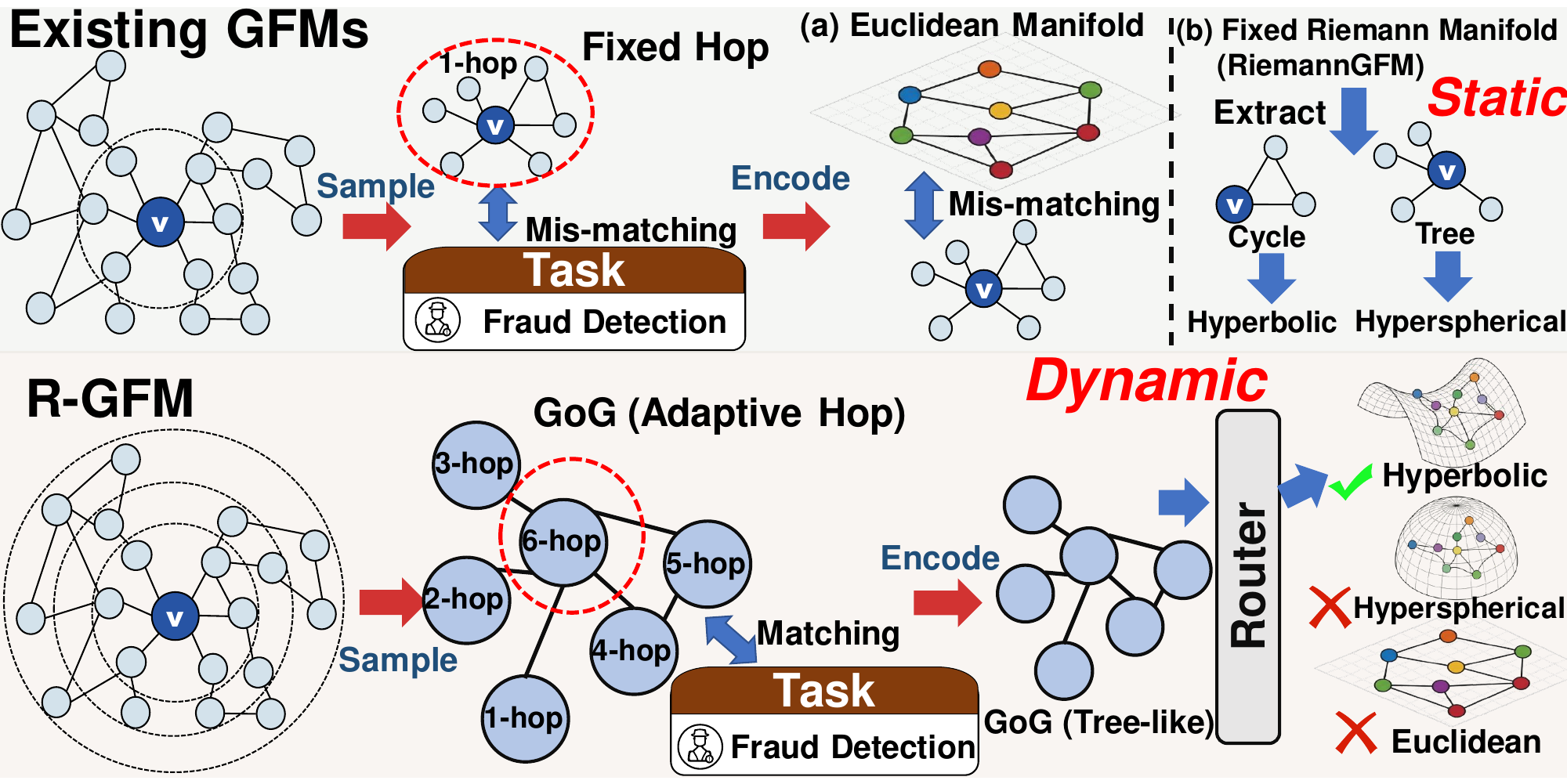}
    \vspace{-2mm}
    \caption{\textbf{Motivation of R-GFM.} Existing GFMs rely on fixed-hop sampling and static manifolds, causing representation mismatch. R-GFM constructs an adaptive-hop GoG and uses a router to dynamically choose the best-fit geometry for better matching.}
    \vspace{-3mm}
    \label{fig:motivation}
\end{figure}

Graph foundation models (GFMs)~\cite{wang2025gfm_survey, wang2025multi,yu2025samgpt,liu2024ofa,xia2024anygraph}, which are pretrained on massive graph datasets, have shown strong 
performance across diverse downstream applications, including molecular property prediction \cite{wu2018moleculenet}, quantum chemistry \cite{gilmer2017neural}, web-scale recommendation \cite{ying2018pinsage,he2020lightgcn},  knowledge graph completion \cite{schlichtkrull2018rgcn}, and spatio-temporal forecasting \cite{li2018dcrnn}. 

Despite this progress, existing GFMs adopt a common design choice: they rely on {\it fixed-hop subgraph sampling} during pretraining, where each node or subgraph is encoded using a predetermined neighborhood range (e.g., $1$-hop as shown in Figure~\ref{fig:motivation}). This design implicitly imposes a {\it fixed structural receptive field} on representation learning. As a result, the model is restricted to a static receptive field of structural context and fails to capture adaptive-hop structural information (i.e., diverse \textit{structural scales}).
However, in practice, different downstream tasks impose heterogeneous and unknown requirements on the hop range of subgraph sampling.

For instance, in node classification on highly homophilous citation networks~\cite{kipf2017gcn,cavallo2022ncs}, low-hop information (e.g., within 2 hops) is often sufficient, as local neighbors tend to share similar labels.
In contrast, for fraud detection in e-commerce~\cite{liu2021pcgnn}, higher-hop information (e.g., beyond 4 hops) is required to uncover complex collusion patterns, where fraudsters hide their traces through long transaction chains.

More generally, fine-grained structural scales preserve feature distinctiveness~\cite{00020GTW22}, but limit the scope of the receptive fields and are prone to the under-reaching problem~\cite{0002Y21}. 
In contrast, coarse-grained structural scales capture long-range dependencies~\cite{LukovnikovF21} and global structural information, but introduce irrelevant structural noise~\cite{DongK23} and suffer from the over-smoothing problem~\cite{Keriven22}. 

These observations reveal a fundamental {\it structural scale mismatch} in current GFMs: a fixed structural receptive field cannot simultaneously satisfy the heterogeneous structural requirements of diverse tasks, underscoring the necessity of designing a dynamic and adaptive, hop-aware representation learning paradigm.

In response, we propose R-GFM (as shown in Figure 1), a new GFM incorporating the adaptive-hop subgraphs of training nodes based on a dynamic Graph-of-Graphs (GoG) framework (Section~\ref{sec:method_gog}). 
In R-GFM, subgraphs sampled at different hops are treated as nodes in a higher-order GoG, enabling the model to explicitly represent and reason over structural relationships across multiple scales.
Different structural scales in the dynamically constructed GoGs often show different Riemannian geometric characteristics. 
Specifically, local-scale subgraphs tend to be dense and highly connected, while larger-scale subgraphs are typically sparse and hierarchical. 
Encoding such heterogeneous structures within a single Euclidean embedding space leads to geometric mismatch and representation distortion, which motivates a geometry-adaptive routing mechanism that assigns GoGs to appropriate Riemannian manifolds via a mixture-of-experts (MoE) framework.

However, routing GoGs to an appropriate Riemannian manifold poses non-trivial challenges.  
\textbf{First}, determining an appropriate number of Riemannian experts is non-trivial. Too few experts lead to underfitting, while too many experts result in overfitting and training instability.
\textbf{Second}, existing graph MoE models~\cite{xia2024anygraph} typically select a fixed top-$m$ set of experts for aggregation. An overly small $m$ leads to insufficient representation capacity, while an excessively large $m$ exacerbates generalization~\cite{abs-2403-17404}, thereby hindering model convergence.
In response, we propose a dynamic MoE-based Riemannian manifold routing strategy ({{Section~\ref{sec:method_moe}}}) 
that maintains a dynamic candidate set of Riemannian experts and activates the experts for each input GoG, capturing the Riemannian geometric information in GoGs and improving generalization.

These designs are supported by both theoretical analysis and empirical evaluation.
Theoretically, we prove that our {adaptive-hop GoG construction} reduces subgraph embedding error compared to fixed-hop. Moreover, we prove that our routing strategy yields a tighter excess-risk bound and a tighter cross-domain generalization error bound than using a fixed expert set.
Further experiments demonstrate that R-GFM shows strong domain generalization on diverse real-world graphs and downstream tasks ({Section~\ref{sec:experiments}).

Our main contributions are summarized as follows:
\begin{itemize}
    \item 
        We propose R-GFM, a novel GFM based on the Graph-of-Graphs mechanism, enabling scale-adaptive domain generalization across diverse graph domains.
    \item 
We introduce a memory-aware GoG construction strategy that effectively captures structural information across varying hops with theoretical guarantees.
    \item 
    We introduce an MoE-based Riemannian routing strategy that effectively captures Riemannian geometric information in GoGs with theoretical guarantees.
    \item 
    Experiments on \textbf{18} real-world graphs and various downstream tasks demonstrate consistent superiority over baselines, achieving up to a \textbf{49\%} improvement.
\end{itemize}

\section{Preliminaries}
\label{sec:prelim}
\subsection{Attributed Graph}
\label{sec:graph_notation}
We denote an attributed graph as $G=(V, E,\mathbf{X}_V)$ with
node set $V$, edge set $E$, and node features $\mathbf{X}_V$.

$\mathrm{dist}(u,v)$ denotes the shortest-path distance between nodes $u\in V$ and $v\in V$.
$\mathcal{N}_{k}(u) = \{v\in V:\mathrm{dist}(u,v) \leq k\}$ denotes the $k$-hop neighbors of node $u$.

\subsection{GFM Design Problem~\cite{wang2025multi}}
{Given a set of graphs $\{G_i\}$ (including training and testing graphs) from domains $\mathcal{D} = \{D_i\}$, our goal is to design a GFM that is pretrained only on the training graph domains $D_\mathrm{Train} \subset \mathcal{D}$. The objective is to maximize the downstream task performance (e.g., node classification) on the test graphs from the unseen domain $D_\mathrm{Test} \cap D_\mathrm{Train} = \emptyset$.}

\subsection{Graph-of-Graphs (GoG)~\cite{zhen2023gog}}
Given a graph set $\{G_i\}$, the GoG of $\{G_i\}$ is denoted as $\mathcal{G} = (\mathcal{V}, \mathcal{E}, \mathbf{X}_G)$, where $\mathcal{V} =  \{G_i\}$, $\mathcal{E}\footnote{In our design, $\mathcal{E}$ is a subset of $\mathcal{V} \times \mathcal{V}$, which is constructed by graph similarity-based sampling strategy, as shown in Equation~\ref{eq:sampleedge}.} \subseteq \mathcal{V} \times \mathcal{V}$, and $\mathbf{X}_G$ is the graph features of $\{G_i\}$.

\subsection{Riemannian Manifold~\cite{sun2025riemanngfm}}
A Riemannian manifold is a geometric space, which is a smooth manifold endowed with a Riemannian metric. The curvature $\kappa$ is the geometric quantity measuring the extent to which a surface deviates from being flat. Following the setting of \cite{sun2025riemanngfm}, we use the constant-curvature Riemannian manifolds, where curvature $\kappa$ is equal everywhere, including three types of manifolds: (1) Hyperbolic manifold with $\kappa <0$, (2) Euclidean manifold with $\kappa = 0$, and (3) Hyperspherical manifold with $\kappa > 0$.

\section{The Proposed GFM: R-GFM}
\label{sec:method}
In this section, we present \methodname{}, a GFM that combines adaptive-hop GoG construction with an MoE-based Riemannian routing to produce transferable node representations.

\begin{figure*}[t]
    \centering
    \includegraphics[width=\textwidth]{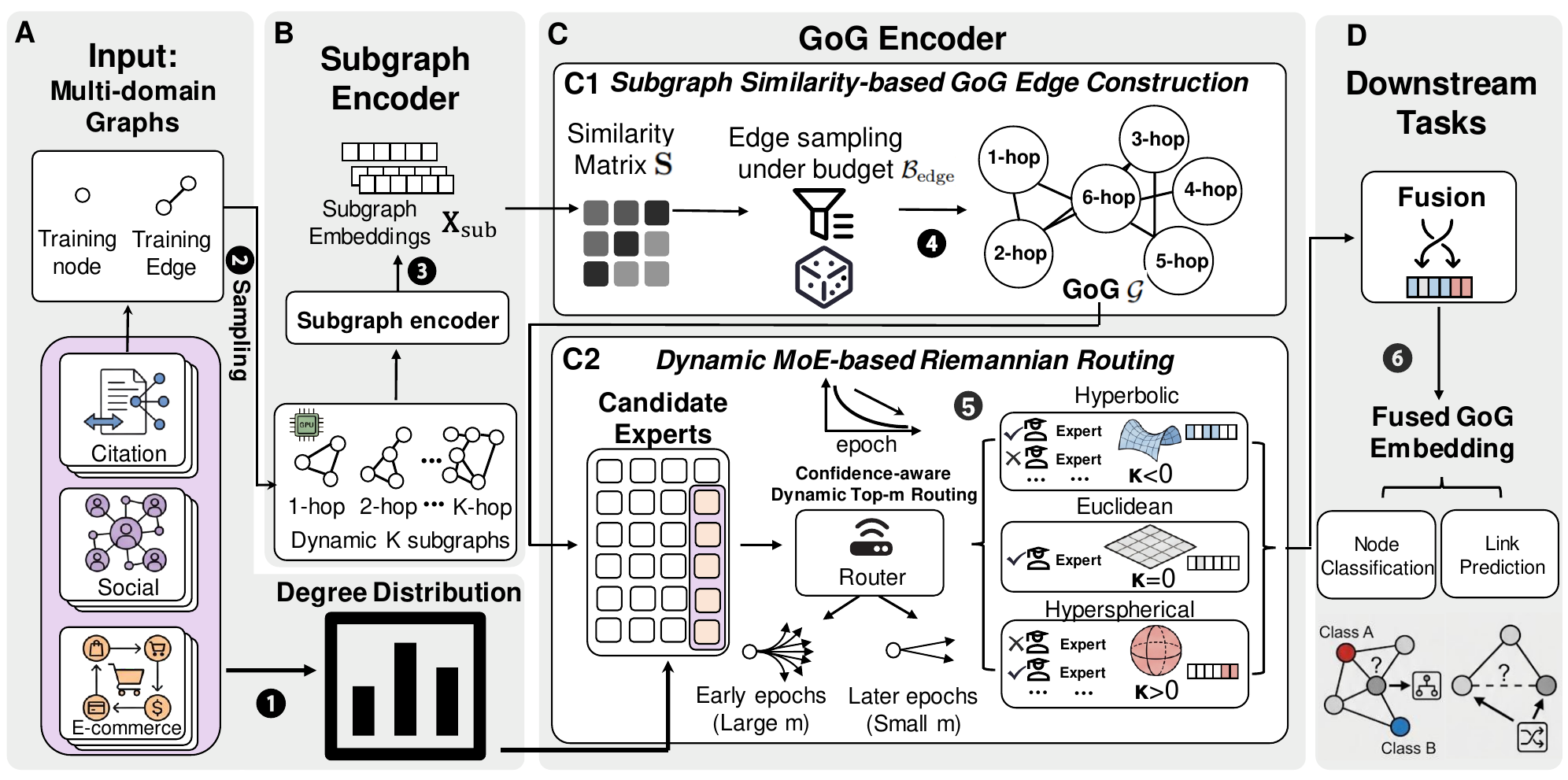}
    \caption{\textbf{Overview of \methodname{}}.
    \cstep{1} Calculate the coefficient of variation (CV) to quantify the node degree distribution, which is used to determine the candidate Riemannian expert set.
    \cstep{2} Sample adaptive-hop subgraphs for each training node $v$ (or an edge $(u,v)$ in link prediction task~\cite{zhang2018link})  with a adaptive-hop strategy (Section~\ref{method:autok_strategy}).
     \cstep{3} Encode sampled subgraphs $\{G_v^{(i)}\}_{i=1}^{K}$ to obtain subgraph embeddings ($\mathbf{X}_{\mathrm{sub}}$), and pretrain this subgraph encoder via contrastive learning using NT-Xent loss~\cite{ChenK0H20}.
     \cstep{4} Construct a  GoG $\mathcal{G}$ by sampling edges from the subgraph similarity matrix $\mathbf{S}$ under the edge number budget $\mathcal{B}_{\mathrm{edge}}$ (Section~\ref{sec:method_gog}).
     \cstep{5} Encode $\mathcal{G}$ with the dynamic MoE-based Riemannian routing encoder and perform dynamic top-$m$ expert routing  (Section~\ref{sec:method_moe}).
     \cstep{6} generate the fused GoG embedding, which is used for downstream tasks, such as node classification or link prediction.
    }
    \label{fig:rgfm}

\end{figure*}

\subsection{Overview of \methodname{}}
\label{sec:overviewRGFM}
In Figure~\ref{fig:rgfm}, the R-GFM framework consists of four stages. \textbf{Stage A (Steps~\cstep{1}-\cstep{2})} determines the candidate Riemannian expert set and samples subgraphs of each training node with hop distances ranging from $1$ to $K$. 
\textbf{Stage B (Step~\cstep{3})} encodes the sampled subgraphs to obtain their embeddings.
\textbf{Stage C (Steps~\cstep{4}-\cstep{5})} constructs a GoG based on subgraph embeddings and applies MoE-based Riemannian routing to encode the GoG by using the active Riemannian experts, generating GoG embeddings. 
\textbf{Stage D (Step~\cstep{6})} aggregates the GoG embeddings into a single fused embedding, which is used for downstream tasks. 
In the sequel, we present two core strategies in {Steps~\cstep{4}}-\cstep{5}, 
including adaptive-hop GoG Construction and MoE-based Riemannian routing in Sections~\ref{sec:method_gog} and \ref{sec:method_moe}, respectively.

\subsection{Adaptive-Hop GoG
Construction}
\label{sec:method_gog}
We now describe how to construct a GoG for each training node to capture different structural scales. 
The construction proceeds as follows:
(1) determining the GoG node set via multi-hop subgraph sampling, and (2) establishing GoG edges based on subgraph similarity.

\paragraph{Memory-aware GoG Node Determination.}
\label{method:autok_strategy}
To capture both local and global information, we sample subgraphs ($\{G_{v}^{(i)}\}_{i=1}^{K}$) centered at node $v$ from hop distances $1$ to $K$, where $K$ is a dynamic parameter that controls the receptive field size and determines the number of GoG nodes.  
Larger $K$ enables richer structural modeling but incurs higher GPU memory cost.
Therefore, our objective is to maximize $K$ under an explicit GPU memory budget $\mathcal{B}_{\mathrm{GPU}}$. 

To address the problem, we design an online greedy strategy. 
We progressively enlarge $K$ and test whether the resulting configuration fits within $\mathcal{B_{\mathrm{GPU}}}$. 
Once the feasibility test fails (e.g., out-of-memory), we roll back and return the largest feasible $K$. 

\paragraph{Subgraph Similarity-based GoG Edge Construction.}

After determining the GoG node set, we construct GoG edges based on subgraph similarity.
We define subgraph similarity as follows.

\begin{definition}[\textbf{Subgraph Similarity}]
\label{def:subgraph_similarity}
    The {\it similarity} between two subgraphs is defined as the cosine similarity of their embeddings. The {\it subgraph similarity matrix} ($\mathbf{S}$) includes all pairwise subgraph similarities as:
\begin{equation}
\mathbf{S}={\mathbf{X}}_{\mathrm{sub}}\big({\mathbf{X}}_{\mathrm{sub}}\big)^{\top}\in\mathbb{R}^{K\times K},
\end{equation}
where $\mathbf{X}_\mathrm{sub}$ is the subgraph embedding matrix.
\end{definition}

Rather than densely connecting all subgraphs, we construct a sparse GoG by sampling edges according to subgraph similarity.
Specifically, we normalize the similarity scores into a sampling distribution:

\begin{equation}
\mathrm{Prob}(i,j)=\frac{{\mathrm{e}}^{\mathbf{S}[i,j]}}{\sum_u\sum_ve^{\mathbf{S}[u,v]}},
\label{eq:gog_global_sampling}
\end{equation}
where $\mathrm{Prob}(i, j)$ is {the probability of sampling edge ($i,j$)}, $\mathrm{e}$ is Euler's number.

We then sample $\mathcal{B}_{\mathrm{edge}}$ edges without replacement\footnote{We set $\mathcal{B}_\mathrm{edge}=0.6 \times\frac{K(K-1)}{2}$, as discussed in Section~\ref{sec:hparam}. 

}, as shown in Equation~\ref{eq:sampleedge}.
\begin{equation}
\label{eq:sampleedge}
\mathcal{E} \sim \mathrm{Sample}\big(\mathrm{Prob};\, \mathcal{B}_\mathrm{edge}\big), \mathrm{subject\ to} \ |\mathcal{E}| = \mathcal{B}_\mathrm{edge} 
\end{equation}
where $\mathcal{E}$ is the GoG edge set.

Finally, we symmetrize the sampled edges to enable bidirectional message passing, which empirically improves structural pattern identification and representation quality~\cite{egressy2024directedmultigraphs}.

{\paragraph{Theoretical Analysis.}
We provide theoretical justification for our GoG construction strategy from two perspectives: (1) why sampling subgraphs from varied hops is preferable to using a fixed hop (Theorem~\ref{thm:theorem1}),
and (2) why constructing GoG edges improves subgraph representation quality. 
(Theorem~\ref{thm:gog_edges_help_simple_1}).

To begin, we first introduce the concept of {\it embedding noise} and the {\it sampled subgraph embedding} to analyze the impact of embedding noise on subgraph representations
\cite{li2024noise}. Typically, the subgraph embedding ($\mathbf{x}$) follows a Gaussian distribution ($\mathcal{N}$) with noise level $\sigma$, where smaller $\sigma$ corresponds to higher representation quality~\cite{li2024noise}), as shown in Equation~\ref{label:noisedefintion}.
\begin{equation}
\label{label:noisedefintion}
\mathbf{x}\sim \mathcal{N}(\boldsymbol{\mu}, \boldsymbol{\sigma}^2),
\end{equation}
where $\boldsymbol{\mu}$ is {noise-free subgraph embedding}.

Let $\boldsymbol{\sigma}_\mathrm{F}$ denote the embedding noise induced by sampling from a fixed hop, and let $\boldsymbol{\sigma}_\mathrm{V}$ denote the embedding noise induced by sampling across multiple hops under our strategy.
Theorem~\ref{thm:theorem1} follows from these definitions.

\begin{theorem}
\label{thm:theorem1}
The embedding noise of sampling from various hops is lower than that of sampling from a fixed hop: $\|\boldsymbol{\sigma}_\mathrm{V}\|_2 \le \|\boldsymbol{\sigma}_\mathrm{F}\|_2$, where $\|\cdot\|$ denotes the $\ell_2$ norm.
\end{theorem}
The proof of Theorem~\ref{thm:theorem1} is in Appendix~\ref{app:proof_noise_analysis}.

Theorem~\ref{thm:theorem1} implies that sampling subgraphs across multiple hops yields lower embedding noise than fixed-hop sampling, thereby validating the effectiveness of our sampling strategy.

We next analyze the effectiveness of our GoG edge construction strategy by comparing it with alternative GoG edge designs under an identical set of GoG nodes.
Specifically, we show that our similarity-based sparse construction achieves strictly lower expected embedding error than two canonical constructions: a GoG without edges and a fully connected GoG.
These constructions correspond to the extreme cases of no cross-scale interaction and indiscriminate interaction, respectively, thereby demonstrating the advantage of selective structural connectivity as shown in Theorem~\ref{thm:gog_edges_help_simple_1}.

\begin{theorem}[\bf Effectiveness of GoG Edge Construction]
\label{thm:gog_edges_help_simple_1}
Consider GoG constructions that share the same set of GoG nodes. Let $e_\mathrm{none}$, $e_\mathrm{full}$, and $e_\mathrm{ours}$ denote the squared embedding error obtained by (i) a GoG without edges, (ii) a full connected GoG, and (iii) our similarity score-based GoG, respectively. Then, $e_\mathrm{ours} < e_\mathrm{none}$ and $e_\mathrm{ours} < e_\mathrm{full}$.
\end{theorem}

The proof of Theorem~\ref{thm:gog_edges_help_simple_1} is in Appendix~\ref{app:proof_gog_edges_help_simple_1}.

\subsection{Dynamic MoE-based Riemannian Routing}
\label{sec:method_moe}
Given a constructed GoG $\mathcal{G}=(\mathcal{V},\mathcal{E},\mathbf{X}_\mathrm{G})$,
we learn its representation using a dynamic MoE-based Riemannian routing strategy.
Our routing is dynamic in two aspects: (1) the candidate curvature set
$\mathcal{K}$ is dynamically selected to determine the candidate Riemmanian expert set $M$, based on the node degree distributions of given datasets; and
(2) the number of activated experts is adjusted dynamically via confidence-aware Top-$m$ routing during training.

\paragraph{Dynamical Candidate Expert Set Determination.}
Motivated by prior work showing that heterogeneous graph structures benefit from diverse geometric inductive biases~\cite{gu2019learn},
we adopt a structure-aware strategy based on node degree distributions~\cite{hu2017hetero}.
We quantify structural heterogeneity using the coefficient of variation (CV):
\begin{equation}
\mathrm{CV}(\mathcal{D}_i)\;=\;\frac{\mathrm{std}(\deg({\mathcal{D}_i}))}{\mathrm{mean}(\deg({\mathcal{D}_i}))},
\label{eq:deg_cv}
\end{equation}
where $\deg({\mathcal{D}_i})$ is the node degrees of graph dataset ${\mathcal{D}_i}$.

Then, we define the score $\mathcal{S}_i$ to quantify the structural heterogeneity of datasets $\{\mathcal{D}_1, \mathcal{D}_2, \cdots, \mathcal{D}_i\}$ as:
\begin{equation}
\label{eq:scorec}
\mathcal{S}_i = \mathrm{normalize}(\mu_i+\sigma_i),
\end{equation}
where $\mu_i$ is the mean of $\{\mathrm{CV}(\mathcal{D}_1), \cdots, \mathrm{CV}(\mathcal{D}_{i})\}$ and $\sigma_i$ is the standard variance of that. Here, {$\mu_i$ captures the average structural heterogeneity across datasets, while $\sigma_i$ measures its variability. 
We apply normalization to ensure scale consistency.
}

Then, based on the score, we design a candidate expert set determination strategy as shown below.
Initially, we start with a single expert and predefine a Riemannian curvature list, where each $\kappa_i$ corresponds to the curvature of the $i$-th candidate expert.

{As datasets $\{\mathcal{D}_1, \mathcal{D}_2, \cdots\}$ are introduced sequentially, we update the heterogeneity score $\mathcal{S}_i$ upon receiving each new dataset according to Equation~\ref{eq:scorec}. 
Before training on $\{\mathcal{D}_1, \mathcal{D}_2, \cdots, \mathcal{D}_i \}$, we compute $\mathcal{S}_i$ and determine the candidate set size as $\left\lceil \mathcal{S}_i \cdot \zeta \right\rceil$.
After that, we generate the set of curvature values by starting at $0$ and alternating between decreasing and increasing integer values (e.g., $0,-1,+1,-2,+2,\ldots$) until the set reaches size $\left\lceil \mathcal{S}_i \cdot \zeta \right\rceil$. Then, we use this curvature set to determine the candidate Riemannian experts. Each curvature defines a specific Riemannian manifold, which is associated with one expert.}

\paragraph{Confidence-aware Dynamic Top-$m$ Routing.}
Standard MoE models select a fixed Top-$m$ experts for routing~\cite{shazeer2017outrageously}.
In contrast, we dynamically adjust $m$ based on routing confidence.

To begin, we define the routing score of a $\mathcal{G}$ in Equation~\ref{eq:router_softmax}.  
\begin{equation}
\boldsymbol{\alpha}_\mathcal{G}=\mathrm{softmax}\!\left(g(\mathcal{G})/\tau\right),
\label{eq:router_softmax}
\end{equation}
where $\tau$ is a temperature. We use GCN~\cite{kipf2017gcn} as the GNN encoder $g(\cdot)$, following the setting of~\cite{guo2025graphmore}. $\boldsymbol{\alpha}_\mathcal{G} \in \mathbb{R}^{\psi\times1}$, where $\psi$ is the number of experts. 

Then we define a global confidence score in Equation~\ref{eq:confidence}.
\begin{equation}
\mathrm{conf}=\frac{1}{\psi}\sum_{1 \leq i \leq \psi}\max{\alpha^{(i)}_\mathcal{G}},
\label{eq:confidence}
\end{equation}
where $\alpha^{(i)}_\mathcal{G}$ is the routing score of expert $i$.

Since the router confidence typically grows during training, the router becomes more confident (i.e., producing more peaked distributions). This implies that fewer experts are sufficient to capture GoG structure information.

Motivated by this~\cite{huang2024hardneedmore}, we design the workflow of updating $m$ as follows. Initially, $m$ is set to an initial value $m_0$. 
Then, we dynamically adjust the active expert number $m$ at the epoch level by decreasing it stepwise:
$m \leftarrow \max\!\big(1,\, m-\mathrm{conf}\big)$.

\paragraph{Theoretical Analysis.}
Here, we theoretically analyze our strategies from these two perspectives:
(1) why a dynamic candidate expert set is needed for different input datasets (Theorem~\ref{thm:dynamic_expert}) and
(2) why our method theoretically outperforms MDGFM~\cite{wang2025multi}, which is the state-of-the-art GFM with theoretical analysis~(Theorem~\ref{thm:main_tight}).

\begin{theorem}[\textbf{Excess-risk Upper Bound Analysis}]
\label{thm:dynamic_expert}
Considering $N$ training datasets $\{\mathcal{D}_1, \mathcal{D}_2, \cdots, \mathcal{D}_N \}$, let the excess risk~\cite{koltchinskii2010rademacher} upper bound of selected expert number with  $j$ experts as
$\mathcal{R}(j)$:
\begin{equation}
\label{eq:UB_def}
\mathcal{R}(j) = \frac{A\,\mathcal{S}_{N}}{j}+B\sqrt{\frac{j}{n_N}},
\end{equation}
where $A,B >0$ are constants, $n_N$ is the number of constructed GoGs from $\mathcal{D}_1$ to $\mathcal{D}_N$, and $S_N$ is the score defined in Equation~\ref{eq:scorec}.

Let $\psi_F$ be the fixed candidate expert number, and let
$\psi_D$ be our dynamically selected expert number.
Then:
\begin{equation}
\label{eq:UD_UF}
\mathcal{R}(\psi_D) \le \mathcal{R}(\psi_F).
\end{equation}
Equality holds if and only if the fixed number of experts is identical to the number determined by our dynamic strategy.

\end{theorem}
Typically, existing methods~\cite{guo2025graphmore} employ a manually specified number of experts, which requires extensive trial-and-error tuning and fails to adapt to various graph datasets (i.e., the equality in Equation~\ref{eq:UD_UF} does not hold in general).
The proof of Theorem~\ref{thm:dynamic_expert} is in Appendix~\ref{app:proof_dynamic_expert}.

{\begin{theorem}[\textbf{Domain Generalization Error Bound Analysis}]
\label{thm:main_tight}
Let $\Phi_{\mathrm M}$ denote the encoder class induced by MDGFM~\cite{wang2025multi}, and let $\Phi_{\mathrm R}$ denote the encoder class induced by \methodname{}.
Let the best achievable target-domain surrogate upper bounds of MDGFM and \methodname{} be $\epsilon_{\mathrm{MDGFM}}, \epsilon_{\mathrm{R\text{-}GFM}}$. 
Under the assumptions detailed in Appendix~\ref{app:tight}, we have $\epsilon_{\mathrm{R\text{-}GFM}}
\ <\
\epsilon_{\mathrm{MDGFM}}$.
\end{theorem}}

The proof of Theorem~\ref{thm:main_tight} is in Appendix~\ref{app:tight}.

\section{Experiments}
\label{sec:experiments}
In this section, we conduct extensive experiments to answer the following research questions: \\
(\textbf{RQ1}): Can \methodname{} outperform existing models on downstream tasks under the same training settings? \\
(\textbf{RQ2}): How robust is \methodname{} to graph perturbations compared with existing models?\\
(\textbf{RQ3}): How do the key components of \methodname{} contribute to its performance? \\
(\textbf{RQ4}): How does \methodname{} scale? \\
(\textbf{RQ5}): How should we choose the key hyperparameters of \methodname{} in practice, and how do they affect performance?
\subsection{Experimental Setup}
\label{sec:exp_setup}

\paragraph{Datasets}
In our main setting, we use 10 benchmark graphs from 4 domains, which are widely used in prior GFM studies~\cite{zhao2024all, wang2025multi, yu2025pronog}, as shown in Table~\ref{table.datasets} of Appendix~\ref{app:dataset:main}.
We adopt a leave-one-dataset-out transfer setting: for each target dataset, we pretrain on the other datasets and test on the held-out target.
To further validate \methodname{} under a unified semantic feature space and at larger scale, we additionally use 4 large training datasets including ArXiv\_2023~\cite{he2023explanation}, ogbn-Arxiv~\cite{wang2020ogb}, Reddit~\cite{huang2024reddit}, and PubMed~\cite{namata2012query}, and 4 test datasets including Cora~\cite{sen2008collective}, Ele-Computers and Books-History~\cite{yan2023ele}, and Instagram~\cite{huang2024reddit}. 
More information about the datasets is provided in Appendix~\ref{app:dataset:lmtext}.

\paragraph{Baselines.}
As summarized in Table~\ref{tab:main}, we group all baselines into four categories based on their training and adaptation paradigms:
(1) \textbf{Task-Supervised GNNs}, 
(2) \textbf{Self-Supervised Pretraining with Fine-Tuning}, 
(3) \textbf{Prompt-based Adaptation}, and 
(4) \textbf{Graph Foundation Models}. 
More information is provided in Appendix~\ref{app:baselines}.

\paragraph{Evaluation Metrics}
We evaluate \methodname{} on two task types: node classification and link prediction.
Following the metric setting in~\cite{sun2025riemanngfm}, we use accuracy in node classification and AUC-ROC in link prediction.

\subsection{Downstream Task Performance (RQ1)}
\label{sec:performance}
We address (\textbf{RQ1}) by evaluating on diverse downstream tasks, including node classification and link prediction.

\paragraph{\textbf{Node Classification.}}
\label{sec:perf_node}
\begin{table*}[t]
	\centering
	\small
	\caption{Accuracy Evaluation on 1-shot Node Classification.}
	\label{table.node-classification}
	\resizebox{1\linewidth}{!}{
    \begin{tabular}{@{}l|c|c|c|c|c|c|c|c@{}}
    \toprule
    Methods & Wisconsin & Cornell & Citeseer & Cora & Pubmed & Computers & Photos & Texas \\\midrule
    \multicolumn{9}{@{}l@{}}{\textbf{Task-Supervised GNNs}}\\
    \midrule
    \method{GCN}
    & 17.46 $\pm$ \phantom{0}3.35
    & 19.53 $\pm$ \phantom{0}5.77
    & 26.89 $\pm$ \phantom{0}3.01
    & 31.98 $\pm$ \phantom{0}4.11
    & 44.29 $\pm$ \phantom{0}2.22
    & 39.43 $\pm$ \phantom{0}5.68
    & 50.39 $\pm$ \phantom{0}1.19
    & 18.48 $\pm$ \phantom{0}4.08
    \\
    
    \method{GAT}
    & 16.86 $\pm$ \phantom{0}3.16
    & 16.51 $\pm$ \phantom{0}8.41
    & 25.27 $\pm$ \phantom{0}2.10
    & 26.81 $\pm$ \phantom{0}4.42
    & 45.11 $\pm$ \phantom{0}5.52
    & 38.05 $\pm$ \phantom{0}4.47
    & 56.51 $\pm$ \phantom{0}2.35
    & 18.36 $\pm$ \phantom{0}4.63
    \\
    
    \method{H2GCN}
    & 20.34 $\pm$ \phantom{0}5.77
    & 29.30 $\pm$ 10.40
    & 24.47 $\pm$ \phantom{0}1.35
    & 23.61 $\pm$ \phantom{0}3.92
    & 40.97 $\pm$ \phantom{0}1.26
    & 19.29 $\pm$ \phantom{0}6.87
    & 42.95 $\pm$ \phantom{0}4.32
    & 29.94 $\pm$ 16.38
    \\
    
    \method{FAGCN}
    & 26.19 $\pm$ \phantom{0}7.93
    & 28.72 $\pm$ \phantom{0}8.12
    & 20.56 $\pm$ \phantom{0}1.02
    & 20.75 $\pm$ \phantom{0}3.85
    & 41.65 $\pm$ \phantom{0}1.46
    & 27.02 $\pm$ \phantom{0}2.97
    & 47.91 $\pm$ \phantom{0}2.17
    & 31.35 $\pm$ 14.17
    \\
    
    \midrule
    \multicolumn{9}{@{}l@{}}{\textbf{Self-Supervised Pretraining with Fine-tuning}}\\
    \midrule
    \method{DGI}
    & 28.24 $\pm$ \phantom{0}5.30
    & 29.73 $\pm$ 10.11
    & 34.52 $\pm$ \phantom{0}8.81
    & 39.88 $\pm$ \phantom{0}5.53
    & 47.10 $\pm$ \phantom{0}4.95
    & 34.59 $\pm$ \phantom{0}5.94
    & 50.66 $\pm$ \phantom{0}6.93
    & 29.73 $\pm$ \phantom{0}6.34
    \\
    
    \method{GraphCL}
    & 32.16 $\pm$ \phantom{0}8.83
    & 31.35 $\pm$ 12.47
    & 27.56 $\pm$ \phantom{0}5.01
    & 35.26 $\pm$ \phantom{0}8.62
    & 43.66 $\pm$ \phantom{0}3.12
    & 34.40 $\pm$ \phantom{0}6.93
    & 48.86 $\pm$ \phantom{0}5.53
    & 31.35 $\pm$ 22.66
    \\
    
    \method{GraphACL}
    & 29.41 $\pm$ 11.09
    & 25.41 $\pm$ \phantom{0}9.86
    & 28.34 $\pm$ \phantom{0}4.19
    & 36.96 $\pm$ \phantom{0}8.71
    & 38.20 $\pm$ 10.11
    & 44.42 $\pm$ \phantom{0}7.06
    & 56.60 $\pm$ \phantom{0}8.90
    & 24.86 $\pm$ \phantom{0}9.44
    \\

    {\method{MixHop}} & {29.82 $\pm$ \phantom{0}9.32} & {17.14 $\pm$ \phantom{0}5.54} & {21.35 $\pm$ \phantom{0}3.67} & {21.46 $\pm$ \phantom{0}5.95} & {43.57 $\pm$ \phantom{0}4.05} & {29.77 $\pm$ 15.57} & {32.77 $\pm$ \phantom{0}9.26} & {26.09 $\pm$ 15.93} \\
    
    \midrule
    \multicolumn{9}{@{}l@{}}{\textbf{Prompt-based Adaptation}}\\
    \midrule
    \method{GPPT}
    & 23.53 $\pm$ 15.93
    & 25.41 $\pm$ \phantom{0}5.27
    & 33.24 $\pm$ \phantom{0}8.73
    & 40.60 $\pm$ \phantom{0}6.00
    & 41.64 $\pm$ \phantom{0}9.29
    & 44.57 $\pm$ 13.38
    & 53.94 $\pm$ 11.93
    & 18.38 $\pm$ 12.00
    \\
    
    \method{GPF}
    & 23.53 $\pm$ \phantom{0}9.36 
    & 36.22 $\pm$ \phantom{0}9.46
    & 27.16 $\pm$ \phantom{0}4.10 
    & 34.74 $\pm$ \phantom{0}3.48 
    & 26.62 $\pm$ 10.18
    & 42.15 $\pm$ \phantom{0}8.75 
    & 59.13 $\pm$ \phantom{0}6.05 
    & 28.65 $\pm$ 17.31
    \\
    
    \method{GraphPrompt}
    & 20.39 $\pm$ \phantom{0}7.92
    & 28.65 $\pm$ \phantom{0}5.92
    & 19.80 $\pm$ \phantom{0}2.79
    & 33.68 $\pm$ \phantom{0}9.83
    & 42.94 $\pm$ \phantom{0}1.98
    & 33.34 $\pm$ \phantom{0}6.09
    & 42.60 $\pm$ \phantom{0}6.75
    & 24.32 $\pm$ \phantom{0}9.17
    \\
    
    \method{GraphPrompt+}
    & 16.86 $\pm$ \phantom{0}5.63
    & 29.19 $\pm$ \phantom{0}5.24
    & 28.28 $\pm$ \phantom{0}3.55
    & 29.76 $\pm$ \phantom{0}6.06
    & 38.74 $\pm$ \phantom{0}1.70
    & 43.01 $\pm$ \phantom{0}6.72
    & 58.23 $\pm$ \phantom{0}4.61
    & 19.46 $\pm$ \phantom{0}14.84
    \\
    
    \midrule
    \multicolumn{9}{@{}l@{}}{\textbf{Graph Foundation Models}}\\
    \midrule
    \method{GCOPE}
    & 28.69 $\pm$ \phantom{0}7.00
    & 25.00 $\pm$ \phantom{0}6.58
    & \underline{38.62 $\pm$ \phantom{0}2.58}
    & 39.05 $\pm$ \phantom{0}2.53
    & 44.52 $\pm$ \phantom{0}2.51
    & 28.50 $\pm$ \phantom{0}4.91
    & 59.17 $\pm$ \phantom{0}2.45
    & 19.50 $\pm$ 12.18
    \\
    
    \method{SAMGPT}
    & 21.57 $\pm$ 12.09
    & \underline{36.22 $\pm$ \phantom{0}5.30}
    & 28.78 $\pm$ \phantom{0}8.99
    & 43.08 $\pm$ 10.40
    & 37.95 $\pm$ \phantom{0}9.58
    & 45.80 $\pm$ \phantom{0}4.54
    & 49.80 $\pm$ 11.89
    & \underline{31.89 $\pm$ 11.26}
    \\
    
    \method{MDGFM}
    & 17.65 $\pm$ \phantom{0}7.34
    & 26.49 $\pm$ 10.73
    & 23.32 $\pm$ 15.32
    & 42.88 $\pm$ \phantom{0}9.12
    & 38.42 $\pm$ \phantom{0}7.92
    & 42.68 $\pm$ \phantom{0}2.53
    & 57.78 $\pm$ \phantom{0}9.07
    & 31.54 $\pm$ 11.78
    \\
    
    \method{RiemannGFM}
    & 31.20 $\pm$ \phantom{0}9.35
    & 31.35 $\pm$ \phantom{0}7.76
    & 31.86 $\pm$ \phantom{0}6.45
    & \underline{43.48 $\pm$ \phantom{0}5.28}
    & 44.10 $\pm$ \phantom{0}5.05
    & 46.20 $\pm$ \phantom{0}8.79
    & 57.20 $\pm$ 16.84
    & 29.12 $\pm$ \phantom{0}7.44
    \\

    {\method{GPM}} & {29.19 $\pm$ \phantom{0}4.69} & {27.83 $\pm$ 11.00} & {29.88 $\pm$ \phantom{0}6.36} & {38.96 $\pm$ \phantom{0}8.18} & {35.80 $\pm$ \phantom{0}8.20} & {48.58 $\pm$ 10.66} & {37.85 $\pm$ \phantom{0}8.17} & {27.70 $\pm$ 12.09} \\
    
    {\method{G2PM}} & {\underline{34.70 $\pm$ \phantom{0}8.73}} & {32.96 $\pm$ \phantom{0}5.44} & {31.04 $\pm$ \phantom{0}2.90} & {40.04 $\pm$ \phantom{0}3.13} & {\underline{49.71 $\pm$ \phantom{0}8.66}} & {\underline{51.92 $\pm$ 10.10}} & {\underline{59.90 $\pm$ \phantom{0}5.02}} & {30.37 $\pm$ \phantom{0}9.36} \\
    
    \midrule
    \method{\methodname{}}
    & \textbf{35.41 $\pm$ \phantom{0}7.29}
    & \textbf{36.71 $\pm$ \phantom{0}9.92}
    & \textbf{57.54 $\pm$ \phantom{0}9.49}
    & \textbf{49.50 $\pm$ \phantom{0}3.97}
    & \textbf{49.80 $\pm$ \phantom{0}5.38}
    & \textbf{52.30 $\pm$ \phantom{0}3.33}
    & \textbf{61.08 $\pm$ \phantom{0}5.26}
    & \textbf{32.36 $\pm$ 12.10}
    \\
    \bottomrule
    \end{tabular}}
	\\
	\parbox{1\textwidth}{\footnotesize Results are reported in percent. The best method is bolded, and the runner-up is underlined.}
    \label{tab:main}
\end{table*}

\begin{table*}[t]
	\centering
	\small
	\caption{Accuracy evaluation on \textbf{3-shot} node classification under the protocol in Section~\ref{sec:exp_setup}.}
	\label{tab:nodecls_3shot}
	\resizebox{1\linewidth}{!}{
    \begin{tabular}{@{}l|c|c|c|c|c|c|c|c@{}}
    \toprule
    Methods & Wisconsin & Cornell & Citeseer & Cora & Pubmed & Computers & Photos & Texas \\
    \midrule
    \multicolumn{9}{@{}l@{}}{\textbf{Task-Supervised GNNs}}\\
    \midrule
    
    \method{GCN}
    & 15.51 $\pm$ \phantom{0}5.61
    & 20.00 $\pm$ \phantom{0}2.04
    & 33.61 $\pm$ \phantom{0}4.98
    & 40.12 $\pm$ \phantom{0}6.52
    & 40.97 $\pm$ \phantom{0}2.16
    & 31.54 $\pm$ \phantom{0}5.47
    & 59.05 $\pm$ \phantom{0}4.60
    & 21.95 $\pm$ 11.22
    \\
    
    \method{GAT}
    & 18.94 $\pm$ \phantom{0}1.99
    & 15.46 $\pm$ \phantom{0}5.51
    & 22.07 $\pm$ \phantom{0}8.64
    & 29.20 $\pm$ \phantom{0}3.13
    & 34.70 $\pm$ \phantom{0}7.32
    & 38.92 $\pm$ \phantom{0}9.21
    & 47.20 $\pm$ \phantom{0}4.62
    & 28.05 $\pm$ 12.06
    \\
    
    \method{H2GCN}
    & 21.32 $\pm$ 13.55
    & 31.17 $\pm$ 11.47
    & 23.29 $\pm$ \phantom{0}2.97
    & 24.94 $\pm$ \phantom{0}6.05
    & 42.95 $\pm$ \phantom{0}3.02
    & 13.77 $\pm$ \phantom{0}4.30
    & 35.97 $\pm$ \phantom{0}4.50
    & 16.95 $\pm$ \phantom{0}1.36
    \\
    
    \method{FAGCN}
    & 31.32 $\pm$ \phantom{0}5.57
    & 31.90 $\pm$ 11.89
    & 20.15 $\pm$ \phantom{0}0.65
    & 19.75 $\pm$ \phantom{0}6.68
    & 40.09 $\pm$ \phantom{0}1.09
    & 25.51 $\pm$ \phantom{0}5.75
    & 42.35 $\pm$ \phantom{0}2.53
    & 26.95 $\pm$ \phantom{0}8.06
    \\
    
    \midrule
    \multicolumn{9}{@{}l@{}}{\textbf{Self-Supervised Pretraining with Fine-tuning}}\\
    \midrule
    \method{DGI}
    & 36.08 $\pm$ \phantom{0}3.82
    & 30.27 $\pm$ 14.48
    & 45.26 $\pm$ \phantom{0}2.85
    & 58.26 $\pm$ \phantom{0}2.61
    & 52.94 $\pm$ \phantom{0}7.65
    & 45.15 $\pm$ \phantom{0}6.63
    & 68.87 $\pm$ \phantom{0}9.46
    & 40.00 $\pm$ \phantom{0}5.54
    \\
    
    \method{GraphCL}
    & 36.08 $\pm$ 10.61
    & 44.32 $\pm$ 13.46
    & 32.02 $\pm$ \phantom{0}3.48
    & 45.30 $\pm$ \phantom{0}6.52
    & 44.20 $\pm$ \phantom{0}1.53
    & 49.62 $\pm$ \phantom{0}7.42
    & 62.69 $\pm$ \phantom{0}0.12
    & 41.62 $\pm$ 10.22
    \\
    
    \method{GraphACL}
    & 27.45 $\pm$ \phantom{0}6.04
    & 27.03 $\pm$ \phantom{0}6.04
    & 38.14 $\pm$ \phantom{0}5.59
    & 57.10 $\pm$ \phantom{0}4.72
    & 45.04 $\pm$ \phantom{0}3.37
    & 46.23 $\pm$ \phantom{0}8.33
    & 66.85 $\pm$ \phantom{0}3.98
    & 37.84 $\pm$ \phantom{0}9.56
    \\
    
    \midrule
    \multicolumn{9}{@{}l@{}}{\textbf{Prompt-based Adaptation}}\\
    \method{GPPT}
    & 28.63 $\pm$ \phantom{0}3.28
    & 35.68 $\pm$ \phantom{0}9.63
    & 43.66 $\pm$ \phantom{0}6.31
    & 39.22 $\pm$ \phantom{0}7.75
    & 38.06 $\pm$ 10.37
    & 49.70 $\pm$ \phantom{0}6.78
    & 56.04 $\pm$ \phantom{0}8.56
    & 39.46 $\pm$ 15.94
    \\
    
    \method{GPF}
    & 34.12 $\pm$ 10.35
    & 28.65 $\pm$ \phantom{0}5.01
    & 33.24 $\pm$ \phantom{0}6.79
    & 46.22 $\pm$ \phantom{0}4.83
    & 34.80 $\pm$ 14.11
    & 42.67 $\pm$ \phantom{0}7.08
    & 60.57 $\pm$ \phantom{0}8.19
    & \underline{41.62 $\pm$ \phantom{0}3.24}
    \\
    
    \method{GraphPrompt}
    & 29.02 $\pm$ \phantom{0}6.71
    & 32.43 $\pm$ \phantom{0}10.81
    & 28.94 $\pm$ \phantom{0}2.62
    & 44.42 $\pm$ \phantom{0}5.20
    & 51.44 $\pm$ \phantom{0}3.68
    & 42.62 $\pm$ \phantom{0}7.13
    & 57.73 $\pm$ \phantom{0}7.06
    & 30.27 $\pm$ \phantom{0}5.86
    \\
    
    \method{GraphPrompt+}
    & 27.06 $\pm$ \phantom{0}4.00
    & 31.35 $\pm$ \phantom{0}11.79
    & 35.14 $\pm$ \phantom{0}5.81
    & 46.12 $\pm$ \phantom{0}5.10
    & 52.16 $\pm$ \phantom{0}4.24
    & 54.43 $\pm$ \phantom{0}4.53
    & 64.42 $\pm$ \phantom{0}7.58
    & 34.59 $\pm$ \phantom{0}9.58
    \\
    
    \midrule
    \multicolumn{9}{@{}l@{}}{\textbf{Graph Foundation Models}}\\
    \midrule
    \method{GCOPE}
    & 38.21 $\pm$ \phantom{0}4.22
    & 23.44 $\pm$ 12.42
    & 37.52 $\pm$ \phantom{0}1.48
    & \underline{58.71 $\pm$ \phantom{0}1.96}
    & \underline{55.26 $\pm$ \phantom{0}1.93}
    & 53.01 $\pm$ \phantom{0}3.31
    & 68.15 $\pm$ \phantom{0}2.54
    & 34.90 $\pm$ 11.04
    \\
    
    \method{SAMGPT}
    & 30.98 $\pm$ \phantom{0}7.78
    & 29.73 $\pm$ 11.21
    & 46.46 $\pm$ 10.02
    & 45.76 $\pm$ \phantom{0}3.84
    & 48.26 $\pm$ \phantom{0}6.58
    & \underline{55.27 $\pm$ \phantom{0}6.04}
    & 62.98 $\pm$ \phantom{0}8.16
    & 25.95 $\pm$ 15.06
    \\
    
    \method{MDGFM}
    & 34.90 $\pm$ \phantom{0}6.37
    & 29.73 $\pm$ 10.40
    & \underline{47.72 $\pm$ \phantom{0}4.63}
    & 57.52 $\pm$ \phantom{0}6.12
    & 42.86 $\pm$ \phantom{0}4.88
    & 48.44 $\pm$ 10.53
    & 62.84 $\pm$ \phantom{0}4.87
    & 35.14 $\pm$ 11.72
    \\
    
    \method{RiemannGFM}
    & \underline{40.40 $\pm$ \phantom{0}5.71}
    & \underline{44.86 $\pm$ \phantom{0}8.65}
    & 47.44 $\pm$ \phantom{0}7.81
    & 58.20 $\pm$ \phantom{0}2.10
    & 53.60 $\pm$ 13.59
    & 40.77 $\pm$ \phantom{0}6.98
    & \underline{70.82 $\pm$ \phantom{0}3.10}
    & 33.51 $\pm$ 11.29
    \\
    
    \midrule
    \method{\methodname{}}
    & \textbf{43.10 $\pm$ \phantom{0}4.03}
    & \textbf{45.39 $\pm$ \phantom{0}3.09}
    & \textbf{73.98 $\pm$ \phantom{0}1.42}
    & \textbf{59.26 $\pm$ \phantom{0}5.04}
    & \textbf{59.19 $\pm$ \phantom{0}3.01}
    & \textbf{56.02 $\pm$ \phantom{0}3.90}
    & \textbf{71.50 $\pm$ \phantom{0}2.80}
    & \textbf{44.18 $\pm$ \phantom{0}7.02}
    \\
    \bottomrule
    \end{tabular}}
	\\
	\parbox{1\textwidth}{\footnotesize Results are reported in percent. The best method is bolded, and the runner-up is underlined.}
\end{table*}

\begin{table*}
	\centering
	\small
	\caption{Accuracy evaluation on 5-shot node classification}
	\label{tab:nodecls_5shot}
	\resizebox{1\linewidth}{!}{
    \begin{tabular}{@{}l|c|c|c|c|c|c|c|c@{}}
    \toprule
    Methods & Wisconsin & Cornell & Citeseer & Cora & Pubmed & Computers & Photos & Texas \\
    \midrule
    \multicolumn{9}{@{}l@{}}{\textbf{Task-Supervised GNNs}}\\
    \midrule
    
    \method{GCN}
    & 27.00 $\pm$ \phantom{0}6.60
    & 20.92 $\pm$ \phantom{0}6.02
    & 31.61 $\pm$ \phantom{0}6.05
    & 39.48 $\pm$ \phantom{0}6.08
    & 38.82 $\pm$ \phantom{0}2.98
    & 30.06 $\pm$ 10.72
    & 66.96 $\pm$ \phantom{0}9.21
    & 38.08 $\pm$ 13.65
    \\
    
    \method{GAT}
    & 28.85 $\pm$ \phantom{0}8.91
    & 18.30 $\pm$ \phantom{0}7.49
    & 23.99 $\pm$ \phantom{0}4.10
    & 29.82 $\pm$ \phantom{0}2.31
    & 34.92 $\pm$ 12.27
    & 37.12 $\pm$ \phantom{0}8.21
    & 54.37 $\pm$ \phantom{0}4.90
    & 28.21 $\pm$ 11.78
    \\
    
    \method{H2GCN}
    & 28.39 $\pm$ 13.19
    & 35.82 $\pm$ 10.30
    & 22.29 $\pm$ \phantom{0}3.74
    & 20.83 $\pm$ \phantom{0}6.56
    & 42.61 $\pm$ \phantom{0}3.61
    & 19.67 $\pm$ 11.45
    & 47.07 $\pm$ \phantom{0}3.78
    & 32.82 $\pm$ 14.56
    \\
    
    \method{FAGCN}
    & 24.61 $\pm$ \phantom{0}4.70
    & 25.36 $\pm$ \phantom{0}9.00
    & 18.69 $\pm$ \phantom{0}1.31
    & 15.82 $\pm$ \phantom{0}3.01
    & 35.92 $\pm$ \phantom{0}7.71
    & 28.68 $\pm$ \phantom{0}1.89
    & 50.83 $\pm$ \phantom{0}3.79
    & 46.15 $\pm$ \phantom{0}9.94
    \\
    
    \midrule
    \multicolumn{9}{@{}l@{}}{\textbf{Self-Supervised Pretraining with Fine-tuning}}\\
    \midrule
    \method{DGI}
    & 39.61 $\pm$ \phantom{0}3.77
    & 38.92 $\pm$ \phantom{0}6.78
    & 53.66 $\pm$ \phantom{0}6.21
    & \underline{63.72 $\pm$ \phantom{0}2.23}
    & 57.02 $\pm$ \phantom{0}3.47
    & 52.03 $\pm$ \phantom{0}5.29
    & 69.74 $\pm$ \phantom{0}4.95
    & 42.70 $\pm$ \phantom{0}9.04
    \\
    
    \method{GraphCL}
    & 39.61 $\pm$ \phantom{0}8.25
    & 42.70 $\pm$ \phantom{0}5.54
    & 39.26 $\pm$ \phantom{0}3.63
    & 49.34 $\pm$ \phantom{0}4.94
    & 44.88 $\pm$ \phantom{0}2.40
    & 57.78 $\pm$ \phantom{0}3.96
    & 69.39 $\pm$ \phantom{0}6.46
    & 47.03 $\pm$ 13.05
    \\
    
    \method{GraphACL}
    & 30.98 $\pm$ \phantom{0}6.56
    & 29.73 $\pm$ \phantom{0}1.91
    & 44.68 $\pm$ \phantom{0}6.36
    & 62.90 $\pm$ \phantom{0}5.69
    & 50.10 $\pm$ \phantom{0}5.39
    & 51.60 $\pm$ \phantom{0}6.46
    & 65.22 $\pm$ \phantom{0}8.34
    & 34.59 $\pm$ \phantom{0}7.74
    \\
    
    \midrule
    \multicolumn{9}{@{}l@{}}{\textbf{Prompt-based Adaptation}}\\
    \method{GPPT}
    & 35.29 $\pm$ \phantom{0}7.72
    & 44.86 $\pm$ \phantom{0}4.52
    & 45.98 $\pm$ \phantom{0}5.97
    & 48.78 $\pm$ \phantom{0}5.49
    & 48.88 $\pm$ \phantom{0}8.25
    & 30.14 $\pm$ 12.77
    & 64.89 $\pm$ \phantom{0}2.38
    & \underline{48.11 $\pm$ 10.71}
    \\
    
    \method{GPF}
    & 33.73 $\pm$ \phantom{0}3.37
    & 33.51 $\pm$ \phantom{0}6.96
    & 38.34 $\pm$ \phantom{0}6.35
    & 49.98 $\pm$ \phantom{0}5.98
    & 43.00 $\pm$ \phantom{0}3.48 
    & 47.46 $\pm$ \phantom{0}6.54
    & 70.89 $\pm$ \phantom{0}3.81
    & 30.81 $\pm$ 12.86
    \\
    
    \method{GraphPrompt}
    & 31.37 $\pm$ \phantom{0}3.10
    & 35.14 $\pm$ \phantom{0}8.55
    & 35.96 $\pm$ \phantom{0}1.97
    & 48.24 $\pm$ \phantom{0}3.43
    & 58.56 $\pm$ \phantom{0}2.40
    & 48.43 $\pm$ \phantom{0}4.50
    & 64.34 $\pm$ \phantom{0}3.99
    & 27.03 $\pm$ \phantom{0}3.82
    \\
    
    \method{GraphPrompt+}
    & 27.84 $\pm$ \phantom{0}5.46
    & 32.43 $\pm$ 10.11
    & 41.60 $\pm$ \phantom{0}5.29
    & 51.62 $\pm$ \phantom{0}2.00
    & 58.76 $\pm$ \phantom{0}6.22
    & 57.04 $\pm$ \phantom{0}5.29
    & 69.34 $\pm$ \phantom{0}2.48
    & 45.95 $\pm$ 10.81
    \\
    
    \midrule
    \multicolumn{9}{@{}l@{}}{\textbf{Graph Foundation Models}}\\
    \midrule
    \method{GCOPE}
    & 34.38 $\pm$ \phantom{0}3.93
    & 35.92 $\pm$ \phantom{0}9.30
    & 47.32 $\pm$ \phantom{0}1.59
    & 57.33 $\pm$ \phantom{0}2.08
    & 60.17 $\pm$ \phantom{0}1.81
    & 55.76 $\pm$ \phantom{0}3.86
    & 68.28 $\pm$ \phantom{0}2.45
    & 26.34 $\pm$ 10.63
    \\
    
    \method{SAMGPT}
    & 39.22 $\pm$ \phantom{0}9.20
    & 41.08 $\pm$ 11.13
    & \underline{59.44 $\pm$ \phantom{0}5.70}
    & 62.76 $\pm$ \phantom{0}3.93
    & 47.08 $\pm$ \phantom{0}9.10
    & \underline{58.81 $\pm$ \phantom{0}7.14}
    & 72.08 $\pm$ \phantom{0}4.20
    & 31.35 $\pm$ \phantom{0}6.07
    \\
    
    \method{MDGFM}
    & 36.08 $\pm$ \phantom{0}7.40
    & 41.08 $\pm$ \phantom{0}7.72
    & 57.08 $\pm$ \phantom{0}3.71
    & 61.86 $\pm$ \phantom{0}4.40
    & 48.84 $\pm$ \phantom{0}9.86
    & 49.68 $\pm$ \phantom{0}4.93
    & 70.37 $\pm$ \phantom{0}7.36
    & 21.08 $\pm$ \phantom{0}5.77
    \\
    
    \method{RiemannGFM}
    & \underline{40.40 $\pm$ \phantom{0}8.71}
    & \underline{44.86 $\pm$ \phantom{0}4.05}
    & 54.68 $\pm$ \phantom{0}3.67
    & 62.92 $\pm$ \phantom{0}5.67
    & \underline{61.48 $\pm$ \phantom{0}6.54}
    & 48.93 $\pm$ \phantom{0}5.75
    & \underline{73.47 $\pm$ \phantom{0}6.42}
    & 41.08 $\pm$ 12.01
    \\
    
    \midrule
    \method{\methodname{}}
    & \textbf{47.75 $\pm$ \phantom{0}4.97}
    & \textbf{47.97 $\pm$ \phantom{0}2.82}
    & \textbf{74.59$\pm$ \phantom{0}1.87} 
    & \textbf{63.77 $\pm$ \phantom{0}2.47}
    & \textbf{63.39 $\pm$ \phantom{0}1.85}
    & \textbf{59.55 $\pm$ \phantom{0}4.17}
    & \textbf{73.62 $\pm$ \phantom{0}1.15}
    & \textbf{51.64 $\pm$ \phantom{0}6.31}
    \\
    \bottomrule
    \end{tabular}}
	\\
	\parbox{1\textwidth}{\footnotesize Results are reported in percent. The best method is bolded, and the runner-up is underlined.}
\end{table*}

We report the 1-shot results in Table~\ref{table.node-classification}.  The 3-shot and 5-shot results are in Tables~\ref{tab:nodecls_3shot} and \ref{tab:nodecls_5shot}.
Under different shot settings, \methodname{} consistently achieves the best performance on all datasets, demonstrating strong cross-domain generalization.
Notably, \methodname{} yields large gains on citation networks and WebKB-style graphs (e.g., 18.9\% improvement on Citeseer).

We additionally evaluate \methodname{} in a large-scale setting using unified language features described in Section~\ref{sec:exp_setup}.
The results are summarized in Table~\ref{tab:lm_unified}.
Overall, \methodname{} remains the best performance under this setting. 

\begin{table}[t]
    \centering
    \scriptsize
    \setlength{\tabcolsep}{3pt}
    \renewcommand{\arraystretch}{1.05}
    \caption{Cross-domain node classification under the LM unified-feature setting.}
    \label{tab:lm_unified}
    \resizebox{\columnwidth}{!}{
    \begin{tabular}{@{}l|cccc@{}}
        \toprule
        Methods & Books-History & Ele-Computers & Cora & Instagram \\
        \midrule
        \method{GCOPE}              & \underline{30.67 $\pm$ \phantom{0}7.90} & 23.10 $\pm$ 3.38 & 37.80 $\pm$ 2.58 & \underline{59.17 $\pm$ 7.40} \\
        \method{SAMGPT}             & OOM & OOM & OOM & OOM \\
        \method{MDGFM}              & OOM & OOM & OOM & OOM \\
        \method{RiemannGFM}         & 17.66 $\pm$  \phantom{0}9.12 & \underline{24.56 $\pm$8.97} & \underline{50.92 $\pm$ 3.58} & 46.60 $\pm$  9.87 \\
        \midrule
        \method{\methodname{}}      & \textbf{37.05 $\pm$ 11.06} & \textbf{30.73 $\pm$ 5.31} & \textbf{61.78 $\pm$ 7.77} & \textbf{59.49 $\pm$ 4.46} \\
        \bottomrule
    \end{tabular}}
\end{table}

\paragraph{\textbf{Link Prediction}}
\label{sec:perf_lp}

We evaluate on link prediction via full-parameter fine-tuning in Table~\ref{tab:linkpred_main}, following the task setting in~\cite{sun2025riemanngfm}.

More results are in Tables~\ref{tab:linkpred_full} of  Appendix~\ref{app:linkpred_full}.
Overall, \methodname{} achieves consistently strong AUC-ROC on these representative datasets.

\begin{table}[t]
    \centering
    \scriptsize
    \setlength{\tabcolsep}{3pt}
    \renewcommand{\arraystretch}{1.05}
    \caption{Link prediction performance on representative datasets (\textbf{AUC-ROC}, \%).}
    \label{tab:linkpred_main}
    \resizebox{\columnwidth}{!}{
    \begin{tabular}{@{}l|cccc@{}}
        \toprule
        Methods & Cora & Pubmed & Photos & Texas \\
        \midrule
        \method{GCOPE}      & \underline{88.71$\pm$0.57} & 84.99 $\pm$ 0.06 & 78.45$\pm$5.51 & \underline{80.51$\pm$1.75} \\
        \method{SAMGPT}     & 88.62$\pm$3.51 & 84.92 $\pm$ 0.40 & 80.59$\pm$0.33 & 51.20$\pm$4.42 \\
        \method{MDGFM}      & 88.41$\pm$2.12 & 84.78 $\pm$ 2.09 & 81.24$\pm$5.80 & 65.22$\pm$4.05 \\
        \method{RiemannGFM} & 88.54$\pm$0.11 & \underline{86.37 $\pm$ 0.33} & \underline{81.44$\pm$0.95} & 72.41$\pm$5.96 \\
        \midrule
        \method{\methodname{}} & \textbf{89.27$\pm$0.64} & \textbf{88.66$\pm$5.16} & \textbf{81.53$\pm$0.83} & \textbf{87.94$\pm$0.96} \\
        \bottomrule
    \end{tabular}}
\end{table}

\subsection{Robustness to Perturbations (RQ2)}
\label{sec:robustness}
To answer \textbf{RQ2}, we evaluate \methodname{} under two graph perturbations: edge drop and node masking.
Based on Appendix~\ref{app:linkpred_full}, we apply perturbations only at the evaluation phase and report how performance changes as the perturbation magnitude increases.

\paragraph{Edge Drop.}
We randomly remove edges with drop rate $p \in \{0, 0.1, 0.2, 0.3, 0.4, 0.5\}$.
In Figure~\ref{fig:robust_edge_drop}, \methodname{} consistently achieves the best accuracy across all perturbation levels
and exhibits the slowest degradation rate as $p$ increases.
For example, at $p=0.5$, \methodname{} retains 46.91\% (a drop of 5.76\% points from $p=0$), while the strongest baseline drops to around 41.2\%.

\paragraph{Node Masking.}
We randomly mask node attributes with masking ratio $r \in \{0, 0.1, 0.2, 0.3, 0.4, 0.5\}$ by setting the corresponding feature vectors to zero.
Figure~\ref{fig:robust_node_mask} shows a similar trend: \methodname{} maintains stable performance under increasing attribute missingness.
At $r=0.5$, \methodname{} retains 46.84\% accuracy, whereas the strongest baseline drops to 41.78\%, resulting in a 5.06\% advantage.

\begin{figure}[t]
  \centering
  \captionsetup[subfigure]{skip=-2pt}
  \includegraphics[width=\linewidth]{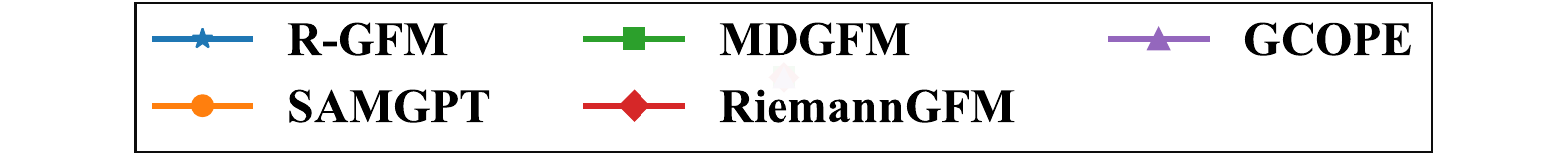}\vspace{-0.1em}
  \begin{subfigure}[t]{0.495\linewidth}
    \centering
    \includegraphics[width=\linewidth]{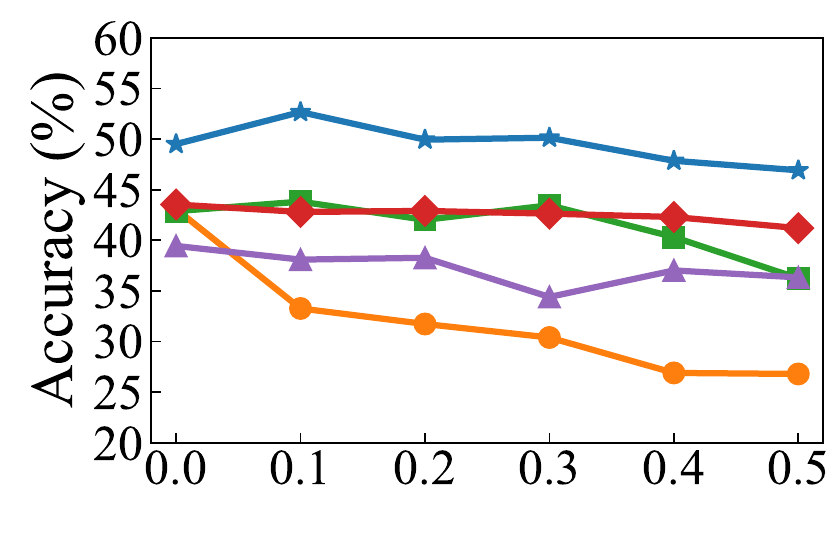}
    \caption{Edge drop.}
    \label{fig:robust_edge_drop}
  \end{subfigure}\hspace{-0.4em}
  \begin{subfigure}[t]{0.495\linewidth}
    \centering
    \includegraphics[width=\linewidth]{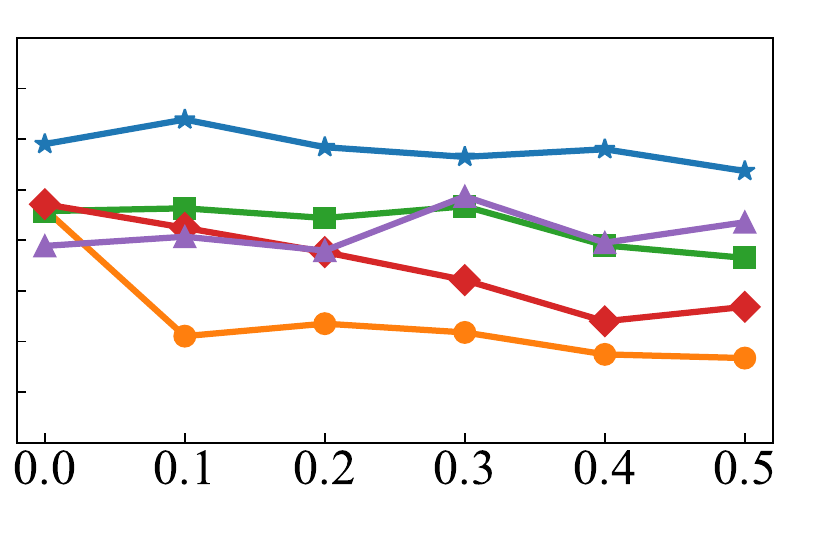}
    \caption{Node masking.}
    \label{fig:robust_node_mask}
  \end{subfigure}
  \caption{Robustness to structural and attribute perturbations.}
  \label{fig:robustness}
\end{figure}
\begin{figure}[t]
    \centering
    \includegraphics[width=0.95\linewidth]{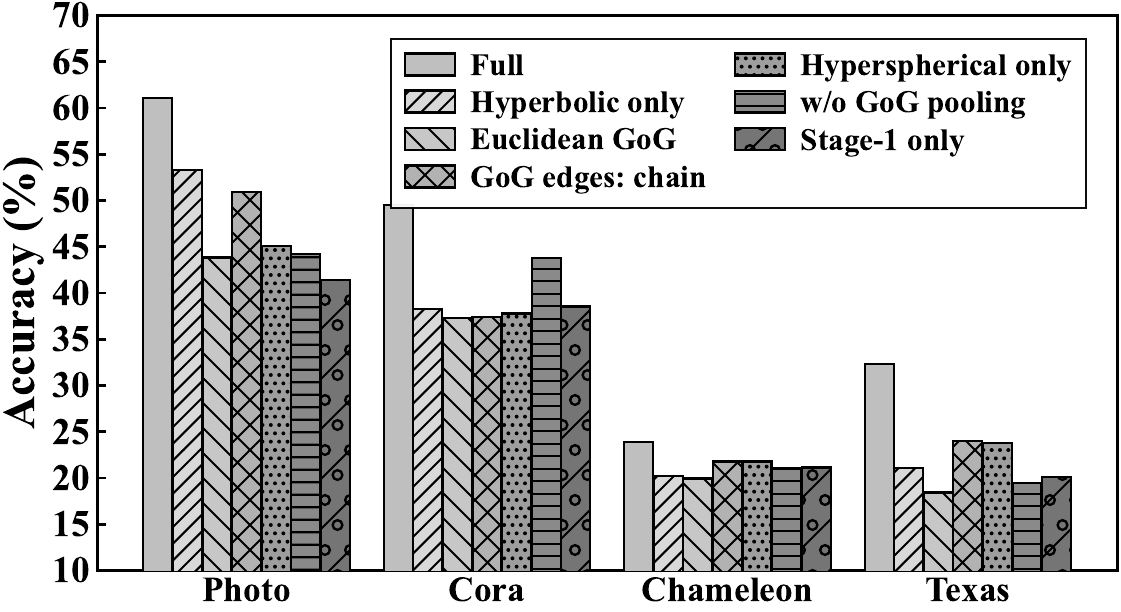}
    \caption{Ablation results on four representative datasets (Photo, Cora, Chameleon, Texas) under 1-shot node classification.}
    \label{fig:ablation_bar}
\end{figure}
\subsection{Ablation Study (RQ3)}
\label{sec:ablation}
To answer \textbf{RQ3}, we ablate the two key design choices of \methodname{}: (i) constructing the GoG, and (ii) employing Riemannian encoders on the GoG.
We report results on four datasets, each from a different domain, to examine the effect across diverse domains.
As shown in Figure~\ref{fig:ablation_bar}, removing or weakening the GoG module (Stage-1 only, w/o GoG pooling, or using a GoG topology with chain edges) consistently hurts performance, verifying the effectiveness of GoG construction.
Moreover, replacing the Riemannian encoder on GoG with an Euclidean one (Euclidean GoG) leads to a clear drop across datasets, and using only one manifold (Hyperbolic only or Hyperspherical only) is inferior to the full model, indicating the benefit of Riemannian MoE.

\subsection{Scalability (RQ4)}
\label{sec:scalability}

To answer \textbf{RQ4}, we test scalability from two perspectives:
(1) scaling with structural scales (i.e., increasing $K$-hop),
and (2) scaling with graph size (i.e., larger graphs), reporting both accuracy and resource costs.

\paragraph{\textbf{Scaling with Hop Number ($K$)}}
\label{sec:scaling_k}
 We study how performance scales with the hop number $K$ on cross-domain 1-shot node classification by extending GCOPE to sample subgraphs of varying hops.
In Figure~\ref{fig:k_scaling}, simply increasing $K$ yields only marginal gains for the naive GCOPE extension,
suggesting that enlarging the receptive field alone is insufficient.
In contrast, \methodname{} benefits more consistently from larger $K$, indicating that our GoG-based multi-scale construction can better leverage additional hop information.

\begin{figure}[t]
    \centering
    \includegraphics[width=\columnwidth]{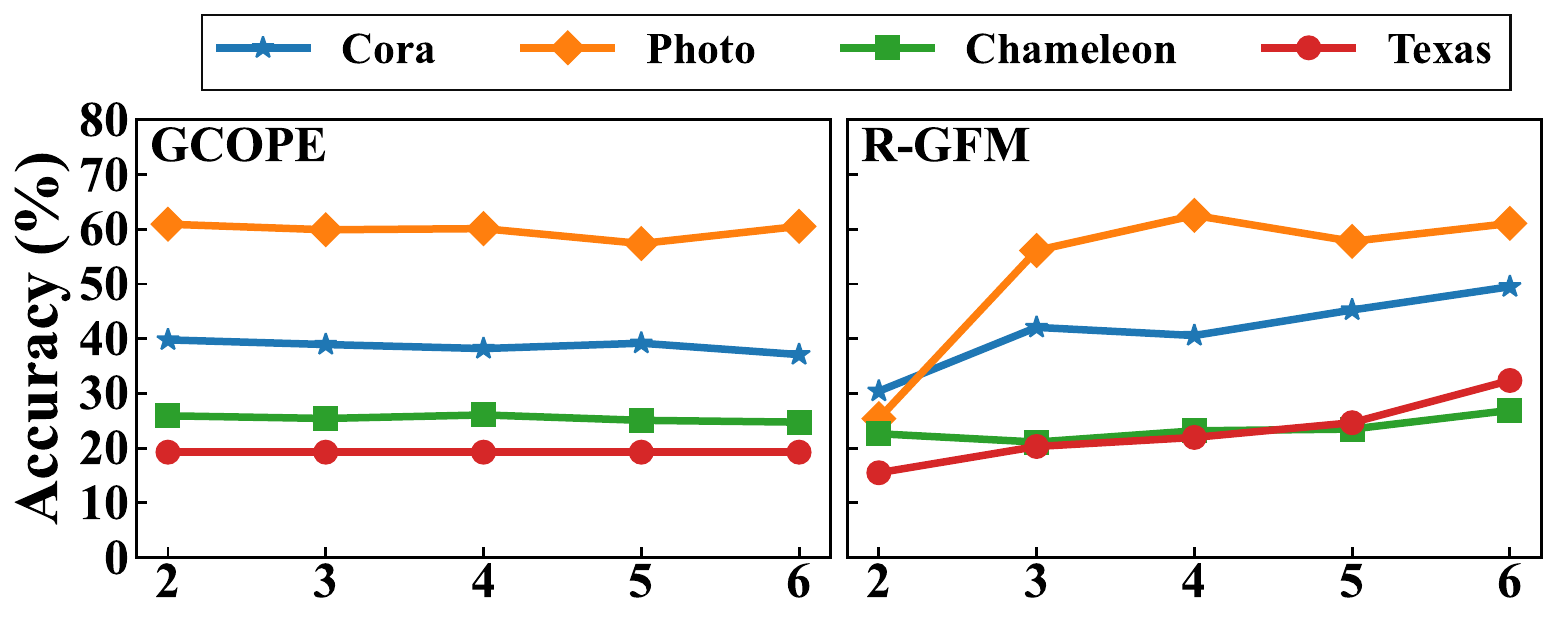}
    \caption{Performance scaling with hop number $K$: \ a naive $K$-hop extension of GCOPE vs. \methodname{} under matched settings.}
    \label{fig:k_scaling}
\end{figure}
\noindent{\textbf{Scaling with Graph Size.}} Next, we evaluate \methodname{} on larger graphs under practical budgets with results shown in Table~\ref{tab:lm_unified}. Under this setting, \method{SAMGPT} and \method{MDGFM} hit out-of-memory (OOM) due to RAM exhaustion, as expected given their full-graph training pipeline.
In contrast, \methodname{} leverages subgraph sampling, which enables it to scale well to larger graphs under the same memory budget. It also delivers the best transfer performance on the downstream datasets, demonstrating strong scalability in both efficiency and effectiveness.

\subsection{Hyperparameter Study(RQ5)}
\label{sec:hparam}
\begin{figure}[t]
  \centering
  \captionsetup[subfigure]{skip=-2pt}
  \makebox[\linewidth][c]{\hspace{0.8em}\includegraphics[width=0.98\linewidth]{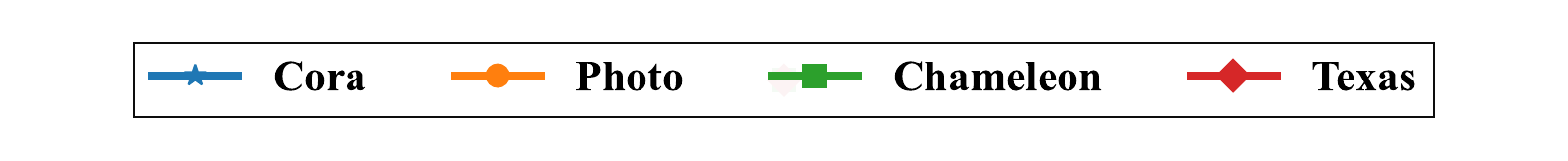}}\vspace{-0.3em}
  \begin{subfigure}[t]{0.495\linewidth}
    \centering
    \includegraphics[width=\linewidth]{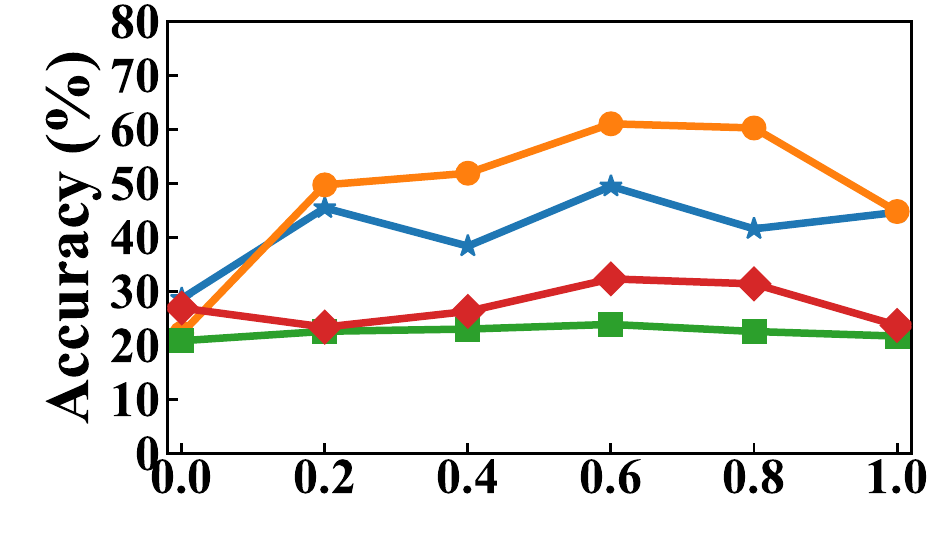}
    \caption{Sensitivity to $\mathcal{B}_{edge}$.}
    \label{fig:hparam_beta}
  \end{subfigure}\hspace{-0.4em}
  \begin{subfigure}[t]{0.495\linewidth}
    \centering
    \includegraphics[width=\linewidth]{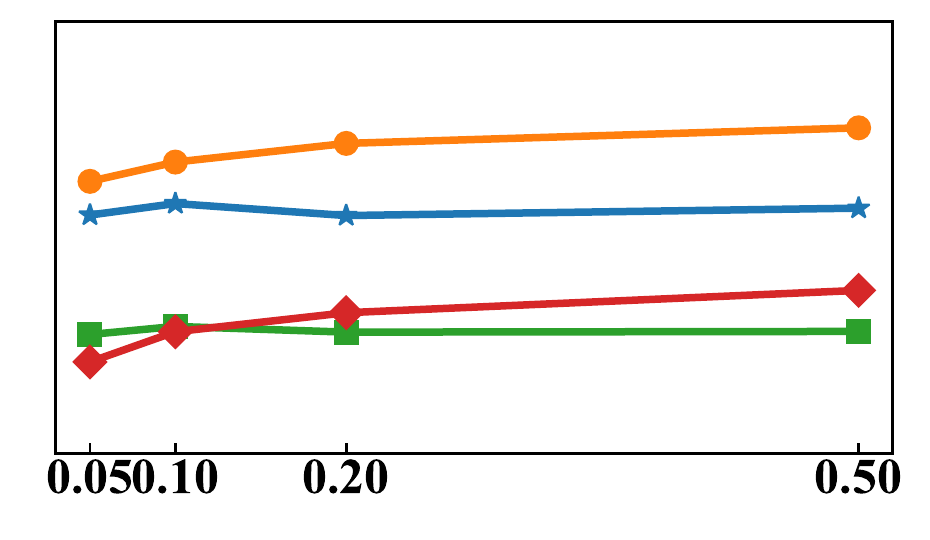}
    \caption{Sensitivity to temperature $\tau$.}
    \label{fig:hparam_temp}
  \end{subfigure}
  \caption{Hyperparameter sensitivity on four datasets. Left: varying $\mathcal{B}_{edge}$. Right: varying temperature $\tau$.}
  \label{fig:hparam}
\end{figure}
To answer \textbf{RQ5}, we conduct a hyperparameter study and report the 1-shot performance across datasets.

\paragraph{GoG Edge Budget $\mathcal{B}_{edge}$.}
Since the node number $K$ in GoG is dynamic, GoG does not have a fixed number of edges. We therefore control the edge budget by fixing the retained-edge ratio $\beta\in[0,1]$ ($\beta=\frac{2\mathcal{B}_{edge}}{K(K-1)}$). In Figure~~\ref{fig:hparam_beta}, performance peaks around $\beta\approx 0.6$. Therefore, we $\mathcal{B}_\mathrm{edge}=0.6\times \frac{K(K-1)}{2}$.
More analysis is in Appendix~\ref{app:proof_gog_edges_help_simple_1}.

\paragraph{Candidate Expert Number$\mathcal{M}$.}
\begin{figure}[t]
  \centering
  \captionsetup[subfigure]{skip=-2pt}
   \makebox[\linewidth][c]{\hspace{0.8em}\includegraphics[width=0.98\linewidth]{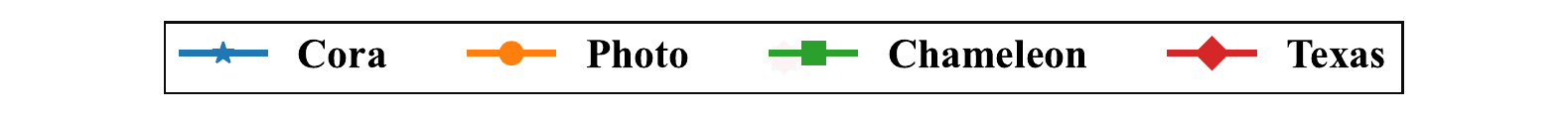}}

  \begin{subfigure}[t]{0.495\linewidth}
    \centering
    \includegraphics[width=\linewidth]{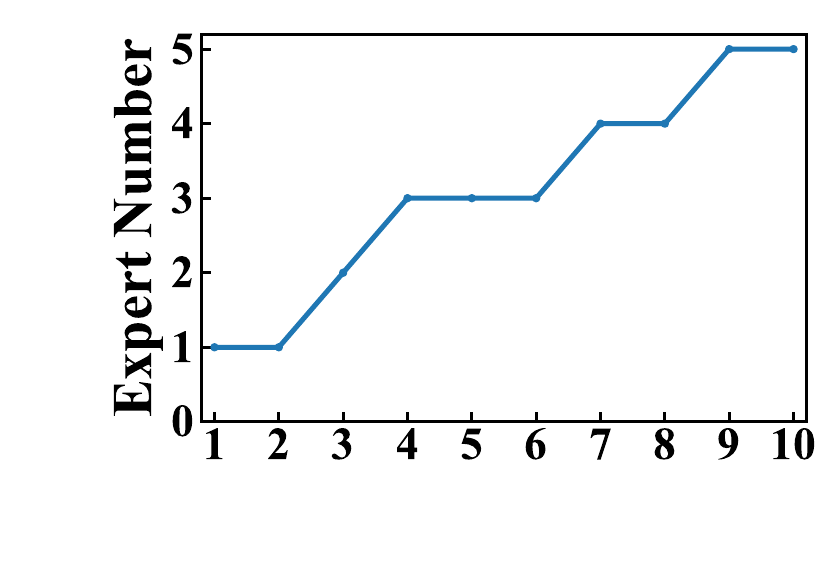}
    \caption{Experts $\mathcal{M}$ vs. Datasets $|\mathcal{D}|$.}
    \label{fig:cand_expert_vs_datasets}
  \end{subfigure}\hspace{-0.4em}
  \begin{subfigure}[t]{0.495\linewidth}
    \centering
    \includegraphics[width=\linewidth]{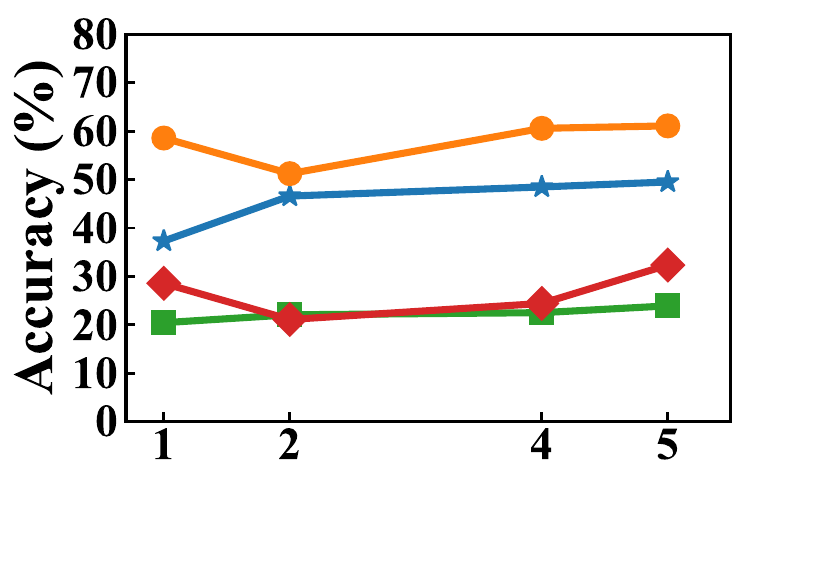}
    \caption{Accuracy vs. Experts $\mathcal{M}$.}
    \label{fig:acc_vs_experts}
  \end{subfigure}

  \caption{Effect of dataset scale and expert number. 
  }
  \label{fig:expert_scaling}
    \centering
    \includegraphics[width=\linewidth]{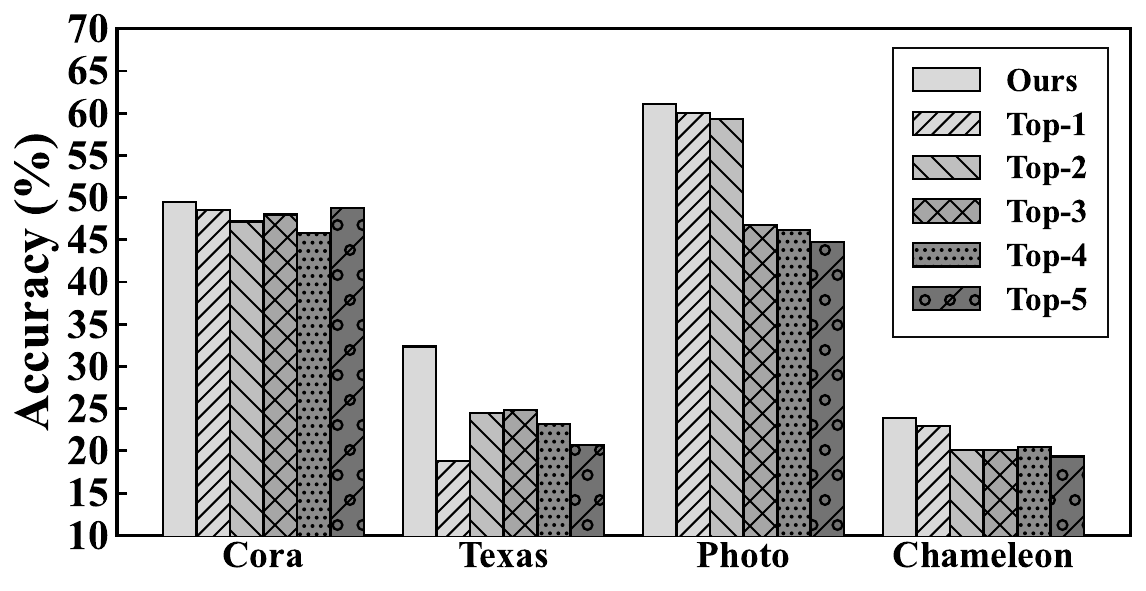}
    \caption{Effect of gate sparsity (Top-$m$) across datasets.}
    \label{fig:topm_grouped}
\end{figure}

Figure~\ref{fig:cand_expert_vs_datasets} shows that our dynamic selection strategy increases the candidate expert number as more datasets are introduced (added in the order: Wisconsin, Texas, Cornell, Cora, Citeseer, Pubmed, Computers, Photo, Chameleon, Squirrel), reflecting the growing diversity across tasks. In the full 10-dataset setting, we further vary $\mathcal{M}$ and observe the accuracy trends in Figure~\ref{fig:acc_vs_experts}; the results indicate that the best-performing $\mathcal{M}$ is dataset-dependent, and our strategy can automatically choose an appropriate expert set without manual tuning.

\noindent\textbf{Dynamic Gate Sparsity (Top-$m$).} We compare the our dynamic top-$m$ strategy with fixed Top-$m$ gating ($m\!\in\!\{1,\dots,5\}$). In Figure~\ref{fig:topm_grouped}, our method achieves the best accuracy across datasets, while fixed $m$ is more sensitive and no single choice generalizes well. 

\paragraph{Temperature $\tau$.}
Figure~\ref{fig:hparam_temp} sweeps the temperature in the contrastive objective.
We observe a reasonably stable region (e.g., $\tau\in[0.1,0.5]$), whereas overly small or overly large values can hurt performance on some datasets.
We set $\tau{=}0.2$ by default, which is competitive across datasets.
\section{Related Work}

One prevalent category of GFMs utilizes GNNs as their primary encoders, including masked graph modeling and contrastive pretraining methods (e.g., DGI~\cite{petar2019deep}, GraphCL~\cite{you2020graph}, and GraphMAE~\cite{hou2022graphmae}). More recent GNN-based GFM designs explicitly address distribution shift across graph domains by topology or structure alignment, prompt-based adaptation, or domain-aware pretraining~(e.g., GCOPE~\cite{zhao2024all}, RiemannGFM~\cite{sun2025riemanngfm}, SAMGPT~\cite{yu2025samgpt}).
However, these methods typically use fixed-hop sampling, which limits the ability to match the multi-scale structural information required by different tasks.

{A natural way to model multi-scale graph structure is to construct a Graph-of-Graphs (GoG), where each node in the higher-level graph corresponds to a graph or subgraph instance. However, existing GoG methods are mainly designed for static graph collections and are therefore not directly applicable to cross-domain graph foundation model pretraining. For example, G2GNN~\cite{wang2020gognn} constructs the GoG according to kernel similarity between graph instances, while ImbGNN~\cite{xu2024imbgnn} adopts similarity-aware random walks to extract informative subgraphs, which can introduce considerable computational overhead. In contrast, our GoG construction is dynamic and tailored to graph foundation models.}

Another line integrates GNN encoders with LLMs.
For example, OFA~\cite{liu2024ofa} trains a single model across domains and tasks via a unified prompting paradigm and language-based descriptions to bridge heterogeneous graph feature spaces.
GraphGPT~\citep{tang2024graphgpt} aligns LLMs with graph structural knowledge through graph instruction tuning and lightweight graph-text projection modules.
Such models treat the GNN encoder as a modular component, and performance depends critically on how well the GNN module captures structure information~\cite{kong2025gofa}.

R-GFM complements these GNN+LLM-based GFMs, as it can be seamlessly integrated as an advanced GNN encoder. This integration bolsters structural modeling capabilities while preserving the semantic reasoning abilities of LLMs.

\section{Conclusion}

We study graph foundation models in the challenging cross-domain setting, where graphs exhibit substantial heterogeneity in topology and geometry. 
To improve transferability, we propose \methodname{}, a Riemannian Graph-of-Graphs GFM framework that (1) constructs Graph-of-Graphs by adaptive-hop subgraphs based on similarity, and (2) adapts to diverse topologies via dynamic routing to a mixture of Riemannian experts.
Empirically, \methodname{} consistently improves cross-domain generalization across diverse benchmarks, highlighting that both geometry-aware specialization and multi-scale structural modeling are essential for robust graph transfer.

\newpage
\section*{Impact Statement}
Our work presents a new graph foundation model that aims to advance the field of Machine Learning, particularly graph machine learning.
This work is primarily foundational research. While our work may be applied in many real-world settings, we do not foresee any immediate negative societal impacts that are specific to this work.

\section*{Acknowledgements}
{This work was supported by the National Natural Science Foundation of China (NSFC) under Grant
62472400.}

\bibliography{example_paper}

@inproceedings{ChenK0H20,
  author       = {Ting Chen and
                  Simon Kornblith and
                  Mohammad Norouzi and
                  Geoffrey E. Hinton},
  title        = {A Simple Framework for Contrastive Learning of Visual Representations},
  booktitle    = {ICML},
  pages        = {1597--1607},
  year         = {2020},
}

@article{abs-2403-17404,
  author       = {Jinze Zhao and
                  Peihao Wang and
                  Zhangyang Wang},
  title        = {Generalization Error Analysis for Sparse Mixture-of-Experts: {A} Preliminary
                  Study},
  journal      = {CoRR},
  volume       = {abs/2403.17404},
  year         = {2024},
}

@inproceedings{Keriven22,
  author       = {Nicolas Keriven},
  title        = {Not too little, not too much: a theoretical analysis of graph (over)smoothing},
  booktitle    = {NeurIPS},
  year         = {2022},
}

@inproceedings{DongK23,
  author       = {Mingze Dong and
                  Yuval Kluger},
  title        = {Towards Understanding and Reducing Graph Structural Noise for GNNs},
  booktitle    = {ICML},
  pages        = {8202--8226},
  year         = {2023},
}

@inproceedings{LukovnikovF21,
  author       = {Denis Lukovnikov and
                  Asja Fischer},
  title        = {Improving Breadth-Wise Backpropagation in Graph Neural Networks Helps
                  Learning Long-Range Dependencies},
  booktitle    = {ICML},
  pages        = {7180--7191},
  year         = {2021},
}

@inproceedings{0002Y21,
  author       = {Uri Alon and
                  Eran Yahav},
  title        = {On the Bottleneck of Graph Neural Networks and its Practical Implications},
  booktitle    = {ICLR},
  year         = {2021},
}

@inproceedings{00020GTW22,
  author       = {Kai Han and
                  Yunhe Wang and
                  Jianyuan Guo and
                  Yehui Tang and
                  Enhua Wu},
  title        = {Vision {GNN:} An Image is Worth Graph of Nodes},
  booktitle    = {NeurIPS},
  year         = {2022},
}

@article{sen2008collective,
  author       = {Prithviraj Sen and
                  Galileo Namata and
                  Mustafa Bilgic and
                  Lise Getoor and
                  Brian Gallagher and
                  Tina Eliassi{-}Rad},
  title        = {Collective Classification in Network Data},
  journal      = {{AI} Mag.},
  number       = {3},
  pages        = {93--106},
  year         = {2008}
}

@inproceedings{namata2012query,
  author       = {Zhilin Yang and
                  William W. Cohen and
                  Ruslan Salakhutdinov},
  title        = {Revisiting Semi-Supervised Learning with Graph Embeddings},
  booktitle    = {{ICML}},
  pages        = {40--48},
  year         = {2016}
}

@article{pei2020geom,
  title        = {{Geom-GCN: Geometric Graph Convolutional Networks}},
  booktitle    = {8th International Conference on Learning Representations, {ICLR} 2020,
                  Addis Ababa, Ethiopia, April 26-30, 2020},
  author={Pei, Hongbin and Wei, Bingzhe and Chang, Kevin Chen-Chuan and Lei, Yu and Yang, Bo},
  journal={arXiv preprint arXiv:2002.05287},
  year={2020}
}

@inproceedings{mcauley2015image,
  title     = {{Image-based Recommendations on Styles and Substitutes}},
  author    = {McAuley, Julian and Targett, Christopher and Shi, Qinfeng and Van Den Hengel, Anton},
  booktitle    = {{SIGIR}},
  pages     = {43--52},
  year      = {2015}
}

@article{shchur2018pitfalls,
  title={Pitfalls of graph neural network evaluation},
  author={Shchur, Oleksandr and Mumme, Maximilian and Bojchevski, Aleksandar and G{\"u}nnemann, Stephan},
  journal={arXiv preprint arXiv:1811.05868},
  year={2018}
}

@article{velickovic2017graph,
author       = {Petar Velickovic and
                  Guillem Cucurull and
                  Arantxa Casanova and
                  Adriana Romero and
                  Pietro Li{\`{o}} and
                  Yoshua Bengio},
  title        = {Graph Attention Networks},
  journal      = {CoRR},
  year         = {2017}
}

@inproceedings{bo2021beyond,
  author       = {Deyu Bo and
                  Xiao Wang and
                  Chuan Shi and
                  Huawei Shen},
  title        = {Beyond Low-frequency Information in Graph Convolutional Networks},
  booktitle    = {{AAAI}},
  pages        = {3950--3957},
  year         = {2021}
}

@inproceedings{zhu2020beyond,
  author       = {Jiong Zhu and
                  Yujun Yan and
                  Lingxiao Zhao and
                  Mark Heimann and
                  Leman Akoglu and
                  Danai Koutra},
  title        = {Beyond Homophily in Graph Neural Networks: Current Limitations and
                  Effective Designs},
  booktitle    = {{NeurIPS}},
  year         = {2020}
}

@inproceedings{petar2019deep,
  author       = {Petar Velickovic and
                  William Fedus and
                  William L. Hamilton and
                  Pietro Li{\`{o}} and
                  Yoshua Bengio and
                  R. Devon Hjelm},
  title        = {Deep Graph Infomax},
  booktitle    = {{ICLR}},
  year         = {2019}
}

@inproceedings{you2020graph,
  author       = {Yuning You and
                  Tianlong Chen and
                  Yongduo Sui and
                  Ting Chen and
                  Zhangyang Wang and
                  Yang Shen},
  title        = {Graph Contrastive Learning with Augmentations},
  booktitle    = {NeurIPS},
  year         = {2020}
}

@inproceedings{xiao2023simple,
  author       = {Teng Xiao and
                  Huaisheng Zhu and
                  Zhengyu Chen and
                  Suhang Wang},
  title        = {{Simple and Asymmetric Graph Contrastive Learning without Augmentations}},
  booktitle    = {NeurIPS},
  year         = {2023}
}

@inproceedings{sun2022gppt,
  author       = {Mingchen Sun and
                  Kaixiong Zhou and
                  Xin He and
                  Ying Wang and
                  Xin Wang},
  title        = {{GPPT:} Graph Pre-training and Prompt Tuning to Generalize Graph Neural
                  Networks},
  booktitle    = {{KDD}},
  pages        = {1717--1727},
  year         = {2022}
}

@inproceedings{liu2023graphprompt,
  author       = {Zemin Liu and
                  Xingtong Yu and
                  Yuan Fang and
                  Xinming Zhang},
  title        = {{GraphPrompt: Unifying Pre-Training and Downstream Tasks for Graph
                  Neural Networks}},
  booktitle    = {{WWW}},
  pages        = {417--428},
  year         = {2023}
}

@article{yu2024generalized,
  author       = {Xingtong Yu and
                  Zhenghao Liu and
                  Yuan Fang and
                  Zemin Liu and
                  Sihong Chen and
                  Xinming Zhang},
  title        = {{Generalized Graph Prompt: Toward a Unification of Pre-Training and
                  Downstream Tasks on Graphs}},
  pages        = {6237--6250},
  year         = {2024}
}

@inproceedings{zhao2024all,
  author       = {Haihong Zhao and
                  Aochuan Chen and
                  Xiangguo Sun and
                  Hong Cheng and
                  Jia Li},
  title        = {{All in One and One for All: {A} Simple yet Effective Method towards
                  Cross-domain Graph Pretraining}},
  booktitle    = {{KDD}},
  pages        = {4443--4454},
  year         = {2024}
}

@inproceedings{yu2025samgpt,
  author       = {Xingtong Yu and
                  Zechuan Gong and
                  Chang Zhou and
                  Yuan Fang and
                  Hui Zhang},
  title        = {{SAMGPT: Text-free Graph Foundation Model for Multi-domain Pre-training
                  and Cross-domain Adaptation}},
  booktitle    = {{WWW}},
  pages        = {1142--1153},
  year         = {2025}
}

@inproceedings{wang2025multi,
  author       = {Shuo Wang and
                  Bokui Wang and
                  Zhixiang Shen and
                  Boyan Deng and
                  Zhao Kang},
  title        = {{Multi-Domain Graph Foundation Models: Robust Knowledge Transfer via
                  Topology Alignment}},
  booktitle    = {{ICML}},
  year         = {2025}
}

@inproceedings{sun2025riemanngfm,
  author       = {Li Sun and
                  Zhenhao Huang and
                  Suyang Zhou and
                  Qiqi Wan and
                  Hao Peng and
                  Philip S. Yu},
  title        = {{RiemannGFM}: Learning a Graph Foundation Model from Riemannian Geometry},
  booktitle    = {{WWW}},
  pages        = {1154--1165},
  year         = {2025}
}

@inproceedings{fang2023gpf,
  author       = {Taoran Fang and
                  Yunchao Zhang and
                  Yang Yang and
                  Chunping Wang and
                  Lei Chen},
  title        = {Universal Prompt Tuning for Graph Neural Networks},
  booktitle    = {NeurIPS},
  year         = {2023}
}

@inproceedings{wang2020gognn,
  author       = {Hanchen Wang and
                  Defu Lian and
                  Ying Zhang and
                  Lu Qin and
                  Xuemin Lin},
  title        = {GoGNN: Graph of Graphs Neural Network for Predicting Structured Entity
                  Interactions},
  booktitle    = {{IJCAI}},
  pages        = {1317--1323},
  year         = {2020}
}

@inproceedings{gu2019learn,
  author       = {Albert Gu and
                  Frederic Sala and
                  Beliz Gunel and
                  Christopher R{\'{e}}},
  title        = {Learning Mixed-Curvature Representations in Product Spaces},
  booktitle    = {{ICLR}},
  year         = {2019}
}

@inproceedings{zhang2018link,
  title={Link prediction based on graph neural networks},
  author={Zhang, Muhan and Chen, Yixin},
  pages={5165--5175},
  booktitle    = {NeuIPS},
  year={2018}
}

@inproceedings{zhen2023gog,
  author       = {Zhiwei Zhen and
                  Yuzhou Chen and
                  Murat Kantarcioglu and
                  Yulia R. Gel},
  title        = {Graph of Graphs: {A} New Knowledge Representation Mechanism for Graph
                  Learning (Student Abstract)},
  booktitle    = {{AAAI}},
  pages        = {16386--16387},
  year         = {2023}
}

@article{wu2018moleculenet,
  author       = {Zhenqin Wu and
                  Bharath Ramsundar and
                  Evan N. Feinberg and
                  Joseph Gomes and
                  Caleb Geniesse and
                  Aneesh S. Pappu and
                  Karl Leswing and
                  Vijay S. Pande},
  title        = {MoleculeNet: {A} Benchmark for Molecular Machine Learning},
  journal      = {CoRR},
  year         = {2017}
}

@inproceedings{gilmer2017neural,
  author       = {Justin Gilmer and
                  Samuel S. Schoenholz and
                  Patrick F. Riley and
                  Oriol Vinyals and
                  George E. Dahl},
  title        = {Neural Message Passing for Quantum Chemistry},
  booktitle    = {{ICML}},
  pages        = {1263--1272},
  year         = {2017}
}

@inproceedings{ying2018pinsage,
  author       = {Rex Ying and
                  Ruining He and
                  Kaifeng Chen and
                  Pong Eksombatchai and
                  William L. Hamilton and
                  Jure Leskovec},
  title        = {Graph Convolutional Neural Networks for Web-Scale Recommender Systems},
  booktitle    = {{KDD}},
  pages        = {974--983},
  year         = {2018}
}

@inproceedings{he2020lightgcn,
  author       = {Xiangnan He and
                  Kuan Deng and
                  Xiang Wang and
                  Yan Li and
                  Yong{-}Dong Zhang and
                  Meng Wang},
  title        = {LightGCN: Simplifying and Powering Graph Convolution Network for Recommendation},
  booktitle    = {{SIGIR}},
  pages        = {639--648},
  year         = {2020}
}

@inproceedings{schlichtkrull2018rgcn,
  author       = {Michael Sejr Schlichtkrull and
                  Thomas N. Kipf and
                  Peter Bloem and
                  Rianne van den Berg and
                  Ivan Titov and
                  Max Welling},
  title        = {Modeling Relational Data with Graph Convolutional Networks},
  booktitle    = {{ESWC}},
  pages        = {593--607},
  year         = {2018}
}

@inproceedings{li2018dcrnn,
  author       = {Yaguang Li and
                  Rose Yu and
                  Cyrus Shahabi and
                  Yan Liu},
  title        = {Diffusion Convolutional Recurrent Neural Network: Data-Driven Traffic
                  Forecasting},
  booktitle    = {{ICLR}},
  year         = {2018}
}

@inproceedings{hou2022graphmae,
  author       = {Zhenyu Hou and
                  Xiao Liu and
                  Yukuo Cen and
                  Yuxiao Dong and
                  Hongxia Yang and
                  Chunjie Wang and
                  Jie Tang},
  title        = {GraphMAE: Self-Supervised Masked Graph Autoencoders},
  booktitle    = {{KDD}},
  pages        = {594--604},
  year         = {2022}
}

@article{wang2025gfm_survey,
  author       = {Zehong Wang and
                  Zheyuan Liu and
                  Tianyi Ma and
                  Jiazheng Li and
                  Zheyuan Zhang and
                  Xingbo Fu and
                  Yiyang Li and
                  Zhengqing Yuan and
                  Wei Song and
                  Yijun Ma and
                  Qingkai Zeng and
                  Xiusi Chen and
                  Jianan Zhao and
                  Jundong Li and
                  Meng Jiang and
                  Pietro Lio and
                  Nitesh V. Chawla and
                  Chuxu Zhang and
                  Yanfang Ye},
  title        = {Graph Foundation Models: {A} Comprehensive Survey},
  journal      = {CoRR},
  year         = {2025}
}

@article{kipf2017gcn,
  author       = {Thomas N. Kipf and
                  Max Welling},
  title        = {Semi-Supervised Classification with Graph Convolutional Networks},
  journal      = {CoRR},
  year         = {2016}
}

@article{cavallo2022ncs,
  author       = {Andrea Cavallo and
                  Claas Grohnfeldt and
                  Michele Russo and
                  Giulio Lovisotto and
                  Luca Vassio},
  title        = {2-hop Neighbor Class Similarity {(2NCS):} {A} graph structural metric
                  indicative of graph neural network performance},
  journal      = {CoRR},
  year         = {2022}
}

@inproceedings{liu2021pcgnn,
  author       = {Yang Liu and
                  Xiang Ao and
                  Zidi Qin and
                  Jianfeng Chi and
                  Jinghua Feng and
                  Hao Yang and
                  Qing He},
  title        = {Pick and Choose: {A} GNN-based Imbalanced Learning Approach for Fraud
                  Detection},
  booktitle    = {{WWW}},
  pages        = {3168--3177},
  year         = {2021}
}

@article{xia2024anygraph,
  author       = {Lianghao Xia and
                  Chao Huang},
  title        = {AnyGraph: Graph Foundation Model in the Wild},
  journal      = {CoRR},
  year         = {2024}
}

@inproceedings{egressy2024directedmultigraphs,
  author       = {B{\'{e}}ni Egressy and
                  Luc von Niederh{\"{a}}usern and
                  Jovan Blanusa and
                  Erik R. Altman and
                  Roger Wattenhofer and
                  Kubilay Atasu},
  title        = {Provably Powerful Graph Neural Networks for Directed Multigraphs},
  booktitle    = {{AAAI}},
  pages        = {11838--11846},
  year         = {2024}
}

@inproceedings{li2024noise,
  author       = {Shiyuan Li and
                  Yixin Liu and
                  Qingfeng Chen and
                  Geoffrey I. Webb and
                  Shirui Pan},
  title        = {Noise-Resilient Unsupervised Graph Representation Learning via Multi-Hop
                  Feature Quality Estimation},
  booktitle    = {{CIKM}},
  pages        = {1255--1265},
  year         = {2024}
}

@inproceedings{guo2025graphmore,
  author       = {Zihao Guo and
                  Qingyun Sun and
                  Haonan Yuan and
                  Xingcheng Fu and
                  Min Zhou and
                  Yisen Gao and
                  Jianxin Li},
  title        = {GraphMoRE: Mitigating Topological Heterogeneity via Mixture of Riemannian
                  Experts},
  booktitle    = {AAAI},
  pages        = {11754--11762},
  year         = {2025}
}

@article{koltchinskii2010rademacher,
  author       = {Vladimir Koltchinskii},
  title        = {Rademacher Complexities and Bounding the Excess Risk in Active Learning},
  journal      = {J. Mach. Learn. Res.},
  pages        = {2457--2485},
  year         = {2010}
}

@article{hu2017hetero,
  author       = {Xiaojun Hu and
                  Loet Leydesdorff and
                  Ronald Rousseau},
  title        = {Heterogeneity in an undirected network: Definition and measurement},
  pages        = {669--682},
  year         = {2017}
}

@inproceedings{shazeer2017outrageously,
  author       = {Noam Shazeer and
                  Azalia Mirhoseini and
                  Krzysztof Maziarz and
                  Andy Davis and
                  Quoc V. Le and
                  Geoffrey E. Hinton and
                  Jeff Dean},
  title        = {Outrageously Large Neural Networks: The Sparsely-Gated Mixture-of-Experts
                  Layer},
  booktitle    = {{ICLR}},
  year         = {2017}
}

@inproceedings{liu2024ofa,
  author       = {Hao Liu and
                  Jiarui Feng and
                  Lecheng Kong and
                  Ningyue Liang and
                  Dacheng Tao and
                  Yixin Chen and
                  Muhan Zhang},
  title        = {One For All: Towards Training One Graph Model For All Classification
                  Tasks},
  booktitle    = {{ICLR}},
  year         = {2024}
}

@inproceedings{tang2024graphgpt,
  author       = {Jiabin Tang and
                  Yuhao Yang and
                  Wei Wei and
                  Lei Shi and
                  Lixin Su and
                  Suqi Cheng and
                  Dawei Yin and
                  Chao Huang},
  title        = {GraphGPT: Graph Instruction Tuning for Large Language Models},
  booktitle    = {{SIGIR}},
  pages        = {491--500},
  publisher    = {{ACM}},
  year         = {2024}
}

@article{wang2020ogb,
  author       = {Kuansan Wang and
                  Zhihong Shen and
                  Chiyuan Huang and
                  Chieh{-}Han Wu and
                  Yuxiao Dong and
                  Anshul Kanakia},
  title        = {Microsoft Academic Graph: When experts are not enough},
  journal      = {Quant. Sci. Stud.},
  pages        = {396--413},
  year         = {2020}
}

@article{he2023explanation,
  author       = {Xiaoxin He and
                  Xavier Bresson and
                  Thomas Laurent and
                  Bryan Hooi},
  title        = {Explanations as Features: LLM-Based Features for Text-Attributed Graphs},
  journal      = {CoRR},
  volume       = {abs/2305.19523},
  year         = {2023}
}

@inproceedings{huang2024reddit,
  author       = {Xuanwen Huang and
                  Kaiqiao Han and
                  Yang Yang and
                  Dezheng Bao and
                  Quanjin Tao and
                  Ziwei Chai and
                  Qi Zhu},
  title        = {Can {GNN} be Good Adapter for LLMs?},
  booktitle    = {{WWW}},
  pages        = {893--904},
  year         = {2024}
}

@inproceedings{yan2023ele,
  author       = {Hao Yan and
                  Chaozhuo Li and
                  Ruosong Long and
                  Chao Yan and
                  Jianan Zhao and
                  Wenwen Zhuang and
                  Jun Yin and
                  Peiyan Zhang and
                  Weihao Han and
                  Hao Sun and
                  Weiwei Deng and
                  Qi Zhang and
                  Lichao Sun and
                  Xing Xie and
                  Senzhang Wang},
  title        = {A Comprehensive Study on Text-attributed Graphs: Benchmarking and
                  Rethinking},
  booktitle    = {NeurIPS},
  year         = {2023}
}

@inproceedings{reimers-2020-multilingual-sentence-bert,
  author       = {Nils Reimers and
                  Iryna Gurevych},
  title        = {Making Monolingual Sentence Embeddings Multilingual using Knowledge
                  Distillation},
  booktitle    = {{EMNLP}},
  pages        = {4512--4525},
  year         = {2020}
}

@inproceedings{platonov2023critical,
  author       = {Oleg Platonov and
                  Denis Kuznedelev and
                  Michael Diskin and
                  Artem Babenko and
                  Liudmila Prokhorenkova},
  title        = {A critical look at the evaluation of GNNs under heterophily: Are we
                  really making progress?},
  booktitle    = {{ICLR}},
  year         = {2023}
}

@inproceedings{yu2025pronog,
  author       = {Xingtong Yu and
                  Jie Zhang and
                  Yuan Fang and
                  Renhe Jiang},
  title        = {Non-Homophilic Graph Pre-Training and Prompt Learning},
  booktitle    = {{KDD}},
  pages        = {1844--1854},
  year         = {2025}
}

@inproceedings{kong2025gofa,
  author       = {Lecheng Kong and
                  Jiarui Feng and
                  Hao Liu and
                  Chengsong Huang and
                  Jiaxin Huang and
                  Yixin Chen and
                  Muhan Zhang},
  title        = {{GOFA:} {A} Generative One-For-All Model for Joint Graph Language
                  Modeling},
  booktitle    = {{ICLR}},
  year         = {2025}
}

@article{huang2024hardneedmore,
  author       = {Quzhe Huang and
                  Zhenwei An and
                  Nan Zhuang and
                  Mingxu Tao and
                  Chen Zhang and
                  Yang Jin and
                  Kun Xu and
                  Liwei Chen and
                  Songfang Huang and
                  Yansong Feng},
  title        = {Harder Tasks Need More Experts: Dynamic Routing in MoE Models},
  journal      = {CoRR},
  year         = {2024}
}

@inproceedings{wang2025gpm,
  author       = {Zehong Wang and
                  Zheyuan Zhang and
                  Tianyi Ma and
                  Nitesh V. Chawla and
                  Chuxu Zhang and
                  Yanfang Ye},
  title        = {Beyond Message Passing: Neural Graph Pattern Machine},
  booktitle    = {{ICML}},
  series       = {Proceedings of Machine Learning Research},
  publisher    = {{PMLR} / OpenReview.net},
  year         = {2025}
}

@inproceedings{wang2025g2pm,
   title={Generative Graph Pattern Machine},
   author={Zehong Wang and Zheyuan Zhang and Tianyi Ma and Chuxu Zhang and Yanfang Ye},
   booktitle={NeurIPS},
   year={2025}
}

@inproceedings{sami2019mixhop,
  author       = {Sami Abu{-}El{-}Haija and
                  Amol Kapoor and
                  Nazanin Alipourfard and
                  Kristina Lerman and
                  Hrayr Harutyunyan and
                  Greg Ver Steeg and
                  Aram Galstyan},
  title        = {MixHop: Higher-Order Graph Convolutional Architectures via Sparsified
                  Neighborhood Mixing},
  booktitle    = {{ICML}},
  publisher    = {{PMLR}},
  year         = {2019}
}

@inproceedings{hu2020ogb,
  author       = {Weihua Hu and
                  Matthias Fey and
                  Marinka Zitnik and
                  Yuxiao Dong and
                  Hongyu Ren and
                  Bowen Liu and
                  Michele Catasta and
                  Jure Leskovec},
  title        = {Open Graph Benchmark: Datasets for Machine Learning on Graphs},
  booktitle    = {NeurIPS},
  year         = {2020}
}

@article{gaulton2012chembl,
  author       = {Anna Gaulton and
                  Louisa J. Bellis and
                  A. Patr{\'{\i}}cia Bento and
                  Jon Chambers and
                  Mark Davies and
                  Anne Hersey and
                  Yvonne Light and
                  Shaun McGlinchey and
                  David Michalovich and
                  Bissan Al{-}Lazikani and
                  John P. Overington},
  title        = {ChEMBL: a large-scale bioactivity database for drug discovery},
  year         = {2012}
}

@inproceedings{xu2024imbgnn,
  author       = {Wei Xu and
                  Pengkun Wang and
                  Zhe Zhao and
                  Binwu Wang and
                  Xu Wang and
                  Yang Wang},
  title        = {When Imbalance Meets Imbalance: Structure-driven Learning for Imbalanced
                  Graph Classification},
  booktitle    = {{WWW}},
  year         = {2024}
}

@inproceedings{muandet2013domain,
  author       = {Krikamol Muandet and
                  David Balduzzi and
                  Bernhard Sch{\"{o}}lkopf},
  title        = {Domain Generalization via Invariant Feature Representation},
  booktitle    = {{ICML}},
  pages        = {10--18},
  year         = {2013}
}

@inproceedings{zhang2019bridge,
  author       = {Yuchen Zhang and
                  Tianle Liu and
                  Mingsheng Long and
                  Michael I. Jordan},
  title        = {Bridging Theory and Algorithm for Domain Adaptation},
  booktitle    = {{ICML}},
  pages        = {7404--7413},
  year         = {2019}
}

@article{ganin2016domain,
  author       = {Yaroslav Ganin and
                  Evgeniya Ustinova and
                  Hana Ajakan and
                  Pascal Germain and
                  Hugo Larochelle and
                  Fran{\c{c}}ois Laviolette and
                  Mario Marchand and
                  Victor S. Lempitsky},
  title        = {Domain-Adversarial Training of Neural Networks},
  journal      = {JMLR},
  year         = {2016}
}

\bibliographystyle{icml2026}

\newpage
\appendix
\onecolumn
\section*{Appendix Overview}
In the Appendix, we provide additional details organized as follows:
\begin{enumerate}
    \item Appendix \ref{app:notaion}: Notations.
    \item Appendix \ref{app:complexity}: Time Complexity Analysis.
    \item Appendix \ref{app:dataset}: Datasets Details.
    \item Appendix \ref{app:additional}: Additional Experiments.
    \item Appendix \ref{app:experimental_details}: Experimental Details.
    \item Appendix \ref{app:more_related_work}: More Related Work Analysis.
    \item Appendix \ref{app:proof_noise_analysis}: Proof of Theorem~\ref{thm:theorem1} (Noise Analysis).
    \item Appendix \ref{app:proof_gog_edges_help_simple_1}: Proof of Theorem~\ref{thm:gog_edges_help_simple_1} (Effectiveness of GoG Edge Construction).
    \item Appendix \ref{app:proof_dynamic_expert}: Proof of Theorem~\ref{thm:dynamic_expert} (Excess-risk Upper Bound Analysis).
    \item Appendix \ref{app:tight}: Proof of Theorem~\ref{thm:main_tight} (Domain Generalization Error Bound Analysis).
\end{enumerate}

\section{Notations}
\label{app:notaion}
\renewcommand{\arraystretch}{1.08}

\begin{longtable}{>{\raggedright\arraybackslash}p{0.26\textwidth} | >{\raggedright\arraybackslash}p{0.5\textwidth}}
\caption{Notations.}\label{tab:notations}\\
\toprule
\textbf{Notation} & \textbf{Description} \\
\midrule
\endfirsthead

\toprule
\textbf{Notation} & \textbf{Description} \\
\midrule
\endhead

\midrule
\multicolumn{2}{r}{\small Continued on next page.}\\
\endfoot

\bottomrule
\endlastfoot

\(G=(V,E,X_V)\) & An attributed graph with node set \(V\), edge set \(E\), and node features \(X_V\). \\
\(\{G_i\}\) & A set/pool of graphs from different domains. \\
\(D_i\) & The domain associated with graph/dataset \(i\). \\
\(D=\{D_i\}\) & The set of all domains. \\
\(\mathrm{dist}(u,v)\) & Shortest-path distance between nodes \(u\) and \(v\) in \(G\). \\
\(N_k(u)\) & \(k\)-hop neighbors of node \(u\): \(N_k(u)=\{v\in V:\mathrm{dist}(u,v)\le k\}\). \\

\midrule

\(G_{\text{GoG}}=(V_{\text{GoG}},E_{\text{GoG}},X_G)\) & Graph-of-graphs (GoG) with GoG nodes \(V_{\text{GoG}}\), edges \(E_{\text{GoG}}\), and features \(X_G\). \\
\(v\) & A center (training) node in the original graph, used to construct a center-local GoG. \\
\(\{G^{(i)}_v\}_{i=1}^{K}\) & Adaptive-hop sampled subgraphs (GoG nodes) from hop \(1\) to hop \(K\) around center \(v\). \\
\(K\) & Hop number; also equals the number of GoG nodes for each center-local GoG. \\
\(B_{\mathrm{GPU}}\) & GPU memory budget used to determine the maximal feasible \(K\). \\
\(X_{\text{sub}}\in\mathbb{R}^{K\times d}\) & Subgraph embedding matrix (each row is a subgraph embedding). \\
\(S\in\mathbb{R}^{K\times K}\) & Subgraph similarity matrix (cosine; implemented as \(S=X_{\text{sub}}X_{\text{sub}}^{\top}\)). \\
\(\mathrm{Prob}(i,j)\) & Sampling probability of GoG edge \((i,j)\) after normalizing similarity scores. \\
\(B_{\mathrm{edge}}\) & Edge budget: number of GoG edges to sample (before symmetrization). \\
\(E_{\text{GoG}}\) & Sampled GoG edge set under budget \(B_{\mathrm{edge}}\). \\
\(e\) & Euler's number (used in the normalization for \(\mathrm{Prob}\)). \\

\midrule

\(x\) & (Random) subgraph embedding viewed as a noisy variable. \\
\(\mu\) & Noise-free subgraph embedding (mean of \(x\)). \\
\(\sigma\) & Embedding noise (std); smaller \(\sigma\) indicates higher quality. \\
\(\sigma_F\) & Noise when sampling from a fixed hop. \\
\(\sigma_V\) & Noise when sampling from various hops (adaptive-hop strategy). \\
\(\|\cdot\|_2\) & \(\ell_2\) norm. \\
\(e_{none}\) & Expected squared embedding error under a GoG without edges. \\
\(e_w\) & Expected squared embedding error with GoG edges. \\
\(e_{full}\) & Expected squared embedding error under a fully connected GoG. \\
\(\hat{\mu}_{none}\) & Subgraph-embedding estimators with constructed a GoG without edges. \\
\(\hat{\mu}_w\) & Subgraph-embedding estimators with constructed GoG edges. \\
\(\hat{\mu}_{full}\) & Subgraph-embedding estimators with fully connected GoG edges.\\

\midrule

\(\kappa\) & Constant curvature of a manifold (e.g., \(\kappa<0\) Hyperbolic, \(\kappa=0\) Euclidean, \(\kappa>0\) Hyperspherical). \\
\(\mathcal{K}\) & Candidate curvature set (selected based on structure statistics). \\
\(\mathcal{M}\) & Candidate Riemannian expert set induced by \(\mathcal{K}\). \\
\(\psi\) & Number of candidate experts (size of \(\mathcal{M}\)). \\
\(\deg(D_i)\) & Node-degree values of dataset/graph \(D_i\). \\
\(\mathrm{CV}(D_i)\) & Coefficient of variation of degree distribution: \(\mathrm{std}(\deg(D_i))/\mathrm{mean}(\deg(D_i))\). \\

\midrule

\(c\) & Center index (each center induces a GoG). \\
\(k\in\{1,\dots,K\}\) & GoG node index (corresponding to hop \(k\)). \\
\(j\in\{1,\dots,\psi\}\) & Expert index. \\
\(z^{(j)}_{c,k}\) & Router logits for assigning GoG node \((c,k)\) to expert \(j\). \\
\(\alpha^{(j)}_{c,k}\) & Routing weights (softmax over experts) for GoG node \((c,k)\). \\
\(\tau\) & Temperature in the softmax routing distribution. \\
\(\mathrm{conf}_{c,k}\) & Confidence score, defined as \(\max_j \alpha^{(j)}_{c,k}\). \\
\(m\) & Number of activated experts in Top-\(m\) routing (dynamic). \\
\(\odot\) & Element-wise multiplication (masking non-Top-\(m\) experts). \\
\(\mathbf{1}[\cdot]\) & Indicator function (e.g., \(\mathbf{1}[\mathrm{rank}(\alpha)\le m]\)). \\
\(\tilde{\alpha}^{(j)}_{c,k}\) & Masked-and-renormalized routing weights after Top-\(m\). \\
\(u_j\) & Normalized usage / selection frequency of expert \(j\) (for balancing). \\
\(\mathcal{L}_{\text{lb}}\) & Load-balancing loss to mitigate expert collapse. \\

\end{longtable}
\clearpage
\section{Time Complexity Analysis}
\label{app:complexity}
\begin{algorithm}
\caption{Dynamical Candidate Expert set Determination}
\label{alg:candidate_set}
\begin{algorithmic}[t]
\STATE \textbf{Input:} dataset pool $\{\mathcal{D}_1,\ldots,\mathcal{D}_i\}$ with full graphs $\{G_{\mathcal{D}_j}\}_{j=1}^i$ (self-loops added), scale $\zeta$, predefined ordered curvature list $\mathcal{K}_0$
\STATE \textbf{Output:} candidate curvature set $\mathcal{K}_{\mathcal{D}_i}$, candidate expert set $\mathcal{M}_{\mathcal{D}_i}$, expert count $M_{\mathcal{D}_i}$
\FOR{$j=1$ \textbf{to} $i$}
    \STATE compute $CV(\mathcal{D}_j)=\frac{\mathrm{std}(\deg(\mathcal{D}_j))}{\mathrm{mean}(\deg(\mathcal{D}_j))}$ \hfill (Eq.~(5))
\ENDFOR
\STATE $\mu_i \leftarrow \mathrm{mean}\big(\{CV(\mathcal{D}_1),\ldots,CV(\mathcal{D}_i)\}\big)$,\ \ 
$\sigma_i \leftarrow \mathrm{std}\big(\{CV(\mathcal{D}_1),\ldots,CV(\mathcal{D}_i)\}\big)$
\STATE $S_i \leftarrow \mathrm{normalize}(\mu_i+\sigma_i)$ \hfill (Eq.~(6))
\STATE $M_{\mathcal{D}_i} \leftarrow \min\big(|\mathcal{K}_0|,\ \lceil S_i\cdot\zeta\rceil\big)$
\STATE $\mathcal{K}_{\mathcal{D}_i} \leftarrow \mathcal{K}_0[1:M_{\mathcal{D}_i}]$ \hfill (prefix of ordered list)
\STATE $\mathcal{M}_{\mathcal{D}_i} \leftarrow \{\mathrm{GNN}_{\kappa}\mid \kappa\in\mathcal{K}_{\mathcal{D}_i}\}$
\STATE \textbf{return} $\mathcal{K}_{\mathcal{D}_i},\mathcal{M}_{\mathcal{D}_i},M_{\mathcal{D}_i}$
\end{algorithmic}
\end{algorithm}

\begin{algorithm}[H]
\caption{Subgraph Similarity-based GoG Edge Construction}
\label{alg:build_gog}
\begin{algorithmic}[t]
\STATE \textbf{Input:} hop-wise embeddings $\{\mathbf{x}_{c,k}\}_{k=1}^{K}$, edge budget $\mathcal{B}_{edge}$
\STATE \textbf{Output:} center-local GoG $\mathcal{G}_c=(\mathcal{V}_c,\mathcal{E}_c)$ and node features $\mathbf{X}_c$
\STATE $\mathcal{V}_c \leftarrow \{1,\ldots,K\}$
\STATE $\mathbf{X}_c \leftarrow [\mathbf{x}_{c,1};\ldots;\mathbf{x}_{c,K}] \in \mathbb{R}^{K\times d}$
\STATE row-normalize $\mathbf{X}_c$ to obtain $\bar{\mathbf{X}}_c$
\STATE $\mathbf{S}_c \leftarrow \bar{\mathbf{X}}_c(\bar{\mathbf{X}}_c)^\top \in \mathbb{R}^{K\times K}$ \hfill (Eq.~(1))
\STATE $\mathcal{U}_c \leftarrow \{(i,j)\mid 1\le i<j\le K\}$
\STATE $q_c(i,j) \leftarrow \frac{{\mathrm{e}}^{\mathbf{S}[i,j]}}{\sum_u\sum_ve^{\mathbf{S}[u,v]}}$, \ \ $\forall (i,j)\in\mathcal{U}_c$ \hfill (Eq.~(2))
\STATE sample $\tilde{\mathcal{E}}_c \leftarrow \mathrm{SampleWOR}(q_c;\ B^{\text{gog}}_{e})$ \hfill (Eq.~(3))
\STATE $\mathcal{E}_c \leftarrow \{(i,j),(j,i)\mid (i,j)\in \tilde{\mathcal{E}}_c\}$ \ \
\STATE \textbf{return} $\mathcal{G}_c=(\mathcal{V}_c,\mathcal{E}_c)$ and $\mathbf{X}_c$
\end{algorithmic}
\end{algorithm}

\begin{algorithm}
\caption{Confidence-aware Dynamic Top-$m$ Routing}
\label{alg:dynamic_routing}
\begin{algorithmic}[t]
\STATE \textbf{Input:} GoG $\mathcal{G}_c=(\mathcal{V}_c,\mathcal{E}_c)$, initial hop features $\mathbf{X}_c=\{\mathbf{x}_{c,k}\}_{k=1}^{K}$, candidate experts $\mathcal{M}_{\mathcal{D}}=\{\mathrm{GNN}_{\kappa}\mid \kappa\in\mathcal{K}_{\mathcal{D}}\}$ with $M_{\mathcal{D}}=|\mathcal{K}_{\mathcal{D}}|$, shortlist $m_{\text{start}}$, active count $m$, minimum $m_{\min}$, temperature $\tau$
\STATE \textbf{Output:} refined hop embeddings $\{\mathbf{h}_{c,k}\}_{k=1}^{K}$, load-balancing loss $\mathcal{L}_{\text{lb}}$
\FOR{$k=1$ \textbf{to} $K$}
    \STATE $\mathbf{z}_{c,k} \leftarrow g(\mathcal{G}_c,\mathbf{x}_{c,k}) \in \mathbb{R}^{M_{\mathcal{D}}}$ \hfill (router logits)
    \STATE $\boldsymbol{\alpha}_{c,k}\leftarrow \mathrm{softmax}(\mathbf{z}_{c,k}/\tau) \in \mathbb{R}^{M_{\mathcal{D}}}$
\ENDFOR
\STATE $\bar{\boldsymbol{\alpha}}_{c}\leftarrow \frac{1}{K}\sum_{k=1}^{K}\boldsymbol{\alpha}_{c,k}$
\STATE $\mathcal{I}^{\text{top}}_{c}\leftarrow \mathrm{TopK}(\bar{\boldsymbol{\alpha}}_c,\ m_{\text{start}})$ \hfill (indices in $[1,M_{\mathcal{D}}]$)
\STATE $\tilde{\boldsymbol{\alpha}}_{c}\leftarrow \mathrm{Normalize}\!\Big(\mathrm{softmax}(\bar{\boldsymbol{\alpha}}_c|_{\mathcal{I}^{\text{top}}_c})\odot \mathbf{1}[\mathrm{rank}\le m]\Big)$
\STATE \textbf{compute only shortlisted experts:} $\mathbf{H}_c^{(j)} \leftarrow \mathrm{GNN}_{\kappa_j}(\mathcal{G}_c,\mathbf{X}_c)$ for $j\in \mathcal{I}^{\text{top}}_c$
\FOR{$k=1$ \textbf{to} $K$}
    \STATE $\mathbf{h}_{c,k}\leftarrow \sum_{j\in\mathcal{I}^{\text{top}}_c}\tilde{\alpha}_{c}^{(j)}\,\mathbf{H}_c^{(j)}[k]$
\ENDFOR
\STATE $u_j\leftarrow \mathbb{E}_{c,k}[\alpha_{c,k}^{(j)}]$, \ \ $\mathcal{L}_{\text{lb}}\leftarrow \|\mathbf{u}-\tfrac{1}{M_{\mathcal{D}}}\mathbf{1}\|_2^2$
\STATE update $m\leftarrow \max(m_{\min},\, m-\Delta(\mathrm{conf}))$ \hfill
\STATE \textbf{return}  $\{\mathbf{h}_{c,k}\}_{k=1}^{K},\mathcal{L}_{\text{lb}}$
\end{algorithmic}
\end{algorithm}

\paragraph{Dynamical Candidate Expert set Determination (Algorithm~\ref{alg:candidate_set}).}
Computing node degrees on the full graph $G_{\mathcal{D}}=(V_{\mathcal{D}},E_{\mathcal{D}})$ takes $\mathcal{O}(|E_{\mathcal{D}}|)$, and computing the coefficient of variation (CV) and score statistics over degrees takes $\mathcal{O}(|V_{\mathcal{D}}|)$.
Mapping the score to $M_{\mathcal{D}}$ and prefix truncation are $\mathcal{O}(M_{\mathcal{D}})$.
Overall, the candidate-set determination cost is
\begin{equation}
\mathcal{O}\big(|E_{\mathcal{D}}|+|V_{\mathcal{D}}|\big),
\end{equation}
and is incurred once per dataset (amortized across all training steps).

\paragraph{Subgraph Similarity-based GoG Edge Construction (Algorithm.~\ref{alg:build_gog}).}
Given hop vectors, row-normalization costs $\mathcal{O}(Kd)$.
The similarity matrix $\mathbf{S}={\mathbf{X}}(\bar{\mathbf{X}})^\top$ costs
\begin{equation}
\mathcal{O}(K^2 d).
\end{equation}
Forming the pair set $\mathcal{U}_c$ and computing the sampling distribution $q_c(i,j)$ over $\Theta(K^2)$ unordered pairs costs $\mathcal{O}(K^2)$, and sampling $\mathcal{B}_{edge}$ edges without replacement can be implemented in $\tilde{\mathcal{O}}(K^2)$ time.
After symmetrization, $E_c = 2\mathcal{B}_{edge}$.

\paragraph{Confidence-aware Dynamic Top-$m$ Routing (Algorithm~\ref{alg:dynamic_routing}).}
The router is a 2-layer GCN on $\mathcal{G}_c$.
Using the standard decomposition of (i) sparse propagation and (ii) linear transforms, each layer costs
$\mathcal{O}(E_c d_r + K d_r^2)$, hence the router costs
\begin{equation}
\mathcal{O}\big(E_c d_r + K d_r^2\big).
\end{equation}
Computing the shortlist via $\mathrm{TopK}(\cdot,m_{\text{start}})$ over $M_{\mathcal{D}}$ experts costs
$\mathcal{O}(M_{\mathcal{D}}\log m_{\text{start}})$ per center (or $\mathcal{O}(M_{\mathcal{D}})$ with selection algorithms).
The remaining normalization and masking are $\mathcal{O}(M_{\mathcal{D}})$.

\paragraph{Riemannian expert computation and fusion.}
Assume each activated expert performs $L$ layers of curvature-aware message passing on the same GoG topology.
Per expert, the cost is $\mathcal{O}\big(L(E_c d' + K d'^2)\big)$, hence activating only $m$ experts yields
\begin{equation}
\mathcal{O}\big(m\,L(E_c d' + K d'^2)\big)\ \le\ \mathcal{O}\big(m_{\text{start}}\,L(E_c d' + K d'^2)\big),
\label{eq:moe_complexity_app}
\end{equation}
instead of $\mathcal{O}\big(M_{\mathcal{D}}\,L(E_c d' + K d'^2)\big)$ under dense routing.
The final fusion $\sum_{m}\tilde{\alpha}_c^{(m)}\mathbf{H}_c^{(m)}[k]$ costs $\mathcal{O}(mKd')$.

\paragraph{End-to-end per-center complexity.}
Combining the above, the dominant per-center cost is
\begin{equation}
\mathcal{O}\big(K^2 d\big)\ +\ \mathcal{O}\big(E_c d_r + K d_r^2\big)\ +\ \mathcal{O}\big(m\,L(E_c d' + K d'^2)\big),
\end{equation}
where $K$ is explicitly controlled by the GPU memory budget and $E_c=2\mathcal{B}_{edge}$ by the GoG edge budget.

\section{Dataset Detail}
\label{app:dataset}
\subsection{Main setting: 10 benchmark graphs}
\label{app:dataset:main}

\textbf{Benchmarks.}
Our main setting uses 10 benchmark graphs covering four domains: citations (Cora, CiteSeer, PubMed), webpage hyperlinks (Cornell, Texas, Wisconsin), Wikipedia page networks (Chameleon, Squirrel), and e-commerce co-purchase graphs (Amazon-Computers, Amazon-Photo)~\cite{sen2008collective,namata2012query,pei2020geom,mcauley2015image,shchur2018pitfalls}.
We follow a leave-one-dataset-out protocol on this benchmark suite: for each target dataset, we pretrain on the remaining datasets and then evaluate transfer on the held-out target. Table~\ref{table.datasets} summarizes dataset statistics for the main setting.

\textbf{Why exclude Chameleon and Squirrel as targets.}
Recent work reports severe issues in Chameleon and Squirrel (e.g., duplicate nodes that may induce train-test leakage), which can make cross-method comparisons unreliable~\cite{platonov2023critical}.
To ensure fair comparisons while preserving domain coverage, we still include Chameleon and Squirrel as pretraining sources in this main setting, but we exclude them from downstream evaluation when reporting the main transfer results (i.e., we report results on the remaining 8 target datasets).
In later experiments, we additionally use Chameleon as a targeted downstream testbed to verify effectiveness on the Wikipedia-network domain.

\begin{table}[t]
    \centering
    \small
    \caption{Benchmark datasets used in standard evaluation.}
    \label{table.datasets}

    \setlength{\tabcolsep}{6pt} 
    \renewcommand{\arraystretch}{1.12}

    \resizebox{0.7\linewidth}{!}{
    \begin{tabular}{@{}l | r r r c l@{}}
        \toprule
        Dataset & \#Nodes & \#Edges & Class & Domain \\
        \midrule

        Cora           & 2,708  & 5,278   & 7  & Academic \\
        CiteSeer       & 3,327  & 4,552   & 6  & Academic \\
        PubMed         & 19,717 & 44,324  & 3  & Academic \\
        \midrule

        Cornell        & 183    & 277     & 5  & Webpage \\
        Texas          & 183    & 279     & 5  & Webpage \\
        Wisconsin      & 251    & 450     & 5  & Webpage \\
        \midrule

        Chameleon      & 2,277  & 31,371  & 5  & Wikipedia \\
        Squirrel       & 5,201  & 198,353 & 5  & Wikipedia \\
        \midrule

        Amazon-Computers & 13,752 & 245,861 & 10 & E-commerce \\
        Amazon-Photo   & 7,650  & 119,081 & 8  & E-commerce \\
        \bottomrule
    \end{tabular}}
\end{table}

\subsection{Large-scale setting: unified semantic node features and large graphs}
\label{app:dataset:lmtext}

\textbf{Motivation.}
To validate \methodname{} under (i) a unified semantic feature space across heterogeneous graphs and (ii) larger-scale real-world graphs, we consider an LM-text scaling setting.

\textbf{Text-based node features.}
For datasets with raw node texts, we encode each node's text using a frozen language model (Sentence-BERT all-MiniLM-L6-v2~\cite{reimers-2020-multilingual-sentence-bert}) and use the resulting embeddings as node features.
This provides semantically rich and fixed-dimensional representations shared across datasets.

\textbf{Pretraining and evaluation.}
Under this setting, we pretrain on ArXiv\_2023~\cite{he2023explanation}, ogbn-Arxiv~\cite{wang2020ogb}, Reddit~\cite{huang2024reddit}, and PubMed~\cite{namata2012query}, and evaluate transfer on Cora~\cite{sen2008collective}, Ele-Computers and Books-History~\cite{yan2023ele}, and Instagram~\cite{huang2024reddit}. Table~\ref{tab:dataset_stats} reports detailed statistics for the LM-text scaling setting.

\begin{table}[t]
    \centering
    \scriptsize
    \setlength{\tabcolsep}{4pt}
    \renewcommand{\arraystretch}{1.05}
    \caption{LM-text scaling setting: dataset statistics for source and target graphs.}
    \label{tab:dataset_stats}
    \resizebox{0.7\columnwidth}{!}{
    \begin{tabular}{@{}l|rrccc@{}}
        \toprule
        & \#Nodes & \#Edges & Domain & Class & Usage \\
        \midrule
        ogbn-ArXiv      & 169,343   & 2,315,598  & Academic   & 40 & $G_{\text{source}}$ \\
        ArXiv\_2023     & 33,868    & 611,344     & Academic   & 40 & $G_{\text{source}}$ \\
        PubMed           & 19,717    & 44,324     & Academic   & 3  & $G_{\text{source}}$ \\
        Reddit          & 33,434    & 269,442    & Social     & 2  & $G_{\text{source}}$ \\
        \midrule
        Cora            & 2,708     & 5,278      & Academic   & 7  & $G_{\text{target}}$ \\
        Ele-Computers    & 87,229    & 1,256,548    & E-commerce & 10 & $ G_{\text{target}}$ \\
        Books-History    & 41,551    & 503,180    & E-commerce & 12 & $G_{\text{target}}$ \\
        Instagram        & 11,339    & 144,010    & Social     & 2  & $G_{\text{target}}$ \\
        \bottomrule
    \end{tabular}}
\end{table}

\subsection{Datasets description}
\noindent\textbf{Citation networks.}
We use several citation graphs including Cora and CiteSeer~\cite{sen2008collective}, PubMed~\cite{namata2012query}, ogbn-ArXiv~\cite{wang2020ogb}, and ArXiv\_2023~\cite{he2023explanation}.
In these datasets, nodes represent papers and edges indicate citation links.
Each node is associated with textual attributes of the paper (e.g., title and abstract), and a bag-of-words feature vector depending on the benchmark construction.
The node label corresponds to the topic or category of the paper.
Compared with Cora and CiteSeer that cover diverse computer science topics, PubMed consists of biomedical publications focused on diabetes, where labels indicate the type of diabetes studied~\cite{namata2012query}.
For ogbn-ArXiv~\cite{wang2020ogb} and ArXiv\_2023~\cite{he2023explanation}, labels follow arXiv subject areas with 40 classes.

\medskip
\noindent\textbf{Webpage hyperlink networks (WebKB).}
We evaluate on the WebKB benchmarks Cornell, Texas, and Wisconsin~\cite{pei2020geom}, which are subsets derived from the WebKB corpus.
Each node corresponds to a web page, and edges denote hyperlinks between pages.
Node features are bag-of-words representations extracted from page content.
Nodes are classified into five categories: student, project, course, staff, and faculty~\cite{pei2020geom}.

\medskip
\noindent\textbf{Wikipedia page networks.}
We use Chameleon and Squirrel~\cite{pei2020geom}, two page-to-page graphs built from Wikipedia.
Nodes represent web pages and edges represent page links.
Node attributes are constructed from the page content (e.g., collections of nouns/words), and node labels correspond to discretized levels of average monthly traffic of the pages~\cite{pei2020geom}.

\medskip
\noindent\textbf{E-commerce graphs.}
We consider multiple product and item graphs, including Amazon-Computers and Amazon-Photo~\cite{mcauley2015image,shchur2018pitfalls}, as well as Ele-Computers and Books-History~\cite{yan2023ele}.
In these datasets, nodes denote products/items and edges connect related items defined by the benchmark construction (e.g., co-purchase/co-interaction relations).
Each node is associated with an attribute vector and/or text fields such as product title/description.
Node labels indicate item categories, with 10 classes for Amazon-Computers and Ele-Computers, 8 classes for Amazon-Photo, and 12 classes for Books-History~\cite{mcauley2015image,shchur2018pitfalls,yan2023ele}.

\medskip
\noindent\textbf{Online social networks.}
We use Reddit~\cite{huang2024reddit} and Instagram~\cite{huang2024reddit}.
These datasets represent large-scale online platforms, where nodes correspond to entities defined by the benchmark (e.g., users or content units) and edges capture their relations/interactions.
Nodes are associated with raw text attributes when available, and the tasks are binary node classification with 2 classes~\cite{huang2024reddit}.

\section{Additional Experiments}
\label{app:additional}
\subsection{Additional link prediction results}
\label{app:linkpred_full}
\begin{table}[t]
    \centering
    \scriptsize
    \setlength{\tabcolsep}{3pt}
    \renewcommand{\arraystretch}{1.05}
    \caption{Link prediction results (\textbf{AUC-ROC}(\%)). Results are reported in percent. The best method is bolded and the runner-up is underlined.}
    \label{tab:linkpred_full}
    \resizebox{\linewidth}{!}{
    \begin{tabular}{@{}l|ccccccc@{}}
        \toprule
        Methods & Wisconsin & Cornell & Citeseer & Pubmed & Cora & Photos & Texas \\
        \midrule
        \method{GCOPE}      
            & \underline{83.12 $\pm$ 1.51}  
            & 64.72 $\pm$ \phantom{0}4.99 
            & 85.09 $\pm$ 0.34 
            & 84.99 $\pm$ 0.06
            & \underline{88.71$\pm$0.57}
            & 78.45$\pm$5.51
            & \underline{80.51$\pm$1.75} \\
        \method{SAMGPT}     
            & 65.32 $\pm$ 7.27 
            & 53.53 $\pm$ \phantom{0}4.57 
            & 89.91 $\pm$ 2.70 
            & 84.92 $\pm$ 0.40
            & 88.62$\pm$3.51
            & 80.59$\pm$0.33
            & 51.20$\pm$4.42 \\
        \method{MDGFM}      
            & 67.46 $\pm$ 6.82 
            & 54.92 $\pm$ 21.10 
            & 89.12 $\pm$ 4.68 
            & 84.78 $\pm$ 2.09
            & 88.41$\pm$2.12
            & 81.24$\pm$5.80
            & 65.22$\pm$4.05 \\
        \method{RiemannGFM} 
            & 78.19 $\pm$ 2.08 
            & \underline{79.06 $\pm$ \phantom{0}2.96} 
            & \underline{90.63 $\pm$ 0.58} 
            & \underline{86.37 $\pm$ 0.33}
            & 88.54$\pm$0.11
            & \underline{81.44$\pm$0.95}
            & 72.41$\pm$5.96 \\
        \midrule
        \method{\methodname{}} 
            & \textbf{84.15 $\pm$ 0.71} 
            & \textbf{85.90 $\pm$ \phantom{0}0.69} 
            & \textbf{90.88 $\pm$ 0.70} 
            & \textbf{88.62 $\pm$ 0.41}
            & \textbf{89.27$\pm$0.64}
            & \textbf{81.53$\pm$0.83}
            & \textbf{87.94$\pm$0.96} \\
        \bottomrule
    \end{tabular}}
\end{table}

Table~\ref{tab:linkpred_full} reports the link prediction performance in terms of AUC-ROC.
Overall, \methodname{} consistently achieves the best results across all evaluated datasets, outperforming prior GFMs by clear margins.
Compared with the strongest baseline in each case (typically RiemannGFM), \methodname{} brings absolute improvements of +1.03 on Wisconsin, +6.84 on Cornell, +0.25 on Citeseer, +2.25 on Pubmed, +0.56 on Cora, +0.09 on Photos, and +7.43 on Texas AUC-ROC.
Notably, the gains are substantially larger on heterophilic WebKB graphs (Wisconsin/Cornell/Texas), indicating that \methodname{} is more robust to structural heterogeneity, while still remaining competitive on citation and co-purchase networks (Citeseer/Pubmed/Cora/Photos).

\subsection{{Additional scalability results}}
\label{app:scalability}

\begin{figure}[t]
    \centering
    \includegraphics[width=0.5\linewidth]{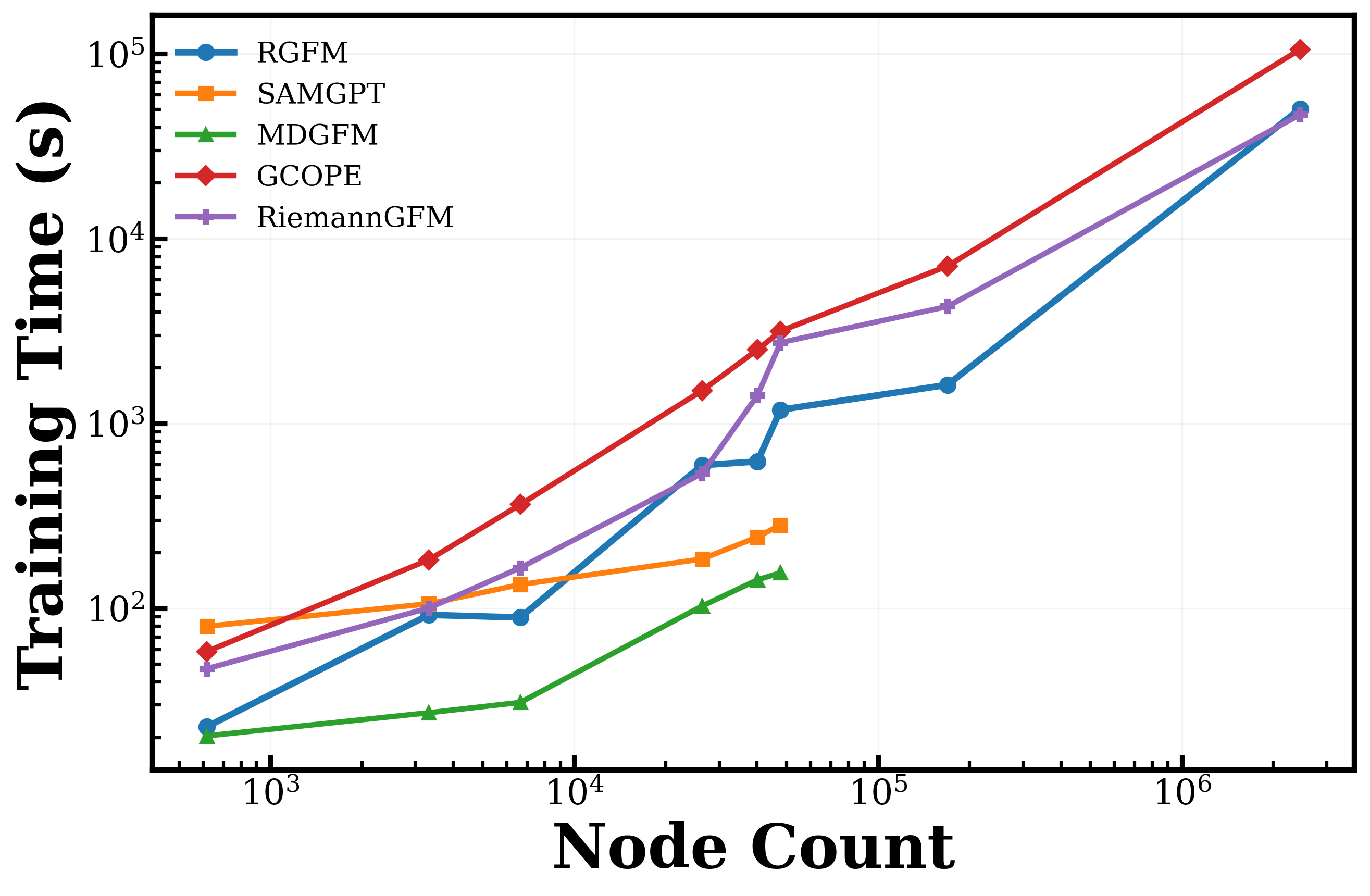}
    \caption{{Scalability comparison in terms of training time with respect to the number of nodes.}}
    \label{fig:scalability}
\end{figure}

{Figure~\ref{fig:scalability} reports the scalability comparison of different graph foundation models in terms of training time. Overall, \methodname{} demonstrates favorable scalability as the graph size increases, maintaining competitive training efficiency compared with existing GFMs. In particular, \methodname{} scales more efficiently than GCOPE and RiemannGFM on large-scale graphs, while achieving substantially better predictive performance as shown in the main experiments. These results indicate that \methodname{} not only improves effectiveness but also remains practical for large-scale graph applications.}

\subsection{{Graph Classification}}
\label{app:graph_classification}

{To further evaluate the generalizability of \methodname{} beyond node-level tasks, we conduct additional graph classification experiments on HIV~\citep{wu2018moleculenet,hu2020ogb}. Specifically, models are pretrained on the large-scale molecular graph dataset ChEMBL~\citep{gaulton2012chembl} and then fine-tuned on the downstream HIV graph classification benchmark.}

\begin{table}[t]
    \centering
    \setlength{\tabcolsep}{6pt}
    \renewcommand{\arraystretch}{1.05}
    \caption{{Transfer pretraining results on graph classification. Models are pretrained on ChEMBL and fine-tuned on HIV. The best result is bolded and the runner-up is underlined.}}
    \label{tab:graph_classification_transfer}
    \begin{tabular}{@{}lc@{}}
        \toprule
        {Model} & {Test AUC} \\
        \midrule
        {\method{MDGFM}}      & {0.7563} \\
        {\method{SAMGPT}}     & {0.7664} \\
        {\method{RiemannGFM}} & {\underline{0.7804}} \\
        {\method{GCOPE}}      & {0.7230} \\
        \midrule
        {\method{\methodname{}}} & {\textbf{0.7849}} \\
        \bottomrule
    \end{tabular}
\end{table}

\begin{table}[t]
    \centering
    \setlength{\tabcolsep}{5pt}
    \renewcommand{\arraystretch}{1.05}
    \caption{{Few-shot graph classification results on HIV.}}
    \label{tab:graph_classification_fewshot}
    \begin{tabular}{@{}lccc@{}}
        \toprule
        {Method} & {1-shot} & {3-shot} & {5-shot} \\
        \midrule
        {Multi-RF Fusion} & {0.5498} & {0.5764} & {0.5913} \\
        {DeepAUC}         & {0.5172} & {0.5910} & {0.6068} \\
        {HIG}             & {0.5433} & {0.5913} & {0.5973} \\
        {PAS-OGB}         & {0.5003} & {0.5948} & {0.6082} \\
        {\method{MDGFM}}      & {0.6101} & {0.6169} & {0.6365} \\
        {\method{SAMGPT}}     & {0.6170} & {0.6248} & {0.6424} \\
        {\method{RiemannGFM}} & {\underline{0.6175}} & {0.6273} & {0.6368} \\
        {\method{GCOPE}}      & {0.5825} & {\underline{0.6354}} & {\underline{0.6536}} \\
        \midrule
        {\method{\methodname{}}} & {\textbf{0.6473}} & {\textbf{0.6557}} & {\textbf{0.6648}} \\
        \bottomrule
    \end{tabular}
\end{table}

{As shown in Tables~\ref{tab:graph_classification_transfer} and~\ref{tab:graph_classification_fewshot}, \methodname{} achieves the best Test AUC in both transfer and few-shot graph classification settings. These results indicate that \methodname{} can generalize from node-level tasks to graph-level prediction while maintaining consistent advantages over existing GFM baselines.}

\section{Experimental Details}
\label{app:experimental_details}
\subsection{Baseline Selection}
\label{app:baselines}
\paragraph{Why these baselines.}
We select four categories of baselines to cover the most common training and adaptation paradigms for graph transfer.
(1) \textbf{Task-supervised GNNs} (GCN~\cite{kipf2017gcn}, GAT~\cite{velickovic2017graph}, H2GCN~\cite{zhu2020beyond}, FAGCN~\cite{bo2021beyond}) are included as widely-used supervised architectures spanning homophily-oriented and heterophily-aware designs, which is important given the diverse structural properties across our benchmarks.
(2) \textbf{Self-supervised pretraining with fine-tuning} baselines (DGI~\cite{petar2019deep}, GraphCL~\cite{you2020graph}, GraphACL~\cite{xiao2023simple}, {MixHop~\cite{sami2019mixhop}}) represent representative self-supervised objectives for learning transferable node/graph representations before downstream adaptation.
(3) \textbf{Prompt-based adaptation} baselines (GPPT~\cite{sun2022gppt}, GPF~\cite{fang2023gpf}, GraphPrompt~\cite{liu2023graphprompt}, GraphPrompt+~\cite{yu2024generalized}) are included as parameter-efficient alternatives to full fine-tuning, which is a common setting in recent graph transfer literature.
(4) \textbf{Existing GFMs} (GCOPE~\cite{zhao2024all}, SAMGPT~\cite{yu2025samgpt}, MDGFM~\cite{wang2025multi}, RiemannGFM~\cite{sun2025riemanngfm}, {GPM~\cite{wang2025gpm}, G2PM~\cite{wang2025g2pm}}) are included as recent general-purpose graph foundation models that aim to scale pretraining across diverse graphs and transfer via fine-tuning or lightweight adaptation.
Following common practice on graphs with varying homophily/heterophily, we adopt FAGCN as the Stage-1 encoder in \methodname{} for its robustness across both regimes.

\subsection{Expertiment Settings}
\label{sec:exp_settings}
In the standard benchmark setting, we follow a leave-one-dataset-out protocol: for each target dataset, we pretrain on all remaining datasets and evaluate transfer on the held-out target (Table~\ref{tab:transfer_settings_main}).  
In the large-scale setting, we pretrain on the designated large graphs and evaluate transfer on the designated targets as summarized in Table~\ref{tab:transfer_settings_lmtext}.
\begin{table}[t]
    \centering
    \small
    \caption{Settings of transfer evaluation on the 10-benchmark suite.}
    \label{tab:transfer_settings_main}

    \setlength{\tabcolsep}{3.8pt}
    \renewcommand{\arraystretch}{1.12}

    \resizebox{\linewidth}{!}{
    \begin{tabular}{l|ccc|ccc|cc|cc}
        \toprule
        \multirow{2}{*}{Target} &
        \multicolumn{3}{c|}{Citations} &
        \multicolumn{3}{c|}{Webpage} &
        \multicolumn{2}{c|}{Wikipedia} &
        \multicolumn{2}{c}{E-commerce} \\
        \cmidrule(lr){2-4}\cmidrule(lr){5-7}\cmidrule(lr){8-9}\cmidrule(lr){10-11}
        & Cora & CiteSeer & PubMed
        & Cornell & Texas & Wisconsin
        & Chameleon & Squirrel
        & Amazon-Computers & Amazon-Photo \\
        \midrule

        Cora            &       & \cmark & \cmark & \cmark & \cmark & \cmark & \cmark & \cmark & \cmark & \cmark \\
        CiteSeer        & \cmark &        & \cmark & \cmark & \cmark & \cmark & \cmark & \cmark & \cmark & \cmark \\
        PubMed          & \cmark & \cmark &        & \cmark & \cmark & \cmark & \cmark & \cmark & \cmark & \cmark \\
        \midrule
        Cornell         & \cmark & \cmark & \cmark &        & \cmark & \cmark & \cmark & \cmark & \cmark & \cmark \\
        Texas           & \cmark & \cmark & \cmark & \cmark &        & \cmark & \cmark & \cmark & \cmark & \cmark \\
        Wisconsin       & \cmark & \cmark & \cmark & \cmark & \cmark &        & \cmark & \cmark & \cmark & \cmark \\
        \midrule
        Amazon-Computers& \cmark & \cmark & \cmark & \cmark & \cmark & \cmark & \cmark & \cmark &        & \cmark \\
        Amazon-Photo    & \cmark & \cmark & \cmark & \cmark & \cmark & \cmark & \cmark & \cmark & \cmark &        \\
        \bottomrule
    \end{tabular}}
\end{table}

\begin{table}[t]
    \centering
    \small
    \caption{Large scale setting: transfer protocol.}
    \label{tab:transfer_settings_lmtext}

    \setlength{\tabcolsep}{6pt}
    \renewcommand{\arraystretch}{1.10}

    \resizebox{0.85\linewidth}{!}{
    \begin{tabular}{l|ccc|c}
        \toprule
        \multirow{2}{*}{Target} &
        \multicolumn{3}{c|}{Academic sources} &
        \multicolumn{1}{c}{Social source} \\
        \cmidrule(lr){2-4}\cmidrule(lr){5-5}
        & ogbn-ArXiv & ArXiv\_2023 & PubMed & Reddit \\
        \midrule
        Cora           & \cmark & \cmark & \cmark & \cmark \\
        Ele-Computers  & \cmark & \cmark & \cmark & \cmark \\
        Books-History  & \cmark & \cmark & \cmark & \cmark \\
        Instagram      & \cmark & \cmark & \cmark & \cmark \\
        \bottomrule
    \end{tabular}}
\end{table}

\paragraph{Hardware and software.}
All experiments are conducted on a server with $2\times$ Intel(R) Xeon(R) Platinum 8358 CPU @ 2.60GHz, 377\, GiB RAM, and NVIDIA A100 80GB PCIe GPUs (driver 550.54.15); unless otherwise stated, experiments use a single A100 80GB GPU.
The software stack (conda environment \texttt{prog}) includes Python 3.10.13, PyTorch 2.1.0+cu118, DGL 2.4.0+cu118, PyG 2.6.1, OGB 1.3.6, and Transformers 4.36.2.
Unless otherwise stated, each result is averaged over five runs with different random seeds.

\paragraph{Implementation details.}
Both tasks use a per-GPU batch size of 128 and FAGCN (dropout 0.2), and we use temperature $\tau=0.2$ for training.

For node tasks, training follows two stages: Stage~1 trains the encoder for 100 epochs with Adam (lr $5\times 10^{-3}$, wd $2\times 10^{-6}$), and Stage~2 optimizes the MoE for 50 epochs with Adam (lr $10^{-2}$, wd $2\times 10^{-6}$) together with RiemannianAdam (lr $10^{-3}$, wd $5\times 10^{-4}$, stabilize=100).

For link tasks, \texttt{edge\_task.sh} uses two stages with 200 epochs each (total 400 epochs): Stage~1 trains the encoder for 200 epochs and Stage~2 optimizes the MoE for 200 epochs, with $\texttt{encoder\_lr}=\texttt{moe\_lr}=\texttt{classifier\_lr}=3\times10^{-3}$, $\texttt{riemannian\_lr}=10^{-3}$, wd $10^{-5}$, and load-balance weight 0.01.

\subsection{Code-level Training and Evaluation Pipeline}
\label{app:pipeline}
Our released implementation instantiates the four-stage workflow in Fig.~2 (Stage A--D) as an end-to-end, two-stage routine:
(i) Stage-1 subgraph encoder training and (ii) Stage-2 GoG-MoE training, followed by downstream evaluation.

\paragraph{Entry points.}
\texttt{main.py} runs node-level tasks and \texttt{main-link.py} runs link-level tasks. Each entry script performs data preparation,
(sub)graph construction and caching, Stage-1 training, Stage-2 training, evaluation, and checkpointing in one run.

\paragraph{Step 1: Data acquisition and preprocessing.}
For node classification, the loader downloads datasets (PyG/OGB) on demand, standardizes node features by clipping or padding
to a fixed dimensionality, and stores \texttt{x/y/edge\_index} tensors. For link prediction, the loader applies a randomized split
(e.g., \texttt{RandomLinkSplit} with held-out validation/test edges and negative sampling).

\paragraph{Step 2: k-hop subgraph construction and caching.}
For each training center (node or edge), the implementation constructs 1\ldots$K$ hop ego-subgraphs and caches them to disk.
This implements the adaptive-hop idea in the paper: as $K$ increases, the model captures larger structural scales but incurs higher
GPU memory cost; therefore, $K$ is treated as a budget-controlled parameter and the implementation includes OOM-safe fallback.

\paragraph{Step 3: Graph augmentations and GraphCL views (node tasks).}
For node tasks, two augmented views are sampled per subgraph to form a GraphCL-style contrastive pair (e.g., node dropping,
edge perturbation, and attribute masking with a sampled corruption ratio). Augmented views can be precomputed and cached to
reduce repeated preprocessing overhead.

\paragraph{Step 4: Stage-1 training.}
\textbf{Node tasks:} train a GNN subgraph encoder with a contrastive objective (GraphCL-style) on sampled subgraph pairs.
\textbf{Link tasks:} train the encoder (and a lightweight classifier) with supervised edge labels constructed from positive/negative
pairs, using k-hop features as input.

\paragraph{Step 5: Stage-2 GoG-MoE training.}
Subgraph embeddings from multiple hops are treated as nodes in a higher-order Graph-of-Graphs (GoG). The GoG is constructed
based on similarity among subgraph embeddings (Stage C in Fig.~2), then encoded by a Mixture-of-Experts (MoE) router over
Riemannian experts, and fused into a single representation for downstream prediction (Stage D in Fig.~2). The router uses a
sparse top-$m$ expert shortlist; $m$ starts from a larger value and is decreased stepwise as router confidence increases.

\paragraph{Step 6: Evaluation and reporting.}
\textbf{Node tasks:} run multiple k-shot splits on the target dataset for a fixed number of epochs, select the best
checkpoint on validation accuracy, then report test accuracy. Results are averaged over multiple random seeds.
\textbf{Link tasks:} evaluate with AUC-ROC (and optionally Acc/Hits@K) on validation/test splits; save the best checkpoint and
report final test metrics.

\paragraph{Step 7: Outputs, logs, and checkpoints.}
Training prints per-epoch loss/metrics and the evolution of \textit{top-$m$}. Checkpoints are saved with dataset-specific names
(e.g., \texttt{checkpoints/node\_...pt} and \texttt{checkpoints/edge\_...pt}); users are encouraged to include seed in the
prefix to avoid overwriting when running multiple seeds.

The full source code is available in the Supplementary Materials.

\section{More Related Work Analysis}
\label{app:more_related_work}
A Euclidean manifold, which is widely used in existing GFMs~\cite{hou2022graphmae,zhao2024all}, lacks the capacity to capture the varying geometric patterns. 
Specifically, tree-like structures align naturally with a a Hyperbolic manifold, whereas cycle-like structures are better represented in a hyperspherical manifold.
However, the recent Riemannian manifold-based training mechanism, RiemannGFM~\cite{sun2025riemanngfm}, cannot be directly transferred to GoG-based GFMs because it relies on extracting tree-like and cycle-like structures and encoding them via fixed corresponding manifolds that are incompatible with the dynamically constructed GoGs.

\section{Proof of Theorem~\ref{thm:theorem1} (Noise Analysis)}
\label{app:proof_noise_analysis}

In this section, we provide a detailed proof of Theorem~\ref{thm:theorem1}
under an adaptive-hop fusion view.

\paragraph{Setup and assumptions.}
Fix a training node (center) $v$ and consider the sampled $k$-hop subgraph embedding
$\mathbf{x}_{v}^{(k)} \in \mathbb{R}^{d}$ for each $k\in\{1,\ldots,K\}$.

\begin{assumption}[Second-order hop-wise noise model (allowing correlations)]
\label{assump:second_order}
For each hop $k$, the sampled embedding admits the decomposition
\begin{equation}
\mathbf{x}_{v}^{(k)}=\boldsymbol{\mu}_{v}^{(k)}+\boldsymbol{\epsilon}_{v}^{(k)},
\qquad
\mathbb{E}\big[\boldsymbol{\epsilon}_{v}^{(k)}\big]=\mathbf{0},
\end{equation}
where $\boldsymbol{\mu}_{v}^{(k)}$ is the (possibly hop-dependent) noise-free embedding and
$\boldsymbol{\epsilon}_{v}^{(k)}$ is a zero-mean noise vector with finite second moments.

Moreover, the (possibly correlated) second-order structure across hops is characterized by
\begin{equation}
\boldsymbol{\Sigma}_{v,k\ell}
:=
\mathrm{Cov}\!\big(\boldsymbol{\epsilon}_{v}^{(k)},\boldsymbol{\epsilon}_{v}^{(\ell)}\big)
\in\mathbb{R}^{d\times d},
\qquad 1\le k,\ell\le K,
\end{equation}
and the block covariance matrix
$\boldsymbol{\Sigma}_{v}:=
\big[\boldsymbol{\Sigma}_{v,k\ell}\big]_{k,\ell=1}^{K}\in\mathbb{R}^{Kd\times Kd}$
is positive semidefinite.
\end{assumption}

\begin{assumption}[Various-hop sampling via hop fusion]
\label{assump:fusion_general}
The various-hop strategy forms a fused embedding by convexly combining multiple hop embeddings:
\begin{equation}
\tilde{\mathbf{x}}_{v}(\mathbf{w})
:=
\sum_{k=1}^{K} w_k\,\mathbf{x}_{v}^{(k)},
\qquad
\mathbf{w}\in\Delta_{K}:=\Big\{\mathbf{w}\in\mathbb{R}^{K}_{+}:\sum_{k=1}^{K} w_k=1\Big\}.
\end{equation}
Accordingly, the fused sampling noise is the random vector
\begin{equation}
\tilde{\boldsymbol{\epsilon}}_{v}(\mathbf{w})
:=
\sum_{k=1}^{K} w_k\,\boldsymbol{\epsilon}_{v}^{(k)}.
\end{equation}
Let $\boldsymbol{\sigma}_{v}(\mathbf{w})$ denote the coordinate-wise standard-deviation vector of
the fused sampling noise:
\begin{equation}
\boldsymbol{\sigma}_{v}(\mathbf{w})
:= \sqrt{\mathrm{diag}\!\big(\mathrm{Cov}(\tilde{\boldsymbol{\epsilon}}_{v}(\mathbf{w}))\big)}.
\end{equation}
The strategy chooses the fusion weights by minimizing the overall noise magnitude measured by
the $\ell_2$-norm:
\begin{equation}
\mathbf{w}^{\star}\in\arg\min_{\mathbf{w}\in\Delta_K}\ \big\|\boldsymbol{\sigma}_{v}(\mathbf{w})\big\|_{2}.
\end{equation}
\end{assumption}

\paragraph{Definitions of $\boldsymbol{\sigma}_{\mathrm{F}}$ and $\boldsymbol{\sigma}_{\mathrm{V}}$.}
A fixed-hop strategy selects a single hop $k_{\mathrm{fixed}}$ and always uses $\mathbf{x}_{v}^{(k_{\mathrm{fixed}})}$.
Under Assumption~\ref{assump:second_order}, its sampling noise is $\boldsymbol{\epsilon}_{v}^{(k_{\mathrm{fixed}})}$.
Let
\begin{equation}
\boldsymbol{\sigma}_{v,k}
:=\sqrt{\mathrm{diag}\!\big(\mathrm{Cov}(\boldsymbol{\epsilon}_{v}^{(k)})\big)}
=\sqrt{\mathrm{diag}\!\big(\boldsymbol{\Sigma}_{v,kk}\big)}
\in\mathbb{R}^d_{+}.
\end{equation}
We define
\begin{equation}
\boldsymbol{\sigma}_{\mathrm{F}}:=\boldsymbol{\sigma}_{v,k_{\mathrm{fixed}}},
\qquad
\boldsymbol{\sigma}_{\mathrm{V}}:=\boldsymbol{\sigma}_{v}(\mathbf{w}^{\star}).
\end{equation}

\begin{proof}[Proof of Theorem~\ref{thm:theorem1}]
By Assumption~\ref{assump:second_order}, for each hop $k$ the randomness of the sampled embedding
comes from $\boldsymbol{\epsilon}_{v}^{(k)}$ and satisfies $\mathbb{E}[\boldsymbol{\epsilon}_{v}^{(k)}]=\mathbf{0}$.

Under the various-hop fusion strategy (Assumption~\ref{assump:fusion_general}), the fused noise is
\begin{equation}
\tilde{\boldsymbol{\epsilon}}_{v}(\mathbf{w})=\sum_{k=1}^{K} w_k\,\boldsymbol{\epsilon}_{v}^{(k)}.
\end{equation}
Using bilinearity of covariance (and finite second moments), we obtain
\begin{align}
\mathrm{Cov}\!\big(\tilde{\boldsymbol{\epsilon}}_{v}(\mathbf{w})\big)
&=
\mathrm{Cov}\!\Big(\sum_{k=1}^{K} w_k\,\boldsymbol{\epsilon}_{v}^{(k)}\Big)
=
\sum_{k=1}^{K}\sum_{\ell=1}^{K} w_k w_\ell\,
\mathrm{Cov}\!\big(\boldsymbol{\epsilon}_{v}^{(k)},\boldsymbol{\epsilon}_{v}^{(\ell)}\big) \nonumber\\
&=
\sum_{k=1}^{K}\sum_{\ell=1}^{K} w_k w_\ell\,\boldsymbol{\Sigma}_{v,k\ell}.
\label{eq:cov_fusion}
\end{align}
Compared to hop-selection (sampling one hop), the cross-hop covariances $\boldsymbol{\Sigma}_{v,k\ell}$ now
enter explicitly through \eqref{eq:cov_fusion}.

By definition of the standard deviation vector,
\begin{equation}
\big\|\boldsymbol{\sigma}_{v}(\mathbf{w})\big\|_2^2
=
\sum_{r=1}^{d}\mathrm{Var}\!\Big(\big[\tilde{\boldsymbol{\epsilon}}_{v}(\mathbf{w})\big]_r\Big)
=
\mathrm{tr}\!\Big(\mathrm{Cov}\!\big(\tilde{\boldsymbol{\epsilon}}_{v}(\mathbf{w})\big)\Big).
\label{eq:norm_trace}
\end{equation}
Combining \eqref{eq:cov_fusion} and \eqref{eq:norm_trace} gives
\begin{equation}
\big\|\boldsymbol{\sigma}_{v}(\mathbf{w})\big\|_2^2
=
\sum_{k=1}^{K}\sum_{\ell=1}^{K} w_k w_\ell\, \mathrm{tr}\!\big(\boldsymbol{\Sigma}_{v,k\ell}\big).
\label{eq:quadratic_fusion}
\end{equation}

Now, by Assumption~\ref{assump:fusion_general}, $\mathbf{w}^\star$ minimizes
$\|\boldsymbol{\sigma}_{v}(\mathbf{w})\|_2$ over the simplex $\Delta_K$.
For any fixed hop $k_{\mathrm{fixed}}$, the one-hot vector
$\mathbf{e}_{k_{\mathrm{fixed}}}\in\Delta_K$ (with $(\mathbf{e}_{k_{\mathrm{fixed}}})_{k_{\mathrm{fixed}}}=1$)
is feasible and corresponds exactly to the fixed-hop sampling noise, because
\begin{equation}
\tilde{\mathbf{x}}_{v}(\mathbf{e}_{k_{\mathrm{fixed}}})=\mathbf{x}_{v}^{(k_{\mathrm{fixed}})},
\qquad
\tilde{\boldsymbol{\epsilon}}_{v}(\mathbf{e}_{k_{\mathrm{fixed}}})
=
\boldsymbol{\epsilon}_{v}^{(k_{\mathrm{fixed}})},
\qquad
\boldsymbol{\sigma}_{v}(\mathbf{e}_{k_{\mathrm{fixed}}})
=
\boldsymbol{\sigma}_{v,k_{\mathrm{fixed}}}.
\end{equation}
Therefore, by optimality of $\mathbf{w}^\star$,
\begin{equation}
\big\|\boldsymbol{\sigma}_{\mathrm{V}}\big\|_2
=
\big\|\boldsymbol{\sigma}_{v}(\mathbf{w}^{\star})\big\|_2
\le
\big\|\boldsymbol{\sigma}_{v}(\mathbf{e}_{k_{\mathrm{fixed}}})\big\|_2
=
\big\|\boldsymbol{\sigma}_{v,k_{\mathrm{fixed}}}\big\|_2
=
\big\|\boldsymbol{\sigma}_{\mathrm{F}}\big\|_2.
\end{equation}
This proves $\|\boldsymbol{\sigma}_{\mathrm{V}}\|_2 \le \|\boldsymbol{\sigma}_{\mathrm{F}}\|_2$,
which is exactly Theorem~\ref{thm:theorem1}.
\end{proof}

\paragraph{Remark (when fusion can be strictly better).}
Equation~\eqref{eq:quadratic_fusion} shows that the fused noise depends on both within-hop covariances
$\boldsymbol{\Sigma}_{v,kk}$ and cross-hop covariances $\boldsymbol{\Sigma}_{v,k\ell}$.
If the hop-wise noises are not perfectly correlated across hops (e.g., some $\mathrm{tr}(\boldsymbol{\Sigma}_{v,k\ell})$
are sufficiently smaller than $\sqrt{\mathrm{tr}(\boldsymbol{\Sigma}_{v,kk})\,\mathrm{tr}(\boldsymbol{\Sigma}_{v,\ell\ell})}$),
then an interior $\mathbf{w}$ may yield $\|\boldsymbol{\sigma}_{v}(\mathbf{w})\|_2$ strictly smaller than every one-hot choice,
leading to a strict improvement over any fixed hop under the same metric.

\section{Proof of Theorem~\ref{thm:gog_edges_help_simple_1} (Effectiveness of GoG Edge Construction)}
\label{app:proof_gog_edges_help_simple_1}

We prove Theorem~\ref{thm:gog_edges_help_simple_1} via an exact bias--variance decomposition on GoG hop nodes.
Cross-hop mixing reduces the estimation variance by convex smoothing, but it may introduce a leakage (bias) term
when hop-specific targets differ. Similarity-biased mixing reduces the leakage compared to fully-connected uniform mixing,
while still achieving non-trivial variance reduction compared to removing all cross-hop edges.

\paragraph{Setup.}
Fix a training center node $c$ and its $K$ hop-induced GoG nodes indexed by $k=1,\ldots,K$ (assume $K\ge 2$).
Allow hop-specific noise-free targets $\boldsymbol{\mu}_{c,k}\in\mathbb{R}^d$.
The hop-wise estimator before the GoG message passing is
\begin{equation}
\hat{\boldsymbol{\mu}}_{c,k}=\boldsymbol{\mu}_{c,k}+\mathbf e_{c,k},
\qquad k=1,\ldots,K,
\label{eq:hop_model_v2}
\end{equation}
where $\mathbf e_{c,k}$ is a stochastic estimation noise.

\paragraph{Three pipelines (same encoder, only GoG edges differ).}
All three use the same residual propagation form on hop nodes with $\alpha\in[0,1]$:
\begin{equation}
\hat{\boldsymbol{\mu}}^{(m)}_{c,k}
=
\alpha \hat{\boldsymbol{\mu}}_{c,k} + (1-\alpha)\sum_{t=1}^{K}\widehat{\mathbf{P}}^{(m)}_{c}[k,t]\hat{\boldsymbol{\mu}}_{c,t},
\qquad k=1,\ldots,K,
\label{eq:mp_generic_v2}
\end{equation}
where $m\in\{\mathrm{ours},\mathrm{full},\mathrm{none}\}$ and each $\widehat{\mathbf{P}}^{(m)}_{c}$ is row-stochastic.
For $m\in\{\mathrm{ours},\mathrm{full}\}$ we use no self-mixing:
$\widehat{\mathbf{P}}^{(m)}_{c}[k,k]=0$.

(none cross-hop edges): take $\widehat{\mathbf{P}}^{(\mathrm{none})}_{c}=\mathbf I$, hence
$\hat{\boldsymbol{\mu}}^{(\mathrm{none})}_{c,k}=\hat{\boldsymbol{\mu}}_{c,k}$.

(fully-connected): uniform mixing over all $t\neq k$:
$\widehat{\mathbf{P}}^{(\mathrm{full})}_{c}[k,t]=\frac{1}{K-1}\mathbf{1}\{t\neq k\}$.

(ours: similarity-biased):
Let $\mathbf z_{c,k}$ be the hop embedding used for edge construction and define cosine scores
$s_{c}(k,t)=\cos(\mathbf z_{c,k},\mathbf z_{c,t})$.
We use the softmax row distribution
\begin{equation}
\widehat{\mathbf{P}}^{(\mathrm{ours})}_{c}[k,t]
=
\frac{\exp(\beta\, s_c(k,t))}{\sum_{u\neq k}\exp(\beta\, s_c(k,u))}\cdot \mathbf 1\{t\neq k\},
\qquad \beta>0.
\label{eq:ours_softmax_P_v2}
\end{equation}

\paragraph{Convex weights.}
Define
\begin{equation}
w^{(m)}_{c,k}(t):=\alpha\mathbf{1}\{t=k\}+(1-\alpha)\widehat{\mathbf{P}}^{(m)}_{c}[k,t],
\qquad t=1,\ldots,K,
\label{eq:def_weights_v2}
\end{equation}
so that $w^{(m)}_{c,k}(t)\ge 0$ and $\sum_{t=1}^K w^{(m)}_{c,k}(t)=1$.

\paragraph{Error metric.}
We measure hop-wise MSE to the hop-specific target:
\begin{equation}
e_{m}
:=\mathbb{E}_{c,k}\big\|\hat{\boldsymbol{\mu}}^{(m)}_{c,k}-\boldsymbol{\mu}_{c,k}\big\|_2^2,
\qquad m\in\{\mathrm{ours},\mathrm{full},\mathrm{none}\}.
\label{eq:mse_def_v2}
\end{equation}

\paragraph{Assumptions.}
\begin{assumption}[Noise model compatible with weight construction]
\label{assump:noise_v2}
Conditioned on $c$ and on the realized propagation weights $\{w^{(m)}_{c,k}(t)\}_{t=1}^K$,
the noises satisfy:
(i) $\mathbb E[\mathbf e_{c,t}\mid c,\{w\}]=\mathbf 0$ for all $t$;
(ii) $\mathbb E[\langle \mathbf e_{c,t},\mathbf e_{c,t'}\rangle\mid c,\{w\}]=0$ for $t\neq t'$;
(iii) $\mathbb E[\|\mathbf e_{c,t}\|_2^2\mid c,\{w\}]=B_c^2$ for all $t$.
\end{assumption}

\begin{assumption}[One-factor hop mismatch and similarity alignment]
\label{assump:mismatch_align_v2}
For each $(c,k)$ there exists a unit vector $\mathbf u_{c,k}\in\mathbb{R}^d$ and scalars
$\{\delta_{c,k}(t)\}_{t\neq k}\subseteq[0,\infty)$ such that
\[
\boldsymbol{\mu}_{c,t}-\boldsymbol{\mu}_{c,k}=\delta_{c,k}(t)\,\mathbf u_{c,k},
\qquad \forall\, t\neq k.
\]
Moreover, for each fixed $(c,k)$, if $s_c(k,t_1)\ge s_c(k,t_2)$ then $\delta_{c,k}(t_1)\le \delta_{c,k}(t_2)$.
\end{assumption}

\begin{assumption}[Separation between variance reduction and leakage gain]
\label{assump:tradeoff_window_v3}
Assume $K\ge 3$ and $\alpha\in[0,1)$.
Recall
\[
S_{c,k}:=\sum_{t\neq k}\Big(\widehat{\mathbf P}^{(\mathrm{ours})}_c[k,t]\Big)^2,
\quad
\bar\delta^{\mathrm{ours}}_{c,k}:=\sum_{t\neq k}\widehat{\mathbf P}^{(\mathrm{ours})}_c[k,t]\delta_{c,k}(t),
\quad
\bar\delta^{\mathrm{full}}_{c,k}:=\frac{1}{K-1}\sum_{t\neq k}\delta_{c,k}(t).
\]
Fix $(c,k)$ and sort $t\neq k$ so that $s_c(k,t_1)\ge\cdots\ge s_c(k,t_{K-1})$.
Let $p_i:=\widehat{\mathbf P}^{(\mathrm{ours})}_c[k,t_i]$ and $\delta_i:=\delta_{c,k}(t_i)$, and define the
prefix-mass advantage
\begin{equation}
A_{c,k}
:=\sum_{j=1}^{K-2}(\delta_{j+1}-\delta_j)\left(\sum_{i=1}^j p_i-\frac{j}{K-1}\right).
\label{eq:def_A_ck_v3}
\end{equation}

Assume there exist constants $\varepsilon_{\mathrm{wo}}>0$ and $\varepsilon_{\mathrm{full}}>0$ such that
\begin{align}
\mathbb E_{c,k}\!\left[B_c^2\Big(1-\alpha^2-(1-\alpha)^2 S_{c,k}\Big)\right]
&\ge
(1-\alpha)^2\,\mathbb E_{c,k}\!\left[(\bar\delta^{\mathrm{ours}}_{c,k})^2\right]
+\varepsilon_{\mathrm{wo}},
\label{eq:tradeoff_wo_v3}
\\
\mathbb E_{c,k}\!\left[A_{c,k}^2\right]
&\ge
\mathbb E_{c,k}\!\left[B_c^2\Big(S_{c,k}-\frac{1}{K-1}\Big)\right]
+\varepsilon_{\mathrm{full}}.
\label{eq:tradeoff_full_v3}
\end{align}
\end{assumption}

\begin{proof}[Proof of Theorem~\ref{thm:gog_edges_help_simple_1}]
We show $e_{\mathrm{ours}}<e_{\mathrm{none}}$ and $e_{\mathrm{ours}}<e_{\mathrm{full}}$.

\paragraph{Step 1: Exact bias--variance decomposition.}
From \eqref{eq:mp_generic_v2} and \eqref{eq:def_weights_v2},
\[
\hat{\boldsymbol{\mu}}^{(m)}_{c,k}
=
\sum_{t=1}^K w^{(m)}_{c,k}(t)\hat{\boldsymbol{\mu}}_{c,t}.
\]
Subtract $\boldsymbol{\mu}_{c,k}$ and use \eqref{eq:hop_model_v2}:
\begin{align}
\hat{\boldsymbol{\mu}}^{(m)}_{c,k}-\boldsymbol{\mu}_{c,k}
&=
\sum_{t=1}^K w^{(m)}_{c,k}(t)\big(\boldsymbol{\mu}_{c,t}-\boldsymbol{\mu}_{c,k}\big)
+
\sum_{t=1}^K w^{(m)}_{c,k}(t)\mathbf e_{c,t}
\nonumber\\
&=: \mathbf b^{(m)}_{c,k}+\mathbf n^{(m)}_{c,k}.
\label{eq:bias_var_v2}
\end{align}
Conditioned on $c$ and realized weights, Assumption~\ref{assump:noise_v2}(i) gives
$\mathbb E[\mathbf n^{(m)}_{c,k}\mid c,\{w\}]=0$, hence the cross term vanishes:
\begin{equation}
\mathbb{E}\big\|\mathbf b^{(m)}_{c,k}+\mathbf n^{(m)}_{c,k}\big\|_2^2
=
\mathbb{E}\big\|\mathbf b^{(m)}_{c,k}\big\|_2^2
+
\mathbb{E}\big\|\mathbf n^{(m)}_{c,k}\big\|_2^2.
\label{eq:cross_zero_v2}
\end{equation}

\paragraph{Step 2: Noise term (variance after convex smoothing).}
Conditioned on $c$ and weights, by Assumption~\ref{assump:noise_v2}(ii)--(iii),
\begin{align}
\mathbb{E}\big[\|\mathbf n^{(m)}_{c,k}\|_2^2\mid c,\{w\}\big]
&=
\mathbb{E}\Big[\Big\|\sum_t w^{(m)}_{c,k}(t)\mathbf e_{c,t}\Big\|_2^2\mid c,\{w\}\Big]
=
B_c^2\sum_{t=1}^K (w^{(m)}_{c,k}(t))^2.
\label{eq:noise_v2}
\end{align}
Averaging over $(c,k)$ yields the exact variance contribution.

For none cross-hop edges, $w^{(\mathrm{none})}_{c,k}(k)=1$ and others $0$, thus
\begin{equation}
e_{\mathrm{none}}=\mathbb{E}_c B_c^2.
\label{eq:ewo_v2}
\end{equation}

For $m\in\{\mathrm{ours},\mathrm{full}\}$ (no self-mixing), $w^{(m)}_{c,k}(k)=\alpha$ and
\[
\sum_{t=1}^K (w^{(m)}_{c,k}(t))^2=\alpha^2+(1-\alpha)^2\sum_{t\neq k}\big(\widehat{\mathbf P}^{(m)}_c[k,t]\big)^2.
\]
In particular,
\begin{equation}
\sum_{t\neq k}\big(\widehat{\mathbf P}^{(\mathrm{full})}_c[k,t]\big)^2=\frac{1}{K-1},
\qquad
\sum_{t\neq k}\big(\widehat{\mathbf P}^{(\mathrm{ours})}_c[k,t]\big)^2=S_{c,k}.
\label{eq:s2_v2}
\end{equation}

\paragraph{Step 3: Leakage term (exact under one-factor mismatch).}
Since $\boldsymbol{\mu}_{c,k}-\boldsymbol{\mu}_{c,k}=\mathbf 0$, and for $m\in\{\mathrm{ours},\mathrm{full}\}$
we have $\widehat{\mathbf P}^{(m)}_c[k,k]=0$, \eqref{eq:bias_var_v2} gives
\[
\mathbf b^{(m)}_{c,k}
=
(1-\alpha)\sum_{t\neq k}\widehat{\mathbf P}^{(m)}_c[k,t]\;(\boldsymbol{\mu}_{c,t}-\boldsymbol{\mu}_{c,k}).
\]
By Assumption~\ref{assump:mismatch_align_v2},
\[
\mathbf b^{(m)}_{c,k}
=
(1-\alpha)\Big(\sum_{t\neq k}\widehat{\mathbf P}^{(m)}_c[k,t]\delta_{c,k}(t)\Big)\mathbf u_{c,k},
\]
hence
\begin{equation}
\big\|\mathbf b^{(m)}_{c,k}\big\|_2^2
=
(1-\alpha)^2\Big(\sum_{t\neq k}\widehat{\mathbf P}^{(m)}_c[k,t]\delta_{c,k}(t)\Big)^2.
\label{eq:bias_exact_v2}
\end{equation}
For none, $\mathbf b^{(\mathrm{none})}_{c,k}\equiv 0$.

\paragraph{Step 4: Exact MSE formula (all three methods).}
Combining \eqref{eq:cross_zero_v2}, \eqref{eq:noise_v2} and \eqref{eq:bias_exact_v2}, we obtain:
\begin{align}
e_{\mathrm{ours}}
&=
\mathbb E_{c,k}\!\left[B_c^2\Big(\alpha^2+(1-\alpha)^2 S_{c,k}\Big)\right]
+
(1-\alpha)^2\,\mathbb E_{c,k}\!\left[(\bar\delta^{\mathrm{ours}}_{c,k})^2\right],
\label{eq:eours_v2}
\\
e_{\mathrm{full}}
&=
\mathbb E_{c,k}\!\left[B_c^2\Big(\alpha^2+\frac{(1-\alpha)^2}{K-1}\Big)\right]
+
(1-\alpha)^2\,\mathbb E_{c,k}\!\left[(\bar\delta^{\mathrm{full}}_{c,k})^2\right],
\label{eq:efull_v2}
\\
e_{\mathrm{none}}
&=
\mathbb E_c B_c^2.
\label{eq:ewo2_v2}
\end{align}

\paragraph{Step 5: Similarity-biased mixing reduces leakage vs uniform (pointwise).}
Fix $(c,k)$ and sort indices $t\neq k$ so that
$s_c(k,t_1)\ge \cdots \ge s_c(k,t_{K-1})$.
Let $p_i:=\widehat{\mathbf P}^{(\mathrm{ours})}_c[k,t_i]$ and $\delta_i:=\delta_{c,k}(t_i)$, and
$q_i:=\frac{1}{K-1}$.
Using the discrete Abel transform,
\[
\sum_{i=1}^{K-1} p_i\,\delta_i
=
\delta_{K-1}
-
\sum_{j=1}^{K-2}(\delta_{j+1}-\delta_j)\Big(\sum_{i=1}^j p_i\Big),
\qquad
\sum_{i=1}^{K-1} q_i\,\delta_i
=
\delta_{K-1}
-
\sum_{j=1}^{K-2}(\delta_{j+1}-\delta_j)\Big(\sum_{i=1}^j q_i\Big).
\]
Therefore, the gap can be written exactly as
\begin{equation}
\bar\delta^{\mathrm{full}}_{c,k}-\bar\delta^{\mathrm{ours}}_{c,k}
=
\sum_{j=1}^{K-2}(\delta_{j+1}-\delta_j)
\left(\sum_{i=1}^j p_i-\frac{j}{K-1}\right)
=:A_{c,k}.
\label{eq:delta_gap_equals_A_v2}
\end{equation}
Under Assumption~\ref{assump:mismatch_align_v2} and $\beta>0$, we have $\delta_{j+1}-\delta_j\ge 0$,
and $p_1\ge\cdots\ge p_{K-1}$ implies $\sum_{i=1}^j p_i\ge \frac{j}{K-1}$ for all $j$,
hence
\[
A_{c,k}\ge 0,
\qquad\text{and thus}\qquad
\bar\delta^{\mathrm{ours}}_{c,k}\le \bar\delta^{\mathrm{full}}_{c,k}.
\]

\paragraph{Step 6: Compare total MSEs.}

(i) Ours beats no-edge.
Subtract \eqref{eq:eours_v2} from \eqref{eq:ewo2_v2}:
\[
e_{\mathrm{none}}-e_{\mathrm{ours}}
=
\mathbb E_{c,k}\!\left[B_c^2\Big(1-\alpha^2-(1-\alpha)^2 S_{c,k}\Big)\right]
-
(1-\alpha)^2\,\mathbb E_{c,k}\!\left[(\bar\delta^{\mathrm{ours}}_{c,k})^2\right].
\]
By Assumption~\ref{assump:tradeoff_window_v3}\eqref{eq:tradeoff_wo_v3},
$e_{\mathrm{none}}-e_{\mathrm{ours}}\ge \varepsilon_{\mathrm{wo}}>0$, hence $e_{\mathrm{ours}}<e_{\mathrm{none}}$.

(ii) Ours beats fully-connected.
Subtract \eqref{eq:eours_v2} from \eqref{eq:efull_v2}:
\[
e_{\mathrm{full}}-e_{\mathrm{ours}}
=
(1-\alpha)^2\Bigg(
\mathbb E_{c,k}\!\left[(\bar\delta^{\mathrm{full}}_{c,k})^2-(\bar\delta^{\mathrm{ours}}_{c,k})^2\right]
-
\mathbb E_{c,k}\!\left[B_c^2\Big(S_{c,k}-\frac{1}{K-1}\Big)\right]
\Bigg).
\]
Using \eqref{eq:delta_gap_equals_A_v2}, write
$\bar\delta^{\mathrm{full}}_{c,k}=\bar\delta^{\mathrm{ours}}_{c,k}+A_{c,k}$ with $A_{c,k}\ge 0$.
Then
\[
(\bar\delta^{\mathrm{full}}_{c,k})^2-(\bar\delta^{\mathrm{ours}}_{c,k})^2
=
A_{c,k}\big(2\bar\delta^{\mathrm{ours}}_{c,k}+A_{c,k}\big)
\;\ge\;
A_{c,k}^2,
\]
since $\bar\delta^{\mathrm{ours}}_{c,k}\ge 0$. Therefore,
\[
e_{\mathrm{full}}-e_{\mathrm{ours}}
\ge
(1-\alpha)^2\Bigg(
\mathbb E_{c,k}\!\left[A_{c,k}^2\right]
-
\mathbb E_{c,k}\!\left[B_c^2\Big(S_{c,k}-\frac{1}{K-1}\Big)\right]
\Bigg).
\]
By Assumption~\ref{assump:tradeoff_window_v3}\eqref{eq:tradeoff_full_v3},
the bracket is at least $\varepsilon_{\mathrm{full}}>0$, hence $e_{\mathrm{full}}-e_{\mathrm{ours}}
\ge (1-\alpha)^2\varepsilon_{\mathrm{full}}>0$ and thus $e_{\mathrm{ours}}<e_{\mathrm{full}}$.

Combining (i) and (ii) completes the proof.
\end{proof}

\section{Proof of Theorem~\ref{thm:dynamic_expert} (Excess-risk Upper Bound Analysis)}  
\label{app:proof_dynamic_expert}

\noindent\textbf{Setup.}\;
Let $\mathcal{G}_{1:N}$ be the training pool consisting of all constructed GoG subgraphs from
$\{\mathcal{D}_1,\ldots,\mathcal{D}_N\}$ and let $n_N:=|\mathcal{G}_{1:N}|$.

For each candidate expert number $j\in\{1,\ldots,\psi_{\max}\}$, let $\mathcal{F}_j$ denote the hypothesis class
induced by using $j$ curvature experts (including router and experts), and define the empirical risk
$\widehat{R}(f):=\frac{1}{n_N}\sum_{g\in\mathcal{G}_{1:N}}\mathcal{L}(f;g)$ and population risk
$R(f):=\mathbb{E}[\mathcal{L}(f;G)]$ where $G$ follows the (population) distribution of GoG subgraphs.

\medskip
\noindent\textbf{Compressed assumptions (appendix-only).}

\begin{assumption}[Standard uniform generalization bound for $\mathcal{F}_j$]
\label{assump:std_gen}
Fix $\delta\in(0,1)$. With probability at least $1-\delta$ over $\mathcal{G}_{1:N}$, uniformly for all
$j\in\{1,\ldots,\psi_{\max}\}$, the empirical solution $\hat f_j$ (ERM or approximate ERM) satisfies
\[
R(\hat f_j)-R(f_j^\star)\ \le\ B\sqrt{\frac{j}{n_N}},
\]
where $f_j^\star\in\arg\min_{f\in\mathcal{F}_j}R(f)$ and $B>0$ is a constant (possibly absorbing
$\log(\psi_{\max}/\delta)$ and other standard factors).
\end{assumption}

\begin{assumption}[Curvature discretization/approximation controlled by $\mathcal{S}_N$]
\label{assump:curv_approx}
For each $j\in\{1,\ldots,\psi_{\max}\}$, there exists a candidate curvature set
$\mathcal{K}_j=\{\kappa_m\}_{m=1}^j\subset\mathcal{I}$ such that:
(i) (Covering w.r.t.\ population optimum) for $\kappa^\star\in\arg\min_{\kappa\in\mathcal{I}}R(f_\kappa)$, there exists
$\kappa_m\in\mathcal{K}_j$ with $|\kappa_m-\kappa^\star|\le \varepsilon_N(j)$;
(ii) (Risk Lipschitz in curvature) there exists $L_R>0$ such that
$R(f_{\kappa})-R(f_{\kappa'})\le L_R|\kappa-\kappa'|$ for all $\kappa,\kappa'\in\mathcal{I}$;
(iii) (Score-to-covering) $\varepsilon_N(j)\le c\,\frac{\mathcal{S}_N}{j}$ for some $c>0$;
(iv) (Realizability) $f_{\kappa_m}\in\mathcal{F}_j$ for all $\kappa_m\in\mathcal{K}_j$.
\end{assumption}

\begin{assumption}[Dynamic selection rule]
\label{assump:dynamic_select}
Our dynamic strategy selects $\psi_D$ as a minimizer of the bound $\mathcal{R}(j)$ over the candidate set:
\[
\psi_D\in\arg\min_{1\le j\le\psi_{\max}}\mathcal{R}(j).
\]
\end{assumption}

\begin{assumption}[Unique minimizer / deterministic tie-breaking]
\label{assump:tiebreak}
The minimizer of $\mathcal{R}(j)$ over $j\in\{1,\ldots,\psi_{\max}\}$ is unique; alternatively, a deterministic
tie-breaking rule is fixed (e.g., choose the smallest minimizing index), making the selected minimizer unique.
\end{assumption}

\medskip
\noindent\textbf{Oracle comparators and excess risk.}\;
Define $f_j^\star\in\arg\min_{f\in\mathcal{F}_j}R(f)$ and let
$f^\star=f_{\kappa^\star}$ where $\kappa^\star\in\arg\min_{\kappa\in\mathcal{I}} R(f_\kappa)$.
Define the excess risk
\[
\mathcal{E}(j):=R(\hat f_j)-R(f^\star).
\]

\begin{proof}[Proof of Theorem~\ref{thm:dynamic_expert}]
\noindent\textbf{Step 1: Excess-risk decomposition.}\;
\[
\mathcal{E}(j)=\big(R(\hat f_j)-R(f_j^\star)\big)+\big(R(f_j^\star)-R(f^\star)\big).
\]

\smallskip
\noindent\textbf{Step 2: Estimation (generalization) term.}\;
By Assumption~\ref{assump:std_gen}, with probability at least $1-\delta$ (uniformly for all $j$),
\[
R(\hat f_j)-R(f_j^\star)\le B\sqrt{\frac{j}{n_N}}.
\]

\smallskip
\noindent\textbf{Step 3: Approximation term (curvature discretization).}\;
By Assumption~\ref{assump:curv_approx}(i), choose $\kappa_m\in\mathcal{K}_j$ such that
$|\kappa_m-\kappa^\star|\le \varepsilon_N(j)$. Then by Assumption~\ref{assump:curv_approx}(ii),
\[
R(f_{\kappa_m})-R(f_{\kappa^\star})
\le L_R\,|\kappa_m-\kappa^\star|
\le L_R\,\varepsilon_N(j).
\]
By Assumption~\ref{assump:curv_approx}(iv), $f_{\kappa_m}\in\mathcal{F}_j$, and since
$f_j^\star\in\arg\min_{f\in\mathcal{F}_j}R(f)$, we have $R(f_j^\star)\le R(f_{\kappa_m})$. Hence,
\[
R(f_j^\star)-R(f^\star)
\le R(f_{\kappa_m})-R(f_{\kappa^\star})
\le L_R\,\varepsilon_N(j).
\]
Using Assumption~\ref{assump:curv_approx}(iii), $\varepsilon_N(j)\le c\,\frac{\mathcal{S}_N}{j}$, we obtain
\[
R(f_j^\star)-R(f^\star)
\le L_R c\,\frac{\mathcal{S}_N}{j}
= \frac{A\mathcal{S}_N}{j},
\quad\text{where }A:=L_R c.
\]

\smallskip
\noindent\textbf{Step 4: Combine.}\;
Combining Steps 1--3 yields (with probability at least $1-\delta$)
\[
\mathcal{E}(j)\le \frac{A\mathcal{S}_N}{j}+B\sqrt{\frac{j}{n_N}}
=: \mathcal{R}(j),
\]
which matches Eq.~\eqref{eq:UB_def} in the theorem.

\smallskip
\noindent\textbf{Step 5: Dynamic vs fixed.}\;
Let $\psi_F$ be any fixed candidate expert number, and let $\psi_D$ be the expert number selected by our dynamic strategy.
By Assumption~\ref{assump:dynamic_select},
\[
\psi_D\in\arg\min_{1\le j\le\psi_{\max}}\mathcal{R}(j).
\]
Therefore, for any fixed $\psi_F\in\{1,\ldots,\psi_{\max}\}$,
\[
\mathcal{R}(\psi_D)\le \mathcal{R}(\psi_F),
\]
which is exactly Eq.~\eqref{eq:UD_UF}.

\smallskip
\noindent\textbf{Equality condition.}\;
Under Assumption~\ref{assump:tiebreak}, the minimizer is unique, hence
$\mathcal{R}(\psi_D)=\mathcal{R}(\psi_F)$ holds if and only if $\psi_F=\psi_D$.
This completes the proof.
\end{proof}

{
\section{Proof of Theorem~\ref{thm:main_tight} (Domain Generalization Error Bound Analysis)}
\label{app:tight}

We prove that \methodname{} admits a strictly tighter best achievable target-domain surrogate upper bound than MDGFM.

For any encoder $\phi$, define the target-domain surrogate upper bound as
\begin{equation}
\label{eq:def_B_app}
\mathcal{B}(\phi):=
\underbrace{\widehat{\mathcal{E}}_{S}(\phi)+\mathfrak{C}_n(\phi)}_{\text{fitting and complexity term}}
+
\underbrace{\Delta_{\mathrm{scale}}(\phi)}_{\text{scale mismatch term}}
+
\underbrace{\Delta_{\mathrm{dom}}(\phi)}_{\text{cross-domain discrepancy term}} .
\end{equation}
Accordingly, we define
\begin{equation}
\label{eq:def_eps_app}
\epsilon_{\mathrm{MDGFM}}
:=
\inf_{\phi\in\Phi_{\mathrm M}}\mathcal{B}(\phi),
\qquad
\epsilon_{\mathrm{R\text{-}GFM}}
:=
\inf_{\psi\in\Phi_{\mathrm R}}\mathcal{B}(\psi).
\end{equation}

\begin{assumption}[Surrogate control of target-domain error]
\label{assump:target_surrogate}

Motivated by standard domain adaptation and domain generalization theory~\cite{muandet2013domain,zhang2019bridge}, we assume that for any encoder $\phi$, the target-domain error can be upper bounded by a source-side fitting term, a statistical complexity term, a graph-specific scale-mismatch term, and a cross-domain discrepancy term:
\begin{equation}
\label{eq:target_surrogate_bound}
\mathcal{E}_{T}(\phi)
\le
\widehat{\mathcal{E}}_{S}(\phi)
+
\mathfrak{C}_n(\phi)
+
\Delta_{\mathrm{scale}}(\phi)
+
\Delta_{\mathrm{dom}}(\phi)
=
\mathcal{B}(\phi).
\end{equation}
\end{assumption}

\begin{assumption}[Mild target-domain shift in a reference representation space]
\label{assump:mild_shift}

Motivated by convex-hull based target-domain approximation and discrepancy-based domain adaptation theory~\cite{ganin2016domain,zhang2019bridge}, we assume that the target domain is not arbitrarily far from the family of source domains in a reference representation space.
Specifically, there exist a reference encoder $\bar{\phi}$, a weight vector
\begin{equation}
\pi^\star=(\pi_1^\star,\ldots,\pi_{N_{\mathrm{src}}}^\star)
\in
\Delta^{N_{\mathrm{src}}},
\end{equation}
and a constant $\delta\ge 0$ such that
\begin{equation}
\label{eq:mild_shift_ref_space}
d\!\left(
P_{T,\bar{\phi}},
\sum_{i=1}^{N_{\mathrm{src}}}
\pi_i^\star P_{S_i,\bar{\phi}}
\right)
\le
\delta,
\end{equation}
where $d(\cdot,\cdot)$ is the discrepancy underlying $\Delta_{\mathrm{dom}}(\cdot)$, and $P_{T,\bar{\phi}}$ and $P_{S_i,\bar{\phi}}$ denote the target-domain and source-domain representation distributions induced by $\bar{\phi}$, respectively.

We consider the mild-shift regime where $\delta$ is small. In this regime, the target domain lies in a convex-hull neighborhood of the source domains, so the cross-domain discrepancy term remains controllable.
Moreover, assume that there exist a neighborhood $\mathcal{N}(\bar{\phi})$ and a constant $L>0$ such that for any encoder $\varphi\in\mathcal{N}(\bar{\phi})$,
\begin{equation}
\label{eq:local_stability_dom}
\left|
\Delta_{\mathrm{dom}}(\varphi)
-
\Delta_{\mathrm{dom}}(\bar{\phi})
\right|
\le
L\,\mathrm{dist}_{\Phi}(\varphi,\bar{\phi}).
\end{equation}
Thus, for encoders in $\mathcal{N}(\bar{\phi})$, the additional cross-domain discrepancy residual can be bounded by some $\varepsilon_d\ge 0$:
\begin{equation}
\label{eq:dom_residual}
\Delta_{\mathrm{dom}}(\psi)
\le
\Delta_{\mathrm{dom}}(\phi)+\varepsilon_d .
\end{equation}
\end{assumption}

\begin{proof}[Proof of Theorem~\ref{thm:main_tight}]

\noindent\textbf{Step 1: Nearly optimal MDGFM encoder.}\;
Fix an arbitrary $\xi>0$. By the definition of $\epsilon_{\mathrm{MDGFM}}$, there exists an encoder $\phi_\xi\in\Phi_{\mathrm M}$ such that
\begin{equation}
\label{eq:xi_optimal_M}
\mathcal{B}(\phi_\xi)
\le
\epsilon_{\mathrm{MDGFM}}+\xi .
\end{equation}

\noindent\textbf{Step 2: Source-side gain of \methodname{}.}\;
Next, by the theoretical results established in Theorems~3.2--3.4, under the same source-domain training protocol, there exists an encoder $\psi_\xi\in\Phi_{\mathrm R}$ and a constant $\varepsilon_g>0$ such that
\begin{align}
\label{eq:source_gain_bridge}
&
\widehat{\mathcal{E}}_{S}(\psi_\xi)
+
\mathfrak{C}_n(\psi_\xi)
+
\Delta_{\mathrm{scale}}(\psi_\xi)
\nonumber\\
&\qquad\le
\widehat{\mathcal{E}}_{S}(\phi_\xi)
+
\mathfrak{C}_n(\phi_\xi)
+
\Delta_{\mathrm{scale}}(\phi_\xi)
-
\varepsilon_g .
\end{align}
This gain reflects the benefit of multi-scale GoG modeling and dynamic Riemannian routing in reducing scale mismatch and improving source-side representation quality.

\noindent\textbf{Step 3: Cross-domain discrepancy residual.}\;
By Assumption~\ref{assump:mild_shift}, both $\phi_\xi$ and $\psi_\xi$ lie in the mild-shift neighborhood $\mathcal{N}(\bar{\phi})$, and the cross-domain discrepancy residual is bounded as
\begin{equation}
\label{eq:domain_residual_apply}
\Delta_{\mathrm{dom}}(\psi_\xi)
\le
\Delta_{\mathrm{dom}}(\phi_\xi)
+
\varepsilon_d .
\end{equation}

\noindent\textbf{Step 4: Combine the two bounds.}\;
Combining Eq.~\eqref{eq:source_gain_bridge} and Eq.~\eqref{eq:domain_residual_apply}, we have
\begin{align}
\mathcal{B}(\psi_\xi)
&=
\widehat{\mathcal{E}}_{S}(\psi_\xi)
+
\mathfrak{C}_n(\psi_\xi)
+
\Delta_{\mathrm{scale}}(\psi_\xi)
+
\Delta_{\mathrm{dom}}(\psi_\xi)
\nonumber\\
&\le
\widehat{\mathcal{E}}_{S}(\phi_\xi)
+
\mathfrak{C}_n(\phi_\xi)
+
\Delta_{\mathrm{scale}}(\phi_\xi)
-
\varepsilon_g
+
\Delta_{\mathrm{dom}}(\phi_\xi)
+
\varepsilon_d
\nonumber\\
&=
\mathcal{B}(\phi_\xi)
-
(\varepsilon_g-\varepsilon_d).
\end{align}
Let
\begin{equation}
\eta:=\varepsilon_g-\varepsilon_d>0.
\end{equation}
Then
\begin{equation}
\label{eq:B_psi_gain}
\mathcal{B}(\psi_\xi)
\le
\mathcal{B}(\phi_\xi)-\eta .
\end{equation}

\noindent\textbf{Step 5: Taking the infimum over \methodname{}.}\;
Together with Eq.~\eqref{eq:xi_optimal_M}, this gives
\begin{equation}
\mathcal{B}(\psi_\xi)
\le
\epsilon_{\mathrm{MDGFM}}+\xi-\eta .
\end{equation}
Since $\psi_\xi\in\Phi_{\mathrm R}$, we obtain
\begin{equation}
\epsilon_{\mathrm{R\text{-}GFM}}
=
\inf_{\psi\in\Phi_{\mathrm R}}\mathcal{B}(\psi)
\le
\mathcal{B}(\psi_\xi)
\le
\epsilon_{\mathrm{MDGFM}}+\xi-\eta .
\end{equation}
Because this holds for every $\xi>0$, letting $\xi\downarrow 0$ yields
\begin{equation}
\epsilon_{\mathrm{R\text{-}GFM}}
\le
\epsilon_{\mathrm{MDGFM}}-\eta
<
\epsilon_{\mathrm{MDGFM}}.
\end{equation}

Therefore, \methodname{} admits a strictly tighter best achievable target-domain surrogate upper bound than MDGFM.
This completes the proof.
\end{proof}
}

\end{document}